\documentclass[twoside,11pt]{article}

\usepackage{blindtext}

\usepackage{amssymb,amsmath}
\usepackage{amsthm}
\usepackage[hidelinks]{hyperref}
\usepackage{cleveref}

\makeatletter
\DeclareRobustCommand\onedot{\futurelet\@let@token\@onedot}
\def\@onedot{\ifx\@let@token.\else.\null\fi\xspace}

\def\eg{\emph{e.g}\onedot} 
\def\ie{\emph{i.e}\onedot}

\makeatother

\def\parencite{\citep}
\def\textcite{\citet}

\usepackage{tabularray}
\usepackage{algorithm}
\usepackage[noend]{algpseudocode}
\usepackage{mathrsfs}
\usepackage{url}
\usepackage{xr-hyper}
\usepackage{bbm}
\usepackage{multirow} % For multirow/multicolumn cells in a table. See https://en.wikibooks.org/wiki/LaTeX/Tables#Columns_spanning_multiple_rows
\usepackage{graphicx}
\usepackage{enumitem}   % for more control over enumerate
\usepackage{mdwlist}	% compact lists and the 'note' environment
\usepackage{xspace}
\usepackage{xcolor}
\usepackage{caption}
\usepackage{subcaption}
\usepackage{thm-restate}
\usepackage{cancel}
\usepackage{placeins}

\usepackage[nohyperref,preprint]{jmlr2e}

\providecommand{\main}{.}
\newcommand{\scE}{\mathrm{E}}

\newcommand{\sS}{\mathcal{S}}
\newcommand{\sR}{\mathcal{R}}
\newcommand{\sD}{\mathcal{D}}

\newcommand{\sU}{\mathcal{U}}
\newcommand{\sV}{\mathcal{V}}
\newcommand{\sE}{\mathcal{E}}
\newcommand{\sG}{\mathcal{G}}
\newcommand{\sW}{\mathcal{W}}
\newcommand{\sT}{\mathcal{T}}

\newcommand{\sM}{\mathcal{M}}
\newcommand{\zeros}{\boldsymbol{\mathrm{0}}}
\newcommand{\ones}{\boldsymbol{\mathrm{1}}}
\newcommand{\bz}{\textbf{z}}

\newcommand{\bc}{\textbf{c}}
\newcommand{\bgh}{\hat{\boldsymbol{g}}}
\newcommand{\gh}{\hat{g}}
\newcommand{\bg}{\boldsymbol{g}}
\newcommand{\bx}{\textbf{x}}
\newcommand{\br}{\textbf{r}}
\newcommand{\bzh}{\hat{\boldsymbol{z}}}
\newcommand{\bw}{\textbf{w}}
\newcommand{\bu}{\textbf{u}}
\newcommand{\bus}{\bu^{(s)}}
\newcommand{\uis}{u_i^{(s)}}
\newcommand{\dis}{d_i^{(s)}}
\newcommand{\ds}{d^{(s)}}

\newcommand{\uTil}{\tilde{u}}
\newcommand{\uTils}{\uTil^{(s)}}

\newcommand{\zTil}{\tilde{z}}
\newcommand{\zTils}{\tilde{z}^{(s)}}
\newcommand{\bzTil}{\tilde{\bz}}
\newcommand{\bzTils}{\tilde{\bz}^{(s)}}

\newcommand{\brho}{\boldsymbol{\rho}}
\newcommand{\fz}{f_z}
\newcommand{\Eps}{\boldsymbol{\epsilon}}
\newcommand{\Nz}{\overline{\textbf{z}}}
\newcommand{\nz}{\overline{z}}

\newcommand{\Nu}{\overline{\textbf{u}}}

\newcommand{\Nus}{\overline{\textbf{u}}^{(s)}}
\newcommand{\Ntheta}{\overline{\boldsymbol{\theta}}}
\newcommand{\ntheta}{\overline{\theta}}
\newcommand{\paddedtheta}{\overset{}{\theta}}
\newcommand{\Zeta}{\boldsymbol{\zeta}}
\newcommand{\nZeta}{\overline{\boldsymbol{\zeta}}}
\newcommand{\R}{\mathbb{R}}
\newcommand{\E}[2]{\underset{\scriptscriptstyle #1}{\mathbb{E}}\left[#2\right]}

\newcommand{\Var}[1]{\mathrm{Var}\left[#1\right]}
\newcommand{\Corr}[1]{\mathrm{Corr}\left[#1\right]}
\newcommand{\Cov}[1]{\mathrm{Cov}\left[#1\right]}

\newcommand{\mP}{\mathscr{P}}
\newcommand{\corners}{\{0,1\}^d}
\newcommand{\hypercube}{[0,1]^d}
\newcommand{\argmin}[1]{\underset{#1}{\mathrm{arg\ min}}\ }
\newcommand{\argmax}[1]{\underset{#1}{\mathrm{arg\ max}}\ }
\newcommand{\mymin}[1]{\underset{#1}{\mathrm{min}}\ }
\newcommand{\mymax}[1]{\underset{#1}{\mathrm{max}}\ }
\newcommand{\lz}[1]{\|#1\|_0 }
\newcommand{\diff}{\mathrm{d}}
\newcommand{\ind}{\mathbbm{1}}
\newcommand{\euler}{\mathrm{e}}
\newcommand{\sigmoid}{\sigma}
\newcommand{\logit}{\mathrm{Logit}}
\newcommand{\Ber}{\mathrm{Ber}}
\newcommand{\Exp}{\mathrm{Exp}}
\newcommand{\btheta}{\boldsymbol{\theta}}
\newcommand{\bzs}{\bz^{(s)}}
\newcommand{\bzsp}{\bz^{(s')}}
\newcommand{\zsi}{z_i^{(s)}}
\newcommand{\OrdSetMinusZs}{(\bzsp)_{s' \neq s}}

\newcommand{\Jh}{\hat{J}}
\def\multiset#1#2{\ensuremath{\left(\kern-.3em\left(\genfrac{}{}{0pt}{}{#1}{#2}\right)\kern-.3em\right)}}

\newsavebox{\smlmatAll}% Box to store smallmatrix content
\newsavebox{\smlmatA}
\newsavebox{\smlmatB}
\newsavebox{\smlmatC}
\newsavebox{\smlmatD}

\setlist[enumerate]{leftmargin=12pt,itemsep=0pt}
\setlist[itemize]{leftmargin=12pt,itemsep=0pt}

\newcommand{\eqlabelleft}{(}
\newcommand{\eqlabelright}{)}

\newcommand{\Crefalt}[1]{%
	\begingroup%
	\renewcommand{\eqlabelleft}{}%
	\renewcommand{\eqlabelright}{}%
	\Cref{#1}%
	\endgroup%
}

\creflabelformat{equation}{\eqlabelleft#2#1#3\eqlabelright}

\DefTblrTemplate{firsthead,middlehead,lasthead}{default}{
}
\DefTblrTemplate{firstfoot}{default}{
	\UseTblrTemplate{contfoot}{default}
	\UseTblrTemplate{caption}{default}
}
\DefTblrTemplate{middlefoot}{default}{
	\UseTblrTemplate{contfoot}{default}
	\UseTblrTemplate{capcont}{default}
}
\DefTblrTemplate{lastfoot}{default}{
	\UseTblrTemplate{note}{default}
	\UseTblrTemplate{remark}{default}
	\UseTblrTemplate{capcont}{default}
}

\usepackage{array}
\newcommand{\PreserveBackslash}[1]{\let\temp=\\#1\let\\=\temp}
\newcolumntype{C}[1]{>{\PreserveBackslash\centering}p{#1}}
\newcolumntype{R}[1]{>{\PreserveBackslash\raggedleft}p{#1}}
\newcolumntype{L}[1]{>{\PreserveBackslash\raggedright}p{#1}}
\usepackage{lastpage}

\usepackage{lastpage}
\jmlrheading{XX}{2024}{1-\pageref{LastPage}}{X/XX; Revised X/XX}{X/XX}{XX-XXXX}{Hugo Silva and Martha White}

\ShortHeadings{Discrete Optimization and NN}{Silva and White}
\firstpageno{1}

\begin{document}

\title{What to Do When Your Discrete Optimization Is the Size of a Neural Network?}

\author{\name Hugo Silva \email hugoluis@ualberta.ca\\
		\name Martha White \email whitem@ualberta.ca \\
	\addr Department of Computing Science, Alberta Machine Intelligence Institute (Amii)\\
	University of Alberta\\
	Edmonton, Alberta, Canada
}

\editor{XXXX}

\maketitle

\begin{abstract}%   <- trailing '%' for backward compatibility of .sty file
Oftentimes, machine learning applications using neural networks involve solving discrete optimization problems, such as in pruning, parameter-isolation-based continual learning and training of binary networks. Still, these discrete problems are combinatorial in nature and are also not amenable to gradient-based optimization. Additionally, classical approaches used in discrete settings do not scale well to large neural networks, forcing scientists and empiricists to rely on alternative methods. Among these, two main distinct sources of top-down information can be used to lead the model to good solutions: (1) extrapolating gradient information from points outside of the solution set (2) comparing evaluations between members of a subset of the valid solutions. We take continuation path (CP) methods to represent using purely the former and Monte Carlo (MC) methods to represent the latter, while also noting that some hybrid methods combine the two. The main goal of this work is to compare both approaches. For that purpose, we first overview the two classes while also discussing some of their drawbacks analytically. Then, on the experimental section, we compare their performance, starting with smaller microworld experiments, which allow more fine-grained control of problem variables, and gradually moving towards larger problems, including neural network regression and neural network pruning for image classification, where we additionally compare against magnitude-based pruning.
\end{abstract}


\section{Introduction}

Problems in the mathematical sciences often involve the minimization of a function $J(\cdot)$ over some set $\mathcal{S}$. Initializing some tentative parameters, then changing them according to how $J(\cdot)$ behaves locally is one way to approach these problems. Assuming the current local direction in which $J(\cdot)$ decreases the most leads to reasonable next values of the tentative parameters, we can recalculate and follow it iteratively until local improvement no longer seems possible. If such a direction is determined using the gradient, methods following the above procedure belong to the field of gradient-based optimization \parencite{chong2013introduction, boyd2004convex}.

Gradient-based optimization, however, is not applicable in many problems of interest to mathematicians. Firstly, the objective function has to be differentiable, otherwise it is not possible to calculate the gradients. Furthermore, following the direction of the gradient often involves changing the tentative parameters slightly in the desired direction. In general, for the next potential solutions to also belong to $\mathcal{S}$, all points in a small enough neighbourhood around the current parameter values have to also belong to $\mathcal{S}$. Tabular problems, where $S$ consists of a finite set of values, for example, are not suitable for gradient-based optimization, since no two elements of $S$ are arbitrarily close to each other.

Pseudo-Boolean (PB) functions are mappings from $d$ dimensional binary vectors to $\R$ and correspond to tabular functions that map each of the $2^d$ possible entries to a real number \parencite{boros2002pseudo}. They are called pseudo because they map to the reals, whereas a Boolean function has binary outputs (for an overview of the latter, see \citealt{o2014analysis}). Exact optimization of a PB function is of combinatorial nature: for the general case, where there is no special structure one can take advantage of, finding the best binary value involves trial-and-error evaluation of all elements in the set, which can be prohibitive in case $d$ is also large \parencite{korte2011combinatorial}. 

Then again, such problems appear in many different fields, including game theory \parencite{hammer1992approximations}, computer science \parencite{karp2010reducibility, madhulatha2012overview}, VLSI design \parencite{junger1994quadratic, boros1999optimal}, statistical mechanics \parencite{phillips1994quadratic}, finance \parencite{hammer1971applications}, manufacturing \parencite{kubiak1995new} and operations research \parencite{picard1978cut}. The combinatorial nature of finding the optimal point in some of these large problems can sometimes make it necessary to forfeit exact optimization in favour of approximate solutions. Common approaches used in these types of problems are local search methods \parencite{hansen1986steepest, johnson1988easy, glover1990tabu}, tabu search \parencite{glover1989tabu, glover1990tabu}, quadratization \parencite[Section 4.4]{boros2002pseudo}, branch-and-bound \parencite{lawler1966branch} and genetic algorithms \parencite{kramer2017genetic}. Occasionally, some of these can work in polynomial time. More recently, to cope with this high complexity, studies have started applying quantum computing to speed up optimization of some PB problems \parencite{mcgeoch2013experimental, montanaro2016quantum}. Still, these studies are preliminary and of limited application.

One particular case of such problems appears when combining discrete optimization with neural networks (NN). The latter have become increasingly popular in the last decade, specially since convolutional networks achieved breakthrough performance on Imagenet in 2012 (a work later published by \citealt{krizhevsky2017imagenet}), gaining even more popularity recently after the announcement of ChatGPT, based on the transformer model proposed by \textcite{vaswani2017attention}. Improving certain functionalities of NN sometimes involves minimizing PB functions. This is the case, for example, when trying to regularize the network with pruning \parencite{hoefler2021sparsity}; reduce memory and computation requirements via training binary networks \parencite{qin2020binary}, mixed-precision quantization \parencite[section IV-B]{gholami2022survey}, or, again, by pruning; when allocating sub-networks to different sub-problems, such as in sequential task learning \parencite{parisi2019continual, de2021continual, mai2022online} or multitask learning \parencite{crawshaw2020multi, caruana1997multitask}. PB optimization also appears in ticket search, with applications in transfer learning \parencite{morcos2019one, mehta2019sparse, van2019using}. In such cases, the dimensionality of the desired binary vector can scale with the size of the model, which can reach billions or trillions of parameters, such as in \textcite{fedus2022switch, lepikhin2020gshard}.

When confronted with these kinds of problems, machine learning researchers tend to adopt solutions different from the ones mentioned previously. In this paper, we have roughly categorized the approaches in three distinct groups: continuation path (CP) methods \parencite{allgower2003introduction}, which start from a smooth problem and change it gradually back to the discrete one; Monte Carlo (MC) gradient estimation \parencite{mohamed2020monte}, which model the problem with a probability distribution and use sampling to track good solutions; hybrid methods, which combine the two ideas above by relying on sampling, but also using the gradient of the cost function in some way. Admittedly, these will not cover all existing methods. For example, sometimes one can use problem-specific approaches instead. However, this categorization captures fundamentally distinct ways to approximately solve the problems that have each gathered sufficient attention to warrant further understanding.

The goal, therefore, is twofold: we first aim at understanding whether these approaches used for NN problems are consistent from a more general PB optimization perspective and then we study how they compare against each other. For the first goal, we will introduce each category, expanding on where the ideas come from while also presenting analytical results and simple counter-examples to illustrate their strengths and weaknesses. For the second goal, we have designed multiple benchmarks to investigate how their relative performances and properties change as we increase the dimensionality of the problem. Initial experiments are set in smaller scenarios, providing additional insight and allowing us to connect results to the theory by using some closed-form expressions, which are normally impossible to compute. Then, to both understand the effect of overparametrization and to contrast the methods in more practical scenarios, later experiments are set in larger problems. We will mainly focus in comparing MC and CP approaches, but will also include hybrid approaches in some of these experiments.

Our analysis suggests that both CP and MC methods fall short of being robust options for PB optimization. CP, despite its origins in the well-explored domain of numerical continuation---which aims to identify the roots of nonlinear systems---loses its reliability when extended to PB optimization. This inadequacy is evident even in simple counter-examples, and is later confirmed by our microworld experiments. However, our larger-scale experiments reveal that integrating CP with expansive neural networks significantly enhances performance. This improvement may be linked to the Lipschitz constant of the loss function, a connection hinted at by surveyed literature. We hypothesize that the neural network's functional form, combined with overparametrization, may alter derivative magnitudes in a way that benefits PB optimization. In contrast, MC methods exhibit a different trend. Unlike CP approaches, their objective is equivalent to that of the original PB optimization and they perform well in smaller-scale experiments. Yet, in larger settings, CP methods decisively outperform them. This superiority persists regardless of adjustments to their parametrization, estimator or even large increases in sample size compared to CP methods---some of these larger experiments used $100$ samples for MC methods, but to no avail. Our analytical findings underscore two major limitations of MC methods: their dependency on the current distribution and a disruptive interaction among potential solutions to PB optimization. We demonstrate that the latter can lead to catastrophic outcomes, even with an infinite number of samples.

\subsection{Outline and Contributions}

We start \Cref{ch:pb_opt} by formalizing the pseudo-Boolean optimization objective and presenting examples specific to machine learning with neural networks, followed by introducing some basic notation and definitions. Then, the section concludes by categorizing and describing methods used to approximately solve PB optimization problems of a more general scope than the NN problems presented previously.

The subsequent sections are about the approaches used with neural networks. \Cref{ch:numerical_continuation} talks about CP methods, which are rooted in numerical continuation. The section starts by explaining how numerical continuation surfaces as a way of finding roots of a system of nonlinear equations. Then we expand the discussion by talking about one way of adapting it to be used with neural networks. We conclude the section by designing some simple counter-examples where the intuition behind these approaches does not hold, causing the method to fail.

In \Cref{ch:monte_carlo}, we talk about MC gradient estimation, where we first introduce the probabilistic formulation of the problem and how it is approximated by sampling. We then discuss three main distinct ways of estimating the gradient. Afterwards, we introduce the estimators to be used in our experiments. When doing so, we derive some new algorithms, including iterative variants of REINFORCE, ARMS and LOORF that do not have memory cost scaling with the number of samples, something we will need in the larger experiments. We also derive the closed-form expression of the intractable constant baseline that can reduce the variance the most, which we compare against in our smaller experiments (\Cref{thm:bStarOne,thm:bStarTwo}). Next, we talk about the main drawbacks of this method, where we present counter-examples. For one of them, we derive a bound stating that the minimum value of the cost function has to go to minus infinity exponentially with respect to the problem size, otherwise such problem will not be solvable via the probabilistic formulation (\Cref{thm:mc_bad_gen_2} and \Cref{cor:mc_bad_gen_2}). This happens because of an effect which we call \emph{unwanted generalization} between solutions, which cannot be overcome by reducing the variance with more samples. Finally, we conclude the section by presenting multiple ways of parametrizing the solution, all of which we will compare in the experiments. \Cref{ch:other} briefly talks about hybrid approaches, which combine the two ideas from the previous sections.

\Cref{ch:experiments} corresponds to the experimental part. We start on two newly introduced smaller benchmarks, which we refer to as microworlds. The first defines candidate solutions based an exponential distribution and the second inputs the binary entry to a fixed neural network, which acts directly as the loss function. Afterwards, we move to a larger regression setting and we introduce yet another new benchmark, where a neural network with fixed weights has to learn the mapping from a more complex NN architecture by varying only binary masks. Finally, we move to the pruning setting, where we perform experiments on MNIST and CIFAR-10. We experiment with both training the binary mask of a fixed backbone network as well as jointly training weights and masks. We compare different estimators and parametrizations of the MC approach, as well as perform comparisons between MC, CP and hybrid methods. We conclude the paper with a discussion of the main takeaways from the study in \Cref{ch:conclusion}.

\section{Pseudo-Boolean Optimization}
\label{ch:pb_opt} 

In this section, we start by exemplifying how some problems involving neural networks correspond to pseudo-Boolean optimization objectives. Then, we present basics of PB optimization. Finally, we briefly go over some common classes of algorithms used to find (approximate) solutions. We emphasize, that, although such algorithms are common in more general PB optimization, they are not so prevalent in problems involving neural networks. The subsequent sections of this paper will describe some of the main approaches for the latter case, which will also be the ones we will use in our experiments.

\subsection{Neural Networks and PB Optimization}
\label{sec:pb_nn_examples}
Pseudo-Boolean optimization involves solving the following minimization:
\begin{equation}
	\label{eq:pb_prob}
	\bz^* = \argmin{\bz \in \corners} J(\bz).
\end{equation}
This is a tabular problem whose solution requires a combinatorial search over $2^d$ candidates (for the general case). We will now present some examples of PB optimization involving neural networks, whereas in \Cref{sec:pb_examples} we present more general cases belonging to fields such as graph optimization and computational complexity theory.

When dealing with neural networks, we can expect $d$ to be a very large number and the underlying problem to have no special structure that can be explored. Inception-V3 \parencite{szegedy2016rethinking}, for example, has 5.7 billion arithmetic operations and 27 million parameters, whereas GPT-3 \parencite{brown2020language} requires 175 billion parameters. Some works even have models requiring trillions of parameters \parencite{fedus2022switch, lepikhin2020gshard}. In the examples below, $\bz$ scales with the size of the model.

\begin{example}[Training of Binary Neural Networks]
	\label{ex:pb_binary}
	Considering problems where the model is parametrized by a neural network, minimizing the loss corresponds to:\footnote{We emphasize that some problems have an infimum, but not a minimum. Still, a solution close to the lower bound is often desired even in these cases. When we say the goal is to minimize the problems we refer to both situations.}
	\begin{equation*}
		\text{minimize } J(\bw) = \E{X,Y \sim \sD_s}{\ell(f(X; \bw), Y)}, \qquad \text{for }\bw\in \sW.
	\end{equation*}	
	Where $X$ and $Y$ are the data and labels respectively, $\sD_s$ is the data distribution, $\bw$ represents the  weights of the neural network, which belong to the set $\sW$. The neural network evaluation is denoted by $f(\cdot)$ and $\ell(\cdot)$ is the loss, which can be, for instance, a regression or a classification loss.
	
	Sometimes it is desirable to constrain neural network weights to only have a discrete set of possible values, which can be due to memory constraints or to obtain hardware speedups. Work done in \textcite{courbariaux2015binaryconnect, hubara2016binarized}, for example, allows weights to be either $1$ or $-1$. The problem becomes finding:
	\begin{equation*}
		\argmin{\bz \in \corners} \E{X,Y \sim \sD_s}{\ell(f(X; 2\bz - \ones ), Y)}.
	\end{equation*}
	
\end{example}

\begin{example}[Neural Network Pruning]
	\label{ex:pruning}
	This problem consists of finding a smaller subnetwork that performs as well as possible, which can serve two main purposes: regularization or model compression. The search of such a subset of weights corresponds to the following problem:
	\begin{equation}
		\label{eq:pruning_l0_raw}
		\text{minimize } \E{X,Y \sim \sD_s}{\ell(f(X; \bw), Y)} + \lambda \lz{\bw}, \qquad \text{for }\bw\in \sW,
	\end{equation}
	where the $L_0$ norm is given by:
	\begin{equation*}
		\lz{\bw} = \sum_{i=1}^d \ind[w_i \ne 0].
	\end{equation*}
	We can defer the $L_0$ norm of the original weights to an auxiliary discrete variable $\bz$, obtaining:
	\begin{equation}
		\text{minimize } \E{X,Y \sim \sD_s}{\ell(f(X; \bw \odot \bz ), Y)} + \lambda \sum_{i=1}^d z_i, \qquad \text{for }\bw\in \sW,\ \bz \in \corners,
		\label{eq:pruning_l0_deferred}
	\end{equation}
	which, in turn, consists of the concurrent minimization with respect to the main weights $\bw$, as well as the discrete variable $\bz$. The latter is an instance of PB optimization. In case we only optimize $\bz$ and leave $\bw$ fixed, we say that the problem is of finding supermasks.
\end{example}

\begin{example}[Sequential Task Learning]
	This setting is often used as a simplification of the more general continual learning setting \parencite{parisi2019continual, de2021continual, mai2022online}. The model has to solve tasks from an incoming stream, one at a time, and the final performance across all tasks is measured. Considering supervised learning and a sequence of $T$ tasks, the raw objective then is:
	\begin{equation*}
		\text{minimize } \sum_{t=1}^T \E{\sD_s^{(t)}}{\ell(f_t(X^{(t)}; \bw), Y^{(t)})}, \qquad \text{for }\bw\in \sW.
	\end{equation*}
	Where $X^{(t)}$, $Y^{(t)}$, $\sD_s^{(t)}$ and $f_t$ represent the examples, labels, data set and network function for task $t$. One of the main challenges that appear in this context is catastrophic forgetting, where the adaptation to the later tasks causes forgetting of the earlier ones \parencite{sutton1986two, mccloskey1989catastrophic, french1999catastrophic}. In other words, the plasticity necessary to solve further tasks often comes at the cost of stability.
	
	Protecting weights important to the initial tasks is one way of balancing plasticity and stability. There is a class of methods called parameter isolation that does so by allocating parts of the network to different tasks. Since this approach consists of a search over binary vectors indicating whether weights should be allocated to some task $t$ or not, it is also an instance of PB optimization. In fact, HAT \parencite{serra2018overcoming} and Packnet \parencite{mallya2018packnet} are two parameter isolation approaches representative of the CP methods from \Cref{ch:numerical_continuation} and the magnitude-based approaches we used in our pruning experiments respectively, whereas Pathnet \parencite{fernando2017pathnet} uses genetic algorithms to approximately solve this problem.
\end{example}

The problem of having a model trained concurrently to solve multiple tasks, instead of sequentially, is called multitask learning and also includes approaches that allocate subsets of the network to different tasks \parencite{crawshaw2020multi, caruana1997multitask}. Yet another PB optimization problem using neural networks is that of finding lottery tickets, which consist of sparse subnetworks that, when trained in isolation, can reach performances similar to the dense model, even outperforming it in some cases \parencite{pensia2020optimal, lee2019signal, frankle2018lottery}. Interestingly, these tickets can sometimes be transferred between data sets \parencite{morcos2019one, mehta2019sparse, van2019using}.

Finally, mixed-precision NN quantization \parencite[section IV-B]{gholami2022survey} involves finding the best precision to use for each layer of a neural network. Notably, the model performance will be more affected by changing some of them to be of lower precision than others. The search for the right combination with respect to the performance-memory tradeoff is again a discrete search problem that falls in the PB optimization framework.

\subsection{Basics}
\label{sec:pb_basics} 

This section presents some basic concepts of PB optimization. We will often use boldface to denote vectors, unless otherwise noted. A pseudo-Boolean function $J(\bz)$ is a mapping $\corners \rightarrow \R$. One simple way to identify such a function is to note that it is a table containing $2^d$ entries, each entry corresponding to one of the possible inputs. Additionally, there is a one-to-one correspondence between each such input and the binary representation of the first $2^d$ natural numbers. To denote this relation, we will use $\Zeta_h$ to represent the vector in $\corners$ satisfying:
\begin{equation*}
	\sum_{i=1}^d 2^{i-1} (\Zeta_h)_{i} = h \qquad \text{ for } h \in \{0,...,2^d-1\},
\end{equation*}
where the notation $(\Zeta_h)_i$ is used to index the i-th element of the vector $\Zeta_h$. Sometimes, to avoid clutter, we will denote the i-th element of a vector $\bz$ as simply $z_i$. To exemplify, considering $d=2$ we have:
\begin{align*}
	\Zeta_0 = \begin{bmatrix} 0 \\ 0 \end{bmatrix};
	\Zeta_1 = \begin{bmatrix} 1 \\ 0 \end{bmatrix};
	\Zeta_2 = \begin{bmatrix} 0 \\ 1 \end{bmatrix};
	\Zeta_3 = \begin{bmatrix} 1 \\ 1 \end{bmatrix}
\end{align*}
and $J(\bz)$ can be represented by the table listing all possible $J(\Zeta_h)$. Notably, any alternative function $J': [0,1]^d \rightarrow \R$ satisfying:
\begin{equation}
	\label{eq:different_J}
	J(\Zeta_h) = J'(\Zeta_h), \qquad \text{ for } h \in \{0,...,2^d-1\}
\end{equation}
can also be used to generate the same mapping, regardless of how $J'(\cdot)$ behaves outside of $\corners$. Methods from the next sections will use different functions which will all correspond to the same problem, but result in very different optimization behaviour, specially outside of the binary values. We can additionally represent $J(\bz)$ with the following equivalent form:
\begin{equation}
	\label{eq:multi_poly}
	\mP_J(\bz) = \sum_{h=0}^{2^d-1} \left( \prod_{i=1}^d z_i^{(\Zeta_h)_i} \nz_i^{(\nZeta_h)_i} \right) J(\Zeta_h),
\end{equation}
where we denote $\Nz=\ones-\bz$, for $\ones = [1,\cdots,1]^\top$, and we adopt the convention $0^0=1$. For reference, $\nz$ and $z$ are often called literals. We note that $\mP_J(\bz) = J(\bz)$ for all $\Zeta_h$, since the only summand that is $\neq 0$ for $\bz = \Zeta_k$ is the iterate where $h=k$ and the inner product inside the parenthesis will evaluate to $1$ for this summand.

Note also that, for two different $J'(\cdot)$ and $J''(\cdot)$ satisfying \Cref{eq:different_J}, $\mP_{J'}(\bu) = \mP_{J''}(\bu)$ for $\bu \in [0,1]^d$ (as opposed to simply for $\bu \in \corners$). This will be important later in this work, when we explore the behavior of $\mP_J(\cdot)$ outside of $\corners$. In particular, $\mP_J(\cdot)$ is the objective for the MC methods from \Cref{ch:monte_carlo}, with optimization happening over real values. Sometimes we will have some control over $J(\cdot)$ and changing it to, for example, facilitate optimization can be a good choice, for it may not interfere with $\mP_J(\bz)$.

The form in \Cref{eq:multi_poly} is a multilinear polynomial, since its terms never appear exponentiated by more than a factor of $1$. Considering $\sD = \{1,2,...,d\}$, the polynomial can be alternatively denoted as:
\begin{equation}
	\label{eq:multi_poly_set}
	\mP_J(\bz) = \sum_{\sS \subseteq \sD} w_{\sS}\ \prod_{i \in \sS} z_i.
\end{equation}
Where $w_{\sS}$ are weights assigned to subsets of $\sD$ and the product is one for the empty set. The degree of the PB function is the size of the largest $\sS \subseteq \sD$ for which $w_{\sS} \neq 0$ in \Cref{eq:multi_poly_set}. See \Cref{sec:pb_basics_appendix} for additional properties of the multilinear polynomial, as well as a simple illustration of these forms for two dimensions. Finally, in PB optimization, some local search methods use the following alternative definition of derivative:
\begin{align} \label{eq:pd_der}
	\Delta_i (\bz) = J(z_1,\cdots,&z_{i-1} , 1, z_{i+1},\cdots,z_{d}) - J(z_1,\cdots,z_{i-1}, 0, z_{i+1},\cdots,z_{d}).
\end{align}
In addition to being used in local search, we will see that this difference will appear is some results from \Cref{ch:monte_carlo}. On this study, whenever we mention the term derivative, it will always refer to the conventional derivative, not the one in \Cref{eq:pd_der}, unless specifically stated.

\subsection{Algorithmic Approaches}
\label{sec:pb_algs}

To give the reader a broad overview of common approaches used more generally in PB optimization problems, \Cref{tab:pb_approaches} briefly lists and describes some of them. \Cref{sec:pb_algs_detailed} further explains such approaches and gives some references in case the reader wants to know more about each of them. We mention that these approaches are common in general PB optimization, but not so much in NN problems. For these problems, variants of the methods discussed in the subsequent sections are more prevalent. Sometimes, the methods from \Cref{tab:pb_approaches} can leverage special structure and find solutions in polynomial time. We reiterate, however, that more general cases may not contain such structure to be exploited, limiting the applicability of these methods. This is likely to be the case for many NN problems. Furthermore, in our case, $d$ is high and a single pass through all of the dimensions might be prohibitive, causing even the best-case polynomial algorithms to be unsuitable. Namely, $d$ will be the number of weights on a neural network, sometimes reaching the order of billions or trillions. High dimensions therefore warrant alternative approaches, some of which are explored in this study.
\begin{table}
	\centering
	\begin{tabular}[htb!]{|@{}C{0.2\textwidth}@{}|p{0.7\textwidth}|}
	\hline
	\textbf{Method} & \centering \textbf{Description} \tabularnewline
	\hline
	Local search & Greedily searches the current neighbourhood, keeping the best solution found thus far. \\ \hline
	Tabu search & Modifies local search methods, allowing further exploration when local improvement no longer seems possible. \\ \hline
	Simulated annealing & Controls the exploration with a temperature parameter, which is gradually reduced throughout training. \\ \hline
	Graph cuts & Reinterpret the PB optimization as a graph problem and then leverage results from graph optimization. \\ \hline
	Genetic algorithms & Searches in a structured way, taking inspiration from biological evolution. \\ \hline
	Branch-and-bound & Subdivides the search space and searches sub-regions individually. \\ \hline
	Quadratization & Directly works with \Cref{eq:multi_poly_set}, reducing the degree of the polynomial. \\ \hline
	Approximation algorithms & Change the original problem for an approximation of it that admits more efficient solutions. \\
	\hline
	\end{tabular}
	\caption{Common approaches used in classic PB optimization problems, but not so much in neural network PB optimization problems.}
	\label{tab:pb_approaches}
\end{table}

\section{Numerical Continuation}
\label{ch:numerical_continuation}
This section introduces concepts of numerical continuation, which serve as the basis for the continuation path methods (CP) that we will study in this work. \Cref{sec:cp_basics} introduces basics of numerical continuation, \Cref{sec:pb_application} explains how to use numerical continuation to approximately solve pseudo-Boolean optimization problems. We note that there are alternative ways to approach PB optimization with CP methods, but we will restrict our discussion to its recent use on pruning and sequential task learning. Finally, \Cref{sec:cp_drawbacks} presents very simple failure cases of CP applied to PB optimization and we reason about when these methods can fail.

\subsection{Basics}
\label{sec:cp_basics}
Continuation methods, sometimes called embedding or homotopy methods, are useful when solving a system of equations of the form
\begin{equation}
	\label{eq:cont_desired}
	F(\bx) = \zeros
\end{equation}
where $F: \R^n \rightarrow \R^n$ is smooth. There exist iterative methods for finding roots of $F(x)$, such as Newton-Raphson, where
\begin{equation*}
	\bx_{t+1} = \bx_{t} -  (\nabla F(\bx_t)) ^{-1} F(\bx_t)
\end{equation*}
with $\nabla F(\bx_t)$ being the $n \times n$ Jacobian matrix evaluated at $\bx_t$. These methods often depend on a good initial value $\bx_0$ and are unlikely to work if such an initialization is not available. Continuation methods start with a simpler problem, where finding (reasonable approximations of) the solution is relatively easy, and then transition iteratively to the original problem. The solutions found for the $k$-th problem should serve as reasonable initial values when applying methods akin to Newton-Raphson to the $k+1$-th problem, simplifying the search. Namely, we construct a smooth homotopy or continuation path $H: \R^{n+1} \rightarrow \R^n$, where

\begin{equation*}
	H(\bx, 0)  = F(\bx), \quad H(\bx, 1)  = G(\bx)
\end{equation*}	
and $G(\bx)$ is some simpler function where finding roots is easier. One possible homotopy, for example, is to simply use the interpolation between $G(\bx)$ and $F(\bx)$. Based on that idea, \Cref{alg:cp_alg}, often called embedding algorithm, describes the most direct way of approaching the problem.

\begin{algorithm}
	\caption{General procedure for embedding algorithms}
	\label{alg:cp_alg}
	\begin{algorithmic}
		\State \textbf{Input}: initial value $\bx_{init}$, homotopy $H(\cdot)$;
		\State Define $\tau$ schedule $1=\tau_0 > \tau_1 > \cdots > \tau_K = 0$
		\For{$k=0 \cdots K$}
		\State Find $\hat{\bx}_{k} \approx \argmin{\bx \in \R^n}{H(\bx, \tau_k)}$, starting with $\bx_{init}$ \\
		\Comment \qquad (\eg, using Newton-Raphson)
		\State $\bx_{init} \leftarrow \hat{\bx}_{k}$
		\EndFor
		\State \textbf{Return:} $\hat{\bx}_{K}$
	\end{algorithmic}
\end{algorithm}

Denoting by $H^{-1}(\zeros)$ the set of solutions of $H(\cdot)$ such that
\begin{equation*}
	H^{-1}(\zeros) = \{ (\bx, \tau) \mid H(\bx, \tau) = \zeros \},
\end{equation*}
implicitly, we attempt to trace a curve with elements $\in H^{-1}(\zeros)$ that starts with an easy to find solution $(\bx_{0}^*,1)$ and ends with $(\bx^*,0)$, where $\bx^*$ is a solution to \Cref{eq:cont_desired}. We denote this curve by $c(s)$, with $c:\R \rightarrow \R^{n+1}$. Note that this curve does not have to be parametrized with $s = \tau$. \emph{Predictor-corrector}, for instance, parametrize $s$ using the arclength instead \parencite[Chapter 1]{allgower2012numerical}.

Problems need to satisfy some assumptions for these methods to be guaranteed to work. In particular, the existence of such a smooth curve $c(s)$ starting from $c(0) = (\bx_{0}^*,1)$ is ensured if the Jacobian $\nabla H(\cdot)$ is full-rank at $c(0)$. Furthermore, this curve will have non-zero derivative at $s=0$. If the Jacobian is also full-rank for the other points from $H^{-1}(\zeros)$, the curve will have similar structural properties to a circle. Additionally, for the curve to also contain $(\bx^*,0)$, as opposed to going to infinity or returning to $(\bx_{0}^*,1)$, the problem must satisfy some boundary conditions. For more details, see \textcite[Chapter 11]{allgower2012numerical}.

Alternatively to \Cref{alg:cp_alg}, some methods approach the problem by varying the $s$ parameter directly. Particularly, one can write $H(c(s)) = \zeros$ and differentiate both sides with respect to $s$, combining the resulting equation with the requirement that $H(\bx_{0}^*,1) = \zeros$ to arrive at an initial value problem. Approximately solving the resulting differential equation can then yield a solution close to  $c(s)$. One of the main branches of numerical continuation methods, called \emph{predictor-corrector} methods, alternate between coarse numerical integration of the differential equation and using a local stabilizing step to eventually arrive at a solution close to $(\bx^*,0)$. Another main branch of homotopy methods, the \emph{piecewise linear} methods, instead consider piecewise linear approximation of the homotopy map. These are more general than \emph{predictor-corrector}, but they tend to perform worse in cases where both are applicable. Our discussion, nevertheless, will be centered in approaches following \Cref{alg:cp_alg}.

For a more detailed presentation of numerical continuation methods, see \textcite{allgower2012numerical}. \textcite[Chapter Homotopy Methods, Section 2]{horst2013handbook} additionally list many different applications of CP methods for engineering, economics and overall mathematical problems.

\subsection{Application to PB Optimization}
\label{sec:pb_application}

There are multiple ways to apply CP methods for PB optimization. One could, for example, take \Cref{eq:multi_poly} and use $\tau$ as a multiplicative factor zeroing out the higher order terms on the easier problems, then recovering them as $\tau \rightarrow 0$. Alternatively, one could attempt to approach the problem as constrained optimization, use Lagrange multipliers and then have $\tau$ control the strength of the constraints. We will, however, focus on parametrizing $\bz$ as a sigmoid. This idea surfaces by firstly noting that
\begin{equation*}
	\min_{\bz \in \corners} J(\bz) = \min_{\bx \in \R^d} J(\ind[\bx \geq 0])
\end{equation*}
where $\ind[\cdot]$ is the indicator function, which is applied element-wise and can represent any $\bz$ as long as we choose $\bx$ in the appropriate quadrant. Secondly, we note that
\begin{equation*}
	\lim_{\tau \rightarrow 0} \sigmoid \Big(\frac{x}{\tau}\Big) = \begin{cases} \ind[x > 0] & x \neq 0 \\ 0.5 & x=0 \end{cases}, \quad \text{where} \quad \sigmoid(x) = \frac{1}{1+\euler^{-x}}
\end{equation*}
is the logistic sigmoid function, which here we consider to be applied element-wise when the input is a vector. We will ignore the case $x = 0$, as it rarely occurs in practice when the algorithms get to $\tau \approx 0$. Combining the two observations, we can then define an auxiliary quantity $\tau$ and write the problem as
\begin{equation}
	\label{eq:sigmoid_cp}
	\text{minimize } \lim_{\tau \rightarrow 0} J\Big( \sigmoid \Big(\frac{\bx}{\tau}\Big)\Big), \qquad \text{for }\bx\in \R^d.
\end{equation}

\Cref{fig:sig_per_temp} shows how the plot of $\sigmoid(\cdot/\tau)$ changes as we vary the temperature $\tau$, as well as what happens to its derivative. Notably, the derivative becomes large near $x=0$ and almost zero everywhere else, causing gradient-based optimization to become infeasible for $\tau$ close to zero.

\begin{figure}[tb!]
	\centering
	\begin{tabular}{c}
		\includegraphics[height=1.9em]{\main/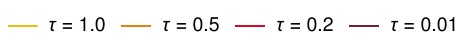}
	\end{tabular}
	
	\begin{tabular}{c c}
		\begin{subfigure}{0.4\columnwidth}
			\includegraphics[width=\columnwidth]{\main/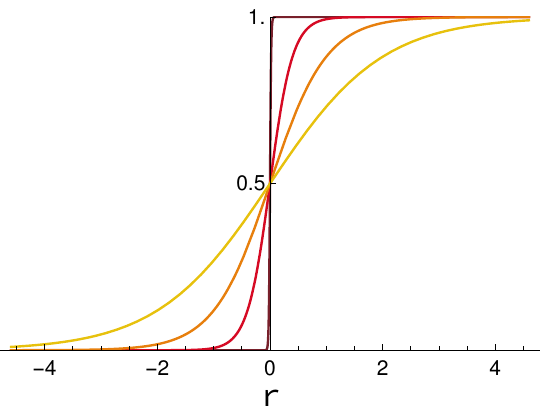}
			\caption{$\sigmoid \Big(\frac{x}{\tau}\Big)$}
		\end{subfigure}
		&
		\begin{subfigure}{0.4\columnwidth}
			\includegraphics[width=\columnwidth]{\main/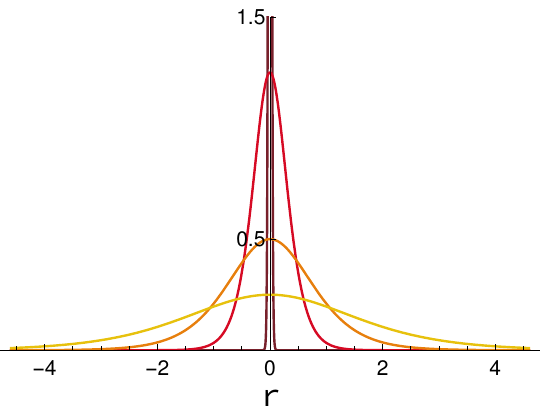}
			\caption{$\frac{\diff}{\diff x} \sigmoid \Big(\frac{x}{\tau}\Big)$}
		\end{subfigure}		
	\end{tabular}
	\caption{Effect of varying the temperature of a sigmoid.}
	\label{fig:sig_per_temp}
\end{figure}

Nonetheless, we can take an approach similar to the one from \Cref{sec:cp_basics}. Starting from  $\tau = 1$, which corresponds to a problem amenable to gradient-based optimization and iteratively changing $\tau$ to be closer to zero, we make the problem gradually closer to the desired one. Conveniently, the procedure also yields good starting solutions for the problems corresponding to the intermediate values of $\tau$. To make the connection to \Cref{sec:cp_basics} more explicit, for an arbitrary $\tau$, our goal is to find $\bx$ close to stationary, which satisfies
\begin{equation}
	\label{eq:}
	\nabla_\bx J\Big( \sigmoid \Big(\frac{\bx}{\tau}\Big)\Big) \approx \zeros.
\end{equation}
We reiterate that the gradient might not ever be exactly zero. For example, if $d = 1$ and $J(\cdot)$ is the identity mapping, the gradient can get arbitrarily close, but it will never be zero. Despite that, we can take:
\begin{equation*}
	H(\bx, \tau) = \nabla_\bx J\Big( \sigmoid \Big(\frac{\bx}{\tau}\Big)\Big)
\end{equation*}
and then follow \Cref{alg:cp_alg}. Instead of using Newton-Raphson, here it is more straightforward to simply use stochastic gradient descent (SGD).
%Care should be taken, since the solutions to $H(\bx, \tau) = \zeros$ must be local minimum, not maximum. This is likely to be the case for SGD, since it follows the direction opposite to the gradient.

Finally, the reader may have noted that we cannot use $\tau = 0$, as that involves division by zero. Be that as it may, a small enough $\tau$ will have the same effect, for limited computer precision will cause $\sigmoid(x / \tau) \in \{0,1\}$ if $x \neq 0$. Interestingly, this method is reminiscent of curriculum learning, where the learner starts with an easier problem that also becomes harder over time. \textcite{savarese2020winning,  yuan2020growing, luo2020autopruner} have successfully used this homotopy for pruning, whereas \textcite{azarian2020learned} also suggest it, but adopt a simplified version. \textcite{serra2018overcoming} have used it for sequential task learning.

\subsection{Drawbacks}
\label{sec:cp_drawbacks}

One concerning aspect of PB methods from \Cref{sec:pb_application} is the extrapolation of information from evaluations inside the hypercube (\ie, $\hypercube$) to the differences between evaluations at the vertices (\ie, $\corners$). To clarify, $\bx$ is initialized to some finite value, a reasonable choice being $\zeros$, so that $\sigmoid(x_i/\tau) = 0.5$, and SGD uses gradients evaluated there to arrive at a new $\bx$. The goal, however, is to find one of the vertices where $J(\cdot)$ is lower than in other elements from $\corners$. Since the gradient only indicates changes in a small neighbourhood around the current value, can one reliably use it to draw conclusions about $J(\cdot)$ in regions far from the current $\bx$?

Furthermore, as we mentioned in \Cref{sec:pb_basics}, any function whose evaluations on the corners are the same as from the multilinear polynomial will correspond to the same PB optimization problem (\Crefalt{eq:different_J} and the subsequent discussion). \Cref{fig:different_J} illustrates one such example, for $d=1$. The $J(\cdot)$ at hand will often not be the multilinear polynomial $\mP(\cdot)$ from \Cref{eq:multi_poly}. Since these derivatives can be completely different depending on the choice of $J(\cdot)$ is their information really relevant? For example, the derivative at $x=0$ can make it seem that the right choice is to go towards $z = 1$, but the (negative) gradient might as well have pointed to $z=0$ had a different $J(\cdot)$ been used. With all of these in mind, we present some failure cases for CP methods implementing the ideas from \Cref{sec:pb_application}.

\begin{figure}[tb!]
	\centering
	\includegraphics[width=0.95\columnwidth]{\main/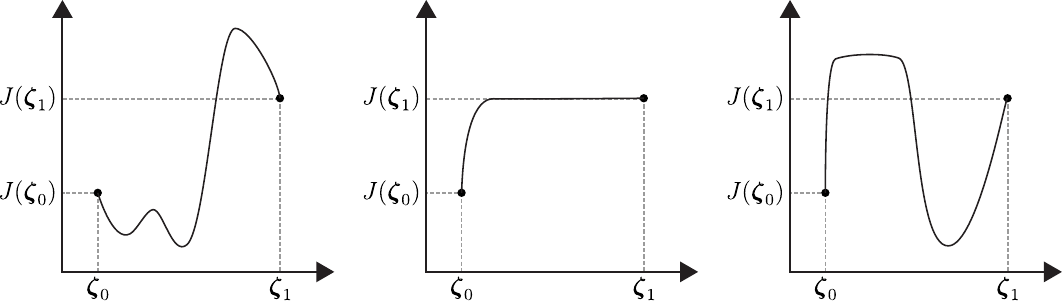}
	\caption{Completely different choices of $J(\cdot)$ lead to the same PB optimization problem.}
	\label{fig:different_J}
\end{figure}

\begin{figure}[tb!]
	\centering
	\begin{subfigure}{0.4\columnwidth}
		\centering
		\includegraphics[width=0.7\columnwidth]{\main/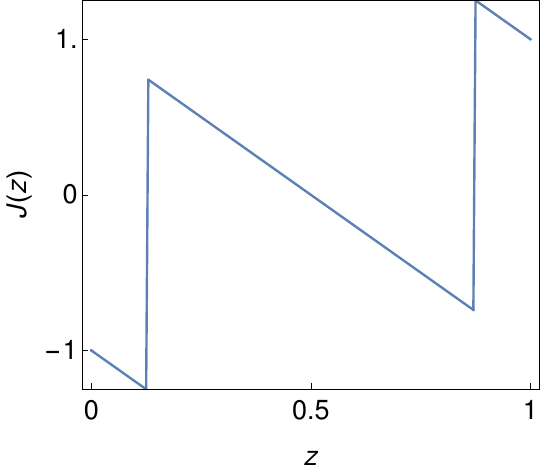}
		\caption{\Cref{ex:deceiving_piecewise}}
		\label{sb:deceiving_piecewise}
	\end{subfigure}%
	\begin{subfigure}{0.4\columnwidth}
		\centering
		\includegraphics[width=0.7\columnwidth]{\main/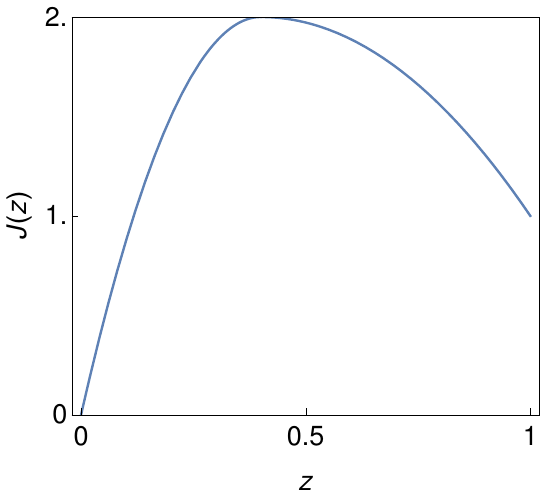}
		\caption{\Cref{ex:deceiving_squared}}
		\label{sb:deceiving_squared}
	\end{subfigure}		
	\caption{$J(\cdot)$ for the counter examples.}
	\label{fig:sig_per_temp}
\end{figure}

\begin{example}
	\label{ex:deceiving_piecewise}
	
	%	\begin{equation*}
		%		J(z) = \begin{cases}
			%			-1 -2 z & z < 0.125 \\
			%			-51 + 398 z & 0.125 \leq z < 0.13 \\
			%			1 -2 z & 0.13 \leq z < 0.87 \\
			%			-347 +398 z & 0.87 \leq z < 0.875 \\
			%			3 -2 z & 0.875 \leq z						
			%		\end{cases}.
		%	\end{equation*}
	
	Consider the function shown in \Cref{sb:deceiving_piecewise}, whose expression we omit for conciseness. In this case, we have $J(0) < J(1)$. However, for most of the $(0,1)$ interval, the negative gradient points towards $z=1$. Taking conclusions about the behaviour of $J(\cdot)$ in $\{0,1\}$ using these derivatives is therefore incorrect. 
	
	Luckily, however, for $z \in [0.125,0.13)\cup[0.87,0.875)$, the negative gradient has a large magnitude and points towards $z=0$. Therefore, depending on the learning rate, the model might be able to reach these regions, which will in turn help lead it to the correct solution. \Cref{fig:deceiving_piecewise_runs} depicts running the CP algorithm for this loss starting with $x=0$ and training for $100$ iterations before updating $\tau$. On these examples, we use $J(\sigmoid(x/\tau))$ during training, but the curves only show $J(\ind[x > 0])$ at each timestep, since only $0$ and $1$ are valid solutions to the problem. Only the higher learning rates got to the correct solution.
	\begin{figure}[htb!]
		\centering
		\begin{subfigure}{0.25\columnwidth}
			\includegraphics[width=\columnwidth]{\main/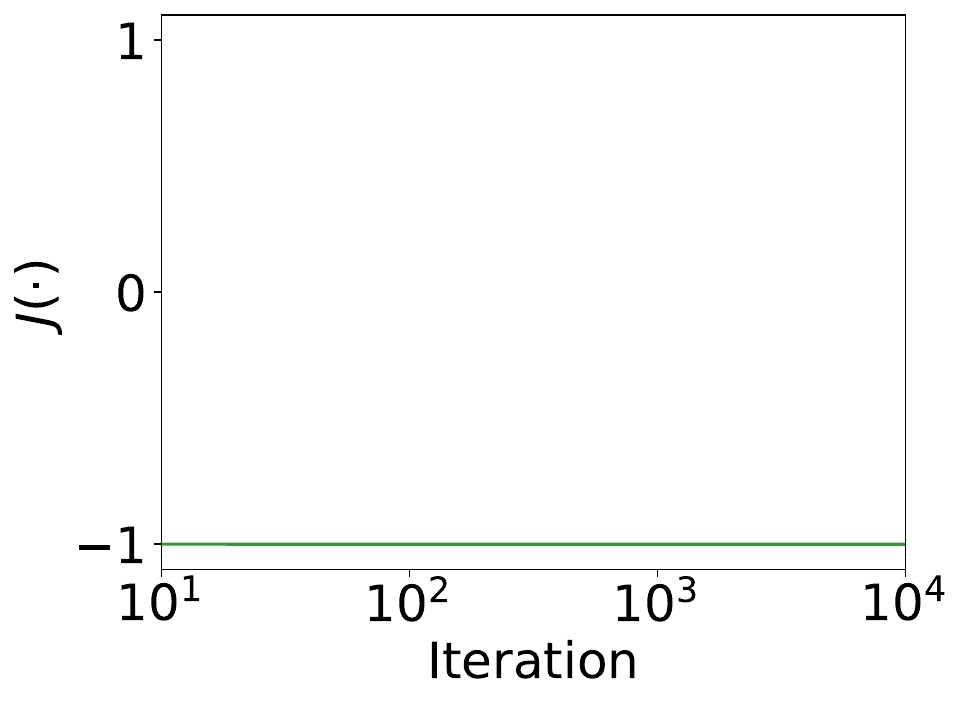}
			\caption{$\alpha = 0.5$}
		\end{subfigure}%
		\begin{subfigure}{0.25\columnwidth}
			\includegraphics[width=\columnwidth]{\main/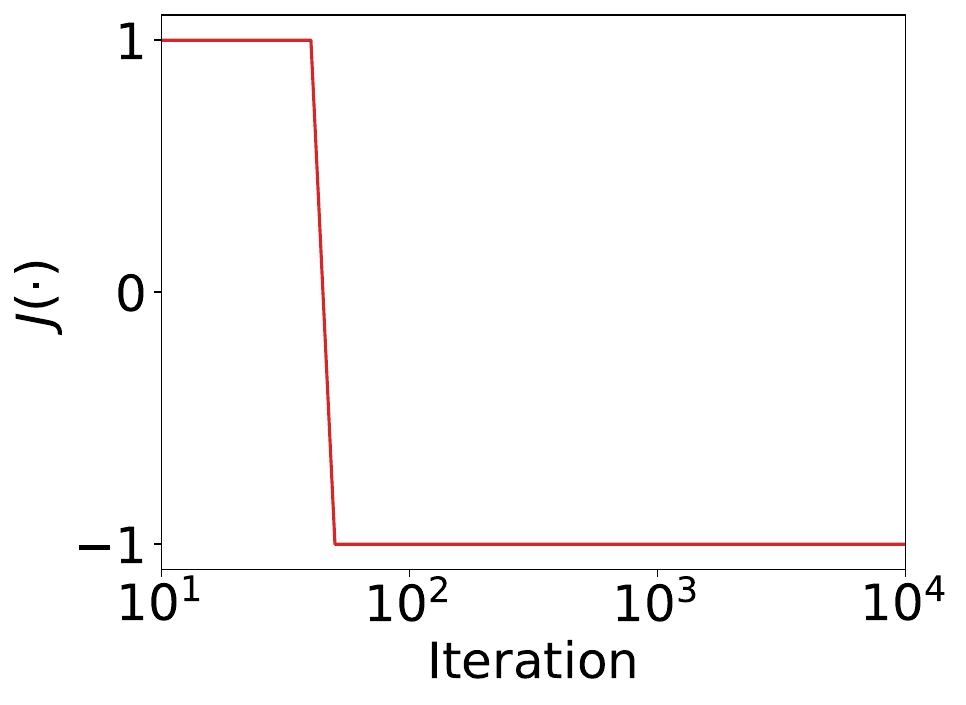}
			\caption{$\alpha = 0.1$}
		\end{subfigure}%
		\begin{subfigure}{0.25\columnwidth}
			\includegraphics[width=\columnwidth]{\main/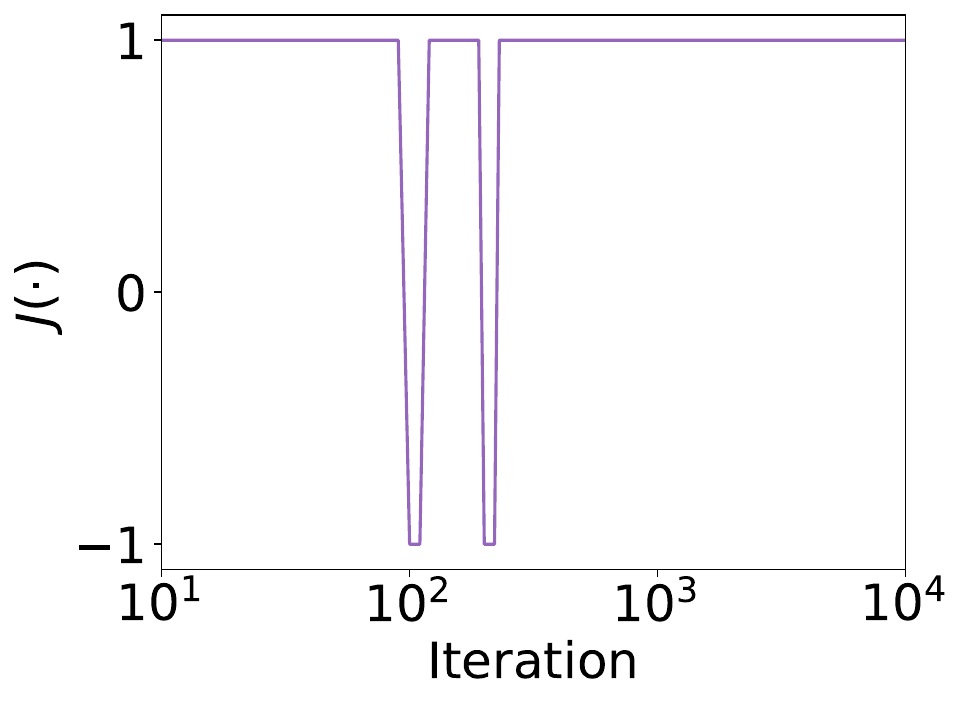}
			\caption{$\alpha = 0.05$}
		\end{subfigure}%
		\begin{subfigure}{0.25\columnwidth}
			\includegraphics[width=\columnwidth]{\main/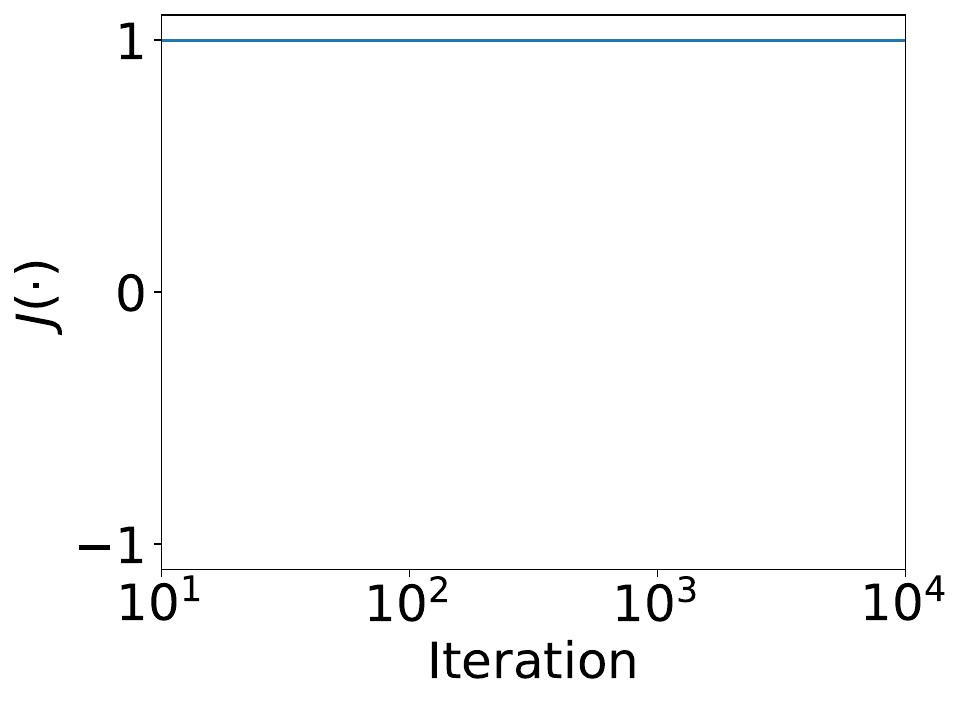}
			\caption{$\alpha = 0.01$}
		\end{subfigure}%
		\caption{Performance of different learning rates ($\alpha$) for the problem in \Cref{ex:deceiving_piecewise}.}
		\label{fig:deceiving_piecewise_runs}
	\end{figure}
	
\end{example}

\begin{example}
	\label{ex:deceiving_squared}
%	\begin{equation*}
	%		J(z) = \begin{cases}
		%			2 - 12.5 (z - 0.4)^2 & z < 0.4 \\
		%			2 - \frac{1}{0.36}(z - 0.4)^2 & 0.4 \leq z		
		%		\end{cases}.
	%	\end{equation*}	
	Consider the function depicted in \Cref{sb:deceiving_squared}, whose expression is again omitted for conciseness. Whenever $z = \sigmoid(x/\tau)$ is greater than $0.4$, the derivative is negative and gradient descent will go towards $z=1$, even though the correct solution is $z=0$. This will happen regardless of the learning rate being used. For the CP algorithm to get to the correct solution, it has to be initialized with $x_{init} < \logit(0.4)$. Since in general we cannot assume prior knowledge about the problem, it is customary to use $x_{init} = 0$, making it impossible to solve the problem with this CP method. For problems like the ones in this section, initializing the method in the right region amounts to knowing what the solution is before running the method, which would defeat the purpose of using CP.
\end{example}

On a more positive note, however, \textcite[Proposition B.5.]{shekhovtsov2021reintroducing} derive some conditions under which extrapolations from $\hypercube$ can be reasonable. Their bound is (indirectly) associated with the Lipschitz constant of $J(\cdot)$, a value representing how fast the function can change. Revisiting \Cref{ex:deceiving_squared}, the reason we had $J(0)$ less than $J(1)$ was because the left side of the curve started moving towards $-\infty$ quicker than the right side. If, instead, the function was unable to change as fast, approximating differences between vertices with gradients inside the hypercube might have been reasonable. Still, understanding how current models control this rate of change and how it interacts with CP methods to produce reasonable solutions still requires more investigation.

\section{Monte Carlo Gradient Estimation}
\label{ch:monte_carlo}
This section starts by presenting an alternative way of writing pseudo-Boolean optimization by optimizing probabilities instead (\Cref{sec:prob_framework}). Then, in \Cref{sec:using_samples}, we talk about how to use Monte Carlo estimation on this framework. \Cref{sec:mc_overview} starts by categorizing MC gradient estimation methods and explaining why some of them cannot be used here, then explaining the ones this work will focus on. \Cref{sec:mc_drawbacks} presents counter-examples elucidating some of the drawbacks of the stochastic approach, as well as presenting a bound that the loss must satisfy for the problem to be solvable, regardless of the variance reduction achieved. Finally, \Cref{sec:alternative_parametr} introduces alternative ways of parametrizing the probabilities, all of which will be investigated in this work.
\subsection{Basics}
The current section presents the basic framework used as basis for all methods from this section by firstly showing how to write the PB optimization problem using probabilities. Then we discuss how to make such a framework more practical by estimating expectations with samples.
\subsubsection{A Probabilistic Framework}
\label{sec:prob_framework}
As mentioned before, in pseudo-Boolean optimization, we want to find:
\begin{gather*}
	\bz^* = \argmin{\bz \in \corners} J(\bz) = \argmin{\bz \in \corners} \mP_J(\bz).
\end{gather*}
The probabilistic approach involves parametrizing the dimensions as factorized Bernoulli random variables and optimizing their expectation instead. The equivalence becomes more explicit after noticing the following:
\begin{theorem}[Restatement of Proposition 5 of {\citealt[Section 4.2]{boros2002pseudo}}]
	Assuming a vector $\btheta \in \hypercube$ and considering $\bz \in \corners$ such that $z_i \sim Ber(\theta_i)$ and $z_i$, $z_j$ independent for $i \neq j$, with $\mP_J(\cdot)$ as defined in \Cref{sec:pb_basics} we have:
	\begin{equation}
		\label{eq:equivalence}
		\E{z_i \sim \Ber[\theta_i]}{J(\bz)} = \mP_J(\btheta).
	\end{equation}
\end{theorem}
We rewrite the objective as:
\begin{equation}
	\label{eq:prob_objective}
	\text{minimize }\E{z_i \sim \Ber[\theta_i]}{J(\bz)}, \qquad \text{for }\btheta\in \hypercube		.
\end{equation}
It is straightforward to verify the equivalence. On the one hand, by \Cref{eq:equivalence}, the optimization in \Cref{eq:prob_objective} has the same objective function as the previous one: $\mP_J(\cdot)$, but now it considers the whole hypercube, not only the vertices. Therefore:
\begin{align*}
	\mymin{\btheta\in \hypercube} \mP_J(\btheta) \leq \mymin{\bz \in \corners} \mP_J(\bz) \\
	\mymin{\btheta\in \hypercube} \E{z_i \sim \Ber[\theta_i]}{J(\bz)} \leq \mymin{\bz \in \corners} J(\bz).
\end{align*}
On the other hand, any such expectation is a weighted average of the evaluations in the corners and thus cannot be smaller than the lowest of those values:
\begin{equation*}
	\mymin{\btheta\in \hypercube} \E{z_i \sim \Ber[\theta_i]}{J(\bz)} \geq \mymin{\bz \in \corners} J(\bz).
\end{equation*}
Therefore, the two minima have to be the same. The main change from the original problem is that, in case of multiple maximizing $\bz$ values, the probabilistic form will also include solutions that sample randomly between them. \textcite[Theorem 1]{daulton2022bayesian} formalize this result and extend it to ordinal and categorical formulations.

Using the probabilistic objective instead of the discrete one benefits from the fact that we can now piggyback on advances from related fields which solve similar problems. One can, for example, attempt to discover latent structure on some simpler space and find the solution using variational inference \parencite{blei2017variational}. Alternatively, one can also attempt to adapt novelty search approaches from machine learning methods or exploration strategies from reinforcement learning \parencite{ladosz2022exploration}. In this section, we will focus on applying Monte Carlo methods, described in the next parts.

\subsubsection{Using Samples}
\label{sec:using_samples}
Although we can rewrite the problem using the probabilistic formulation, we remind the reader that the expectation in \Cref{eq:prob_objective} consists of a sum of $2^d$ terms and therefore still suffers from the problems of the original formulation. Nonetheless, we can use Monte Carlo methods, which approximate the expected value by sampling. In its simplest form, we could, for example, approximate the expected cost by using
\begin{equation}
	\label{eq:simplest_mc}
	\E{z_i \sim \Ber[\theta_i]}{J(\bz)} \approx \frac{1}{n} \sum_{s=1}^{n} J(\bzs), \qquad \text{for }\bzs \sim \prod_{i=1}^d \Ber[\theta_i],
\end{equation}
where $\bzs$ indicates the $s$-th sample and $\bz \sim \prod_{i=1}^d \Ber[\theta_i]$ is an alternative notation to indicate that $z_i$ are independent Bernoulli random variables. Here, in addition to the $d$ per-sample coordinates (\ie, $\zsi$ for sample $s$) being independent, the $n$ samples (\ie, $\bzs$ for $s\in\{1\dots n\}$) are iid.

By the law of large numbers, the RHS of \Cref{eq:simplest_mc} will converge to the true expectation as $n \rightarrow \infty$. We can also show that\footnote{Where $\{\bzs \sim p_{\btheta}(\cdot)\}_{s=1}^n$ denotes that $\bzs \sim p_{\btheta}(\cdot)$ for $s \in \{ 1, \dots, n \}$. We will also use $\{\bzs\}_{s=1}^n$ to denote $\{\bz^{(1)}, \dots, \bz^{(n)}\}$.}
\begin{align}
	\E{ \{\bzs \sim p_{\btheta}(\cdot)\}_{s=1}^n }{\frac{1}{n} \sum_{s=1}^{n} J(\bzs)} = \E{z_i \sim \Ber[\theta_i]}{J(\bz)} \quad \text{and} \quad \Var{\frac{1}{n} \sum_{s=1}^{n} J(\bzs)} &= \frac{1}{n}\Var{J(\bz)} \text{.}
	\label{eq:var_simple_mc}
\end{align}
Alternative estimators to that of the RHS in \Cref{eq:simplest_mc} are also possible, where each technique has its own pros and cons. In case the expected value of the estimator is equal to the desired quantity, the estimator is called unbiased. In some cases, a biased estimator may have lower variance or better convergence rate than unbiased alternatives. In this work, however, we will focus on unbiased estimators. Monte Carlo estimation is an active area of research and techniques improving estimation include: control variates, Rao-Blackwellization, stratification, importance sampling and antithetic sampling. For a more in-depth introduction to these, refer to \textcite{owen2013monte}.

In our problem, however, we do not simply aim at estimating the objective for some fixed $\btheta$, but at finding the minimizing $\btheta$. For this, we can apply a SGD procedure where we choose an initial $\btheta_0$, estimate the gradient of the objective and use it to update $\btheta$. Note, however, that, on typical SGD usage in machine learning problems we have that
\begin{equation*}
	\nabla_{\bw} \E{X,Y \sim \sD_s}{\ell(f(X, \bw), Y)} \approx \frac{1}{B}\sum_{b=1}^{B} \nabla_{\bw} \ell(f(X^{(b)}, \bw), Y^{(b)}) \text{.}
\end{equation*}
This estimator is unbiased because the gradient can simply be moved outside (or inside) the expectation:
\begin{align*}
	\E{X^{(b)},Y^{(b)} \sim \sD_s}{\frac{1}{B}\sum_{b=1}^{B} \nabla_{\bw} \ell(f(X^{(b)}, \bw), Y^{(b)})} &=  \nabla_{\bw} \E{X^{(b)},Y^{(b)} \sim \sD_s}{\frac{1}{B}\sum_{b=1}^{B} \ell(f(X^{(b)}, \bw), Y^{(b)})} \\
	&= \nabla_{\bw} \E{X,Y \sim \sD_s}{\ell(f(X, \bw), Y)} \text{.}
\end{align*}
Interchanging the order of expectation and gradient is only possible because there is no dependency between the sampling probabilities and the desired parameter (\ie, $\sD_s$ is not a function of $\bw$ in the above example). If there was such a dependency, the estimator would change. For example:
\begin{align}
	\label{eq:score_der}
	\nabla_{\btheta} \E{\bz \sim p_{\btheta}(\cdot)}{f(\bz, \btheta)} &= \nabla_{\btheta} \int p_{\btheta}(\bz) f(\bz, \btheta) \diff \bz \nonumber \\
	&= \int \left(\nabla_{\btheta} p_{\btheta}(\bz)\right) f(\bz, \btheta) + p_{\btheta}(\bz) \nabla_{\btheta} f(\bz, \btheta) \diff \bz \nonumber \\
	&= \int p_{\btheta}(\bz) \left(\nabla_{\btheta} \log p_{\btheta}(\bz)\right) f(\bz, \btheta) + p_{\btheta}(\bz) \nabla_{\btheta} f(\bz, \btheta) \diff \bz \nonumber \\
	&= \E{\bz \sim p_{\btheta}(\cdot)} {f(\bz, \btheta) \nabla_{\btheta} \log p_{\btheta}(\bz)}
	+ \E{\bz \sim p_{\btheta}(\cdot)}{\nabla_{\btheta} f(\bz, \btheta)} \nonumber \\
	&\neq \E{\bz \sim p_{\btheta}(\cdot)}{\nabla_{\btheta} f(\bz, \btheta)}.
\end{align}
Estimating $\nabla_{\btheta} \E{}{f(\bz, \btheta)}$ with:
\begin{align*}
	\frac{1}{n}\sum_{s=1}^{n} \nabla_{\btheta} f(\bzs, \btheta) \qquad\text{for }s=\{1,\dots,n\}
\end{align*}
like before would only account for $\E{}{\nabla_{\btheta} f(\bz, \btheta)}$. The cost function in \Cref{eq:prob_objective} presents a dependency of the probability on the gradient parameter. Therefore, gradient estimators have to account for that. Particularly, given samples $\bz^{(1)}, \dots, \bz^{(n)}$ and denoting $(\bz^{(s)})_{s=1}^n$ as the ordered set containing these samples, we want $\bgh(\cdot)$ such that:
\begin{equation*}
	\E{}{\bgh((\bz^{(s)})_{s=1}^n; \btheta)} = \nabla_{\btheta} \E{z_i \sim \Ber[\theta_i]}{J(\bz)}
\end{equation*}
\Cref{sec:mc_overview} will present some different $\bgh(\cdot)$ to estimate the gradient. When estimating scalars, such as in \Cref{eq:simplest_mc}, different unbiased estimators are usually chosen based simply on how much they can reduce the variance in the desired problems. For gradient estimation, nonetheless, $d$ is usually greater than one and there are multiple scalar variances, one per dimension. In that case, variance is usually considered element-wise or as a sum of these $d$ individual scalar variances. This sum is equal to the expected squared $L_2$-norm of the vector corresponding to the error between the estimated gradient and the true gradient.

\subsection{Methods}
\label{sec:mc_overview}
Firstly, we give a broad overview of Monte Carlo gradient estimation, dividing them in three main categories: score function estimation, pathwise gradient estimation and measure-valued gradients. Then, we detail the methods used in our experiments: REINFORCE, LOORF, ARMS and $\beta^*$. Given the varied scale of our experiments---ranging from tens to thousands of samples---we derived iterative versions of the estimators. For REINFORCE and LOORF this mostly involved modifications in how values are accumulated, whereas ARMS required the derivation of an equivalent sampling procedure. $\beta^*$ is a closed-form estimator which we derived based on the optimal constant from \textcite[Section 8.9]{owen2013monte}. We only used it in the smaller settings, as it is intractable to compute otherwise.

\subsubsection{Overview}
On a recent survey, \textcite{mohamed2020monte} categorized approaches for MC gradient estimation, arriving in three fundamentally distinct strategies: score function estimation, pathwise gradient estimation and measure-valued gradients. The first one follows the derivation outlined in \Cref{eq:score_der}, substituting $f(\bz, \btheta)$ for $J(\bz)$, this estimator becomes:
\begin{align}
	\label{eq:score_function_est}
	\nabla_{\btheta} \E{\bz \sim p_{\btheta}(\cdot)}{J(\bz)} &= \E{\bz \sim p_{\btheta}(\cdot)} {J(\bz) \nabla_{\btheta} \log p_{\btheta}(\bz)} \nonumber \\
	&\approx \frac{1}{n} \sum_{s=1}^n J(\bzs) \nabla_{\btheta} \log p_{\btheta}(\bzs).
\end{align}
For our purposes, $p_{\btheta}(\cdot)$ is a discrete distribution, but the estimator is also applicable in the continuous case.\footnote{\textcite[Section 4.3.2]{mohamed2020monte} explain that this estimator can be biased and also give a simple example. For our use cases, however, it is unbiased.} In practice, this estimator is known to have high variance. In some cases, it is possible to instead use pathwise gradient estimators, also known as reparametrization trick, by writing:
\begin{equation}
	\bz = \fz(\Eps; \btheta) \qquad\text{for }\Eps \sim p(\cdot).
	\label{eq:rpm_cond}
\end{equation}
Where the underscript in $f_z(\cdot)$ is merely to indicate that it maps to $\bz$ (\ie, not a parameter relation, like in $p_{\btheta}(\cdot)=p(\cdot; \btheta)$). Note that $p(\cdot)$ above no longer depends on $\btheta$. In that case, we have:
\begin{align}
	\nabla_{\btheta} \E{\bz \sim p_{\btheta}(\cdot)}{J(\bz)} &= \nabla_{\btheta} \E{\Eps \sim p(\cdot)}{J(\fz(\Eps; \btheta))} \nonumber \\
	&= \E{\Eps \sim p(\cdot)}{\nabla_{\btheta} J(\fz(\Eps; \btheta))} \nonumber \\
	&= \E{\Eps \sim p(\cdot)}{\nabla_{\btheta} \fz(\Eps; \btheta) (\nabla_{\bz} J(\bz))\big\rvert_{\bz = \fz(\Eps; \btheta)} } \nonumber \\
	&\approx \frac{1}{n} \sum_{s=1}^n \nabla_{\btheta} \fz(\Eps^{(s)}; \btheta) (\nabla_{\bz} J(\bz))\big\rvert_{\bz = \fz(\Eps^{(s)}; \btheta)}.
	\label{eq:rpm_function_est}
\end{align}
One very popular distribution where it is possible to write $\bz$ as in \Cref{eq:rpm_cond} is the multivariate Normal. This method is well-known empirically for having lower variance compared to the score function estimator from \Cref{eq:score_function_est}, with successes in generative modeling \parencite{kingma2013auto} and reinforcement learning \parencite{haarnoja2018soft}. More recently \textcite{lan2021model} derived alternatives to some of the main theorems in reinforcement learning to use pathwise gradient estimators instead of the score function counterparts.

Interestingly, the variance of the reparametrized estimators can be shown to be bounded by the squared Lipschitz constant of the cost function $J(\cdot)$ (\citealt[Section 7.2.2]{glasserman2004monte} and \citealt[Section 10]{fan2015fast}) without the presence of the dimension $d$ in the bound. Conversely, as we will see in \Cref{sec:mc_drawbacks}, the dimensionality of $\bz$ can have catastrophic consequences for the score function estimators from \Cref{eq:score_function_est}. \textcite{xu2019variance} compares both estimators theoretically in a simplified variational inference setting, attributing the increased variance of the score function estimators to the presence of higher order terms in the variance formulas. These successes led to generalizations of the reparametrization trick, including relaxations allowing $p(\cdot)$ in \Cref{eq:rpm_cond} to also be a function of the $\btheta$ parameter,\footnote{In this case, more terms appear in the formula to correct for the dependency.} which in turn allowed the method to be applicable to the Beta, Gamma and Dirichlet distributions \parencite{ruiz2016generalized, naesseth2017reparameterization}.

Nonetheless, this method requires differentiable cost functions and continuous inputs, while $\bz \in \corners$ in our problem setting. As mentioned in \textcite{bengio2013estimating}, a reparametrization exists for Bernoulli variables, since, for $\bu \sim \prod_{i=1}^{d}U[0,1]$:
\begin{align}
	\nabla_{\btheta} \E{z_i \sim \Ber[\theta_i]}{J(\bz)} = \nabla_{\btheta} \E{u_i \sim U[0,1]}{J(\ind[\bu \leq \btheta])} = \E{u_i \sim U[0,1]}{\nabla_{\btheta} J(\ind[\bu \leq \btheta])}. \label{eq:bernoulli_rpm}
\end{align}
Still, similarly to what happened in \Cref{sec:pb_application}, the coordinates of the gradient of $J(\ind[\bu \leq \btheta])$ are undefined for $\theta_i = u_i$ and zero everywhere else, rendering the reparametrization trick unusable.

Moreover, variance when using reparametrization is not guaranteed to be lower than when using the score function. As mentioned before, it depends on the Lipschitz constant of $J(\cdot)$, which can act as a double-edged knife and cause them to perform worse than the score function methods \parencite[figure 3, Section 3 and Section 5.3.2]{mohamed2020monte}. Despite all that, multiple methods attempt to incorporate the reparametrization trick in discrete settings somehow, culminating in estimators that are hybrid, with characteristics from both the CP approaches from \Cref{sec:pb_application} as well as approaches from the current section. These approaches will be discussed in \Cref{ch:other}.

Finally, measure-valued gradients, which correspond to the third category, are not very common in machine learning. They are also unbiased and their derivation uses Hahn decomposition theorem to write the gradient of some signed probability measure as a difference of two unsigned measures \parencite[Chapter 6]{mohamed2020monte}. However, these estimators require $O(nd)$ evaluations of $J(\cdot)$, the same problem the PB optimization methods from \Cref{sec:pb_algs} suffer from. For the reasons discussed, this section will focus only on score function estimators.

\subsubsection{REINFORCE}
REINFORCE \parencite{williams1992simple}, which also sometimes appears in the literature with alternative names \parencite{rubinstein1990optimization, glynn1990likelihood}, corresponds to the simplest form of score function estimator, which we derived and discussed above. For convenience, we repeat its equation below:
\begin{equation}
	\bgh_{REINFORCE}((\bz^{(s)})_{s=1}^n; \btheta) = \frac{1}{n} \sum_{s=1}^n J(\bzs) \nabla_{\btheta} \log p_{\btheta}(\bzs), \qquad\text{for }\bzs \sim p_{\btheta}(\cdot).
	\label{eq:reinforce}
\end{equation}
In our case, we also have:
\begin{align}
	p_{\btheta}(\bz) = \prod_{i=1}^d \Ber[\theta_i] = \prod_{i=1}^d \paddedtheta_i^{z_i} \ntheta_i^{\nz_i}. \label{eq:reinforce_p}
\end{align}
For our experimental section, each $\bz$ will be a mask applied element-wise to large neural network weights. To avoid storing all $n$ samples in memory simultaneously, which would have memory cost $O(nd)$, we resort to iterative versions of the estimators, changing the memory cost to $O(d)$. For REINFORCE, this simply corresponds to following \Cref{alg:iter_reinforce}.
\begin{algorithm}
	\caption{Iterative form of REINFORCE}
	\label{alg:iter_reinforce}
	\begin{algorithmic}
		\State $\bgh \leftarrow \zeros$;
		\For{$s=1 \cdots n$}
		\State Sample $\bzs$ according the distribution from \Cref{eq:reinforce_p} 
		\State $\bgh \leftarrow \frac{s-1}{s}\bgh + \frac{1}{s} J(\bzs) \nabla_{\btheta} \log p_{\btheta}(\bzs)$
		\EndFor
		\State \textbf{Return:} $\bgh$
	\end{algorithmic}
\end{algorithm}

As mentioned above, REINFORCE has very high variance. For this reason, it is common to apply some form of variance reduction technique to $\bgh_{REINFORCE}$. One of the main ones, which will be used in most of the methods discussed in this paper, is called control variates. Briefly going back to the case of estimating $\E{}{J(\bz)}$ instead of its gradient, notice that, for arbitrary $h(\cdot)$ and considering:
\begin{equation*}
	\E{\bz \sim p_{\btheta}(\cdot)}{h(\bz)} = \mu_h,
\end{equation*}
where $\mu_h$ a known constant,\footnote{We omit its dependence on $\btheta$ to simplify the notation.} we can write:
\begin{align*}
	\E{\bz \sim p_{\btheta}(\cdot)}{J(\bz)} &= \E{\bz \sim p_{\btheta}(\cdot)}{J(\bz) - h(\bz)} + \E{\bz \sim p_{\btheta}(\cdot)}{h(\bz)} \\
	&\approx \left(\frac{1}{n} \sum_{s=1}^n J(\bzs) - h(\bzs)\right) + \mu_h.
\end{align*}
The above quantity will be an unbiased estimator of $\E{}{J(\bz)}$ and its variance is equal to:
\begin{equation*}
	\frac{1}{n} \Var{J(\bz) - h(\bz)}.
\end{equation*}
Comparing it to \Cref{eq:var_simple_mc}, we can see that, if the difference $J(\bz) - h(\bz)$ has smaller variance than $J(\bz)$, using control variates will be beneficial. Going back to estimating gradients, we denote:
\begin{align*}
	\bg_h = \nabla_{\btheta} \E{\bz \sim p_{\btheta}(\cdot)}{h(\bz)}
\end{align*}
and similarly, by adding and subtracting this quantity, we can arrive at:
\begin{align}
	\label{eq:main_control_variates}
	\nabla_{\btheta} \E{\bz \sim p_{\btheta}(\cdot)}{J(\bz)} &= \E{\bz \sim p_{\btheta}(\cdot)} {(J(\bz) - h(\bz)) \nabla_{\btheta}\log p_{\btheta}(\bz)} + \nabla_{\btheta} \E{\bz \sim p_{\btheta}(\cdot)}{h(\bz)} \nonumber \\
	&\approx \left(\frac{1}{n} \sum_{s=1}^n (J(\bzs) - h(\bzs)) \nabla_{\btheta} \log p_{\btheta}(\bzs)\right) + \bg_h.
\end{align}
By using this principle, some methods select $h(\cdot)$ close to $J(\cdot)$, but whose corresponding $\bg_h$ is feasible to compute, such as a Taylor expansion of $J(\cdot)$ \parencite{gu2015muprop}. More generally, other methods subtract and re-add the expectation, but use a reparametrization estimator instead of $\bg_h$, combining score function and pathwise gradient estimation and resulting in hybrid methods \parencite{tucker2017rebar, grathwohl2017backpropagation}. As mentioned before, some of these methods will be discussed in \Cref{ch:other}. In this section, we consider control variates which do not rely on $\nabla J(\cdot)$.

\subsubsection{LOORF}
Proposed initially by \textcite{kool2019buy}, this method arises by deriving a control variate similarly to \Cref{eq:main_control_variates}, but now considering all of the samples simultaneously when designing the baseline to be subtracted. Particularly, for a single $s \in \{1,\dots,n\}$, we denote the ordered set containing all other samples by
\begin{equation*}
	\OrdSetMinusZs = (\bzsp \mid s' \in \{1, \dots, n\} \setminus \{s\}).
\end{equation*}
Then, considering $\bz^{(1)}, \dots, \bz^{(n)}$ iid, we have:
\begin{align}
	\E{\bzs \sim p_{\btheta}(\cdot)}{ J(\bzs) \nabla_{\btheta} \log p_{\btheta}(\bzs)} &= \E{}{(J(\bzs) - h(\OrdSetMinusZs))
		\nabla_{\btheta}\log p_{\btheta}(\bzs)} + \nonumber \\
	&\hspace{5em} \E{}{ h(\OrdSetMinusZs) \nabla_{\btheta} \log p_{\btheta}(\bzs)}
	\nonumber \\
	&= \E{}{(J(\bzs) - h(\OrdSetMinusZs))
		\nabla_{\btheta}\log p_{\btheta}(\bzs)} + \nonumber \\ \nonumber
	&\hspace{5em} \E{}{ h(\OrdSetMinusZs)} \cancelto{0}{\E{}{\nabla_{\btheta} \log p_{\btheta}(\bzs)}} \\ 
	&= \E{}{(J(\bzs) - h(\OrdSetMinusZs))
		\nabla_{\btheta}\log p_{\btheta}(\bzs)}. \label{eq:intermediate_loorf}
\end{align}
Where we used that the expected value of the score function is zero. This means that we can use any function of the other $n-1$ samples to compose the control variate. To keep $h(\cdot)$ and $J(\cdot)$ close, \textcite{kool2019buy} use:
\begin{equation*}
	h(\OrdSetMinusZs) = \frac{1}{n-1}\sum_{s' \neq s} J(s'),
\end{equation*}
for iid $\bzs \sim p_{\btheta}(\cdot)$. By linearity of expectations, we can average \Cref{eq:intermediate_loorf} for all $s$ and the estimator then becomes:
\begin{align}
	\bgh_{LOORF}((\bz^{(s)})_{s=1}^n; \btheta) &= \frac{1}{n} \sum_{s=1}^n \Big(J(\bzs) - \frac{1}{n-1} \sum_{s' \neq s} J(\bzsp) \Big)\nabla_{\btheta} \log p_{\btheta}(\bzs) \nonumber \\
	&= \frac{1}{n-1} \sum_{s=1}^n \Big(J(\bzs) - \frac{1}{n} \sum_{s=1}^n J(\bzsp) \Big)\nabla_{\btheta} \log p_{\btheta}(\bzs).
	\label{eq:loorf}
\end{align}
Where the proof of equivalence between the two forms is shown in \textcite{kool2019buy}. Although this is a simple estimator, it is a strong baseline which has outperformed more advanced variance reduction techniques \parencite{dimitriev2021arms}. \textcite{richter2020vargrad} shows an alternative derivation of this estimator using a variational inference perspective and also enumerates conditions in which it behaves closely to the optimal control variate. Similarly to REINFORCE, we require a way of sampling iteratively, which is shown in \Cref{alg:iter_loorf}.
\begin{algorithm}
	\caption{Iterative form of LOORF}
	\label{alg:iter_loorf}
	\begin{algorithmic}
		\State Start accumulators:
		\State $\Jh \leftarrow 0$;
		\State $\Lambda_{\nabla \log} \leftarrow \zeros$;
		\State $\Lambda_{J \nabla \log} \leftarrow \zeros$;
		\For{$s=1 \cdots n$}
		\State Sample $\bzs$ according the distribution from \Cref{eq:reinforce_p} 
		\State $\Jh \leftarrow \frac{s-1}{s}\Jh + \frac{1}{s} J(\bzs) $
		\State $\Lambda_{\nabla \log} \leftarrow \frac{\mymax{}(s-2, 1)}{\mymax{}(s-1, 1)}\Lambda_{\nabla \log} + \frac{1}{\mymax{}(s-1, 1)} \nabla_{\btheta} \log p_{\btheta}(\bzs)$
		\State $\Lambda_{J \nabla \log} \leftarrow \frac{\mymax{}(s-2, 1)}{\mymax{}(s-1, 1)}\Lambda_{J \nabla \log} + \frac{1}{\mymax{}(s-1, 1)} J(\bzs) \nabla_{\btheta} \log p_{\btheta}(\bzs)$		
		\EndFor
		\State $\bgh \leftarrow \Lambda_{J \nabla \log} - \Lambda_{\nabla \log} \Jh$
		\State \textbf{Return:} $\bgh$
	\end{algorithmic}
\end{algorithm}
\subsubsection{ARMS}
This estimator extends the idea of antithetic sampling from Monte Carlo estimation to gradient estimation. The base principle is to use samples that are opposite in some way and rely on their error cancellation to reduce variance. To understand it, we once more go back to the case where we are estimating a scalar quantity. Considering a set $\sU$, say we want to estimate $\E{}{f(\bu)}$, for $\bu \in \sU$ (instead of $\corners$). Furthermore, assume that $\bu \sim p(\cdot)$, where $p(\cdot)$ is a symmetric density with respect to point $\bc$. Namely, we define a reflection of $\bu$ through $\bc$, here called $\Nu$, such that:
\begin{equation*}
	\bu - \bc = -(\Nu - \bc).
\end{equation*}
For $p(\cdot)$ to be symmetric, we must have:
\begin{equation*}
	p(\bu) = p(\Nu).
\end{equation*}
Notice that requiring such a density is not too restrictive. The uniform density on $\sU = \hypercube$ with $\bc = [0.5\enspace 0.5 \dots]^\top$ and $\Nu = \ones - \bu$, for example, satisfies this condition. This density is commonly used as the basis for sampling from some distributions, since applying specific functions to uniform random variables, such as when using inverse transform sampling or the Gumbel max trick, is often statistically equivalent to sampling from the desired distributions \parencite[Chapter 4]{owen2013monte}. The antithetic sampling estimate is obtained by:
\begin{align*}
	\E{\bu \sim p(\cdot)}{f(\bu)} \approx \frac{1}{n} \sum_{s=1}^{n/2} {f(\bus) + f(\Nus)}.
\end{align*}
The efficacy of this estimator for variance reduction will depend heavily on how $f(\cdot)$ behaves. Its variance is equal to:
\begin{align*}
	\frac{1}{n}\Var{f(\bu)}(1 + \Corr{f(\bu), f(\Nu)}),
\end{align*}
as opposed to $(1/n)\Var{f(\bu)}$ as in simple Monte Carlo estimation. In a monotonic function, for instance, the correlation should be closer to $-1$, causing antithetic sampling to be efficient. More generally, we can write $f(\cdot)$ as a sum of even part $f_E(\cdot)$ and odd part $f_O(\cdot)$:\footnote{Even and odd are meant with respect to $\bc$, not $\zeros$.}
\begin{align*}
	f(\bu) = \underbrace{\frac{f(\bu) + f(\Nu)}{2}}_{f_E(\bu)} + \underbrace{\frac{f(\bu) - f(\Nu)}{2}}_{f_O(\bu)}.
\end{align*}
The odd part has expectation zero and the even part has the same expectation as $f(\cdot)$. Antithetic sampling has benefit of eliminating the variance from the odd part, but has the drawback of doubling the variance from the even part. Going back to the discrete case, estimating the expected loss can be done as follows:
\begin{align*}
	\E{z_i \sim \Ber[\theta_i]}{J(\bz)} &= \E{u_i \sim U[0,1]}{J(\ind[\bu \leq \btheta] )} \\
	&\approx \frac{1}{n} \sum_{s=1}^{n/2} J(\ind[\bus \leq \btheta]) + J(\ind[\Nus \leq \btheta]).
\end{align*}
For $\bus \sim \prod_{i=1}^d U[0,1]$ and $\Nus = \ones - \bus$.

We now move the discussion back to gradient estimation. In the context of neural networks, the first paper that tried to incorporate this technique was ARM \parencite{yin2018arm}, which was later improved by two concurrent works: DisARM \parencite{dong2020disarm} and U2G \parencite{yin2020probabilistic}, both discovering an equivalent improvement of ARM independently. Notably, however, these methods seem to underperform LOORF for larger $n$ as noted by \textcite{dimitriev2021arms}. Specifically, this work argues that LOORF leverages all possible combinations of pairs among the $n$ samples, whereas ARM/U2G/DisARM only combine the antithetic pairs.

To use the idea of antithetic sampling while still taking advantage of all pairs, \textcite{dimitriev2021arms} then propose ARMS. They derive the estimator by first assuming $d=1$ and $n=2$ and then generalize to arbitrary $n$ and $d$. The authors start with LOORF, but use importance sampling to make sampling antithetic, introducing dependency between the variables and ending up with a generalization of DisARM. Normally, combining importance sampling and high dimensions can lead to catastrophic outcomes, but, by further restricting the importance distribution such that its marginals remain the same as in the nominal distribution, they keep the variance under control. Denoting the variable sampled from the modified distribution as $\zTil$, their final expression is as follows:
\begin{align}
	\bgh_{ARMS}((\bzTil^{(s)})_{s=1}^n; \btheta) &= \frac{1}{1-\brho} \bgh_{LOORF}((\bzTil^{(s)})_{s=1}^n; \btheta), \\ &\hspace{2em}\text{where }\rho_i = \Corr{\zTil_i^{(s)}, \zTil_i^{(s')}}\text{, for } s \neq s'. \nonumber
\end{align}
Importantly, for each dimension $i \in \{1, \dots, d\}$, they assume that the correlation between all pairs of scalar samples $(\zTil_i^{(s)}, \zTil_i^{(s')})$ will be the same value.

To obtain samples satisfying such a structure, they rely on copula sampling \parencite[Section 5.6]{owen2013monte}. To summarize, applying a CDF of some random scalar varible to that same variable results in an output distributed according to $U[0,1]$. Therefore, if we sample a $d$-dimensional random variable and apply each marginal CDF to the corresponding dimension, we get $d$ (marginally) uniform random variables. If the original distribution causes the $d$ original variables to have mutual negative dependence, the uniforms should keep some of this dependence. Then, these uniform variables can be used to produce Bernoulli samples by the reparametrization from \Cref{eq:bernoulli_rpm}.

\textcite{dimitriev2021arms} propose two different ways to get these uniforms: Dirichlet copula and Gaussian copula. On their experiments, the Dirichlet copula performed the best, and for that reason it is the one we are going to use in this work. The full procedure as proposed in their paper is summarized \Cref{alg:orig_arms}. As before, we require the algorithm to be iterative and therefore modify their sampling procedure\footnote{Note that $\sum_{s'=1}^{n} \log u_i^{(s')}$ in \Cref{alg:orig_arms} requires all $n$ samples to have been computed beforehand.} to the one outlined in \Cref{alg:iter_dir}, shown only for $d=1$ to simplify exposition. Proof of equivalence between the two sampling procedures is provided in \Cref{sec:arms_iter_proof}, as well as the full iterative algorithm to compute the ARMS estimator, obtained by combining \Cref{alg:iter_loorf,alg:orig_arms,alg:iter_dir}.

\begin{algorithm}
	\caption{Original ARMS with Dirichlet copulas}
	\label{alg:orig_arms}
	\begin{algorithmic}
		\State $\triangleright$ This loop can be parallelized with vectorized implementations
		\For{$i \in 1 \dots d$}
		\State Sample $\uis \overset{iid}{\sim} U[0,1]$ for $s \in \{1, \dots, n\}$ 
		\State $\dis \leftarrow \frac{\log \uis}{\sum_{s'=1}^{n} \log u_i^{(s')}}$ \Comment{Convert iid uniforms to Dirichlet r.v.}
		\State $\uTils_i \leftarrow 1 - (1-\dis)^{n-1}$ \Comment{Apply marginal CDF}
		\If{$\theta_i > 0.5$}  \Comment{Additional steps from ARMS paper}
		\State $\zTils_i \leftarrow \ind[\uTils_i \leq \theta_i]$
		\State  $\rho_i \leftarrow \frac{\max(0, 2(1 - \theta_i)^{\frac{1}{n-1}} - 1)^{n - 1} - (1 - \theta_i)^2}{\theta_i(1 - \theta_i)}$ \Comment{Correlation from ARMS paper}
		\Else
		\State $\zTils_i \leftarrow \ind[1 - \uTils_i \leq \theta_i]$
		\State $\rho_i \leftarrow \frac{\max(0, 2\theta_i^{1 / (n-1)} - 1)^{n - 1} - \theta_i^2 }{\theta_i(1 - \theta_i)}$ \Comment{Correlation from ARMS paper}
		\EndIf
		\EndFor
		\State $\bgh \leftarrow \bgh_{ARMS}((\bzTil^{(s)})_{s=1}^n; \btheta)$
		\State \textbf{Return:} $\bgh$
	\end{algorithmic}
\end{algorithm}

\begin{algorithm}
	\caption{Iterative sampling from Dirichlet copula ($d=1$)}
	\label{alg:iter_dir}
	\begin{algorithmic}
		\State Sample $\sum_{s'}E^{(s')} \sim \mathrm{Gamma}[n,1]$
		\State $R \leftarrow \sum_{s'}E^{(s')}$
		\For{$s \in 1 \dots n$}
		\If{$s < n$}
		\State Sample $U \sim U[0,1]$
		\State $E^{(s)} \leftarrow -R U^{\frac{1}{n-s}} + R$
		\Else
		\State $E^{(s)} \leftarrow R$
		\EndIf
		\State $R \leftarrow R - E^{(s)}$
		\State $\triangleright$ Get single Dirichlet r.v. (recall that $\sum_{s'}E^{(s')}$ was already computed)
		\State $\ds \leftarrow \frac{E^{(s)}}{\sum_{s'}E^{(s')}}$
		\State $\uTils \leftarrow 1 - (1-\ds)^{n-1}$ \Comment{Apply marginal CDF}
		\State \textbf{yield }$\uTils$ 
		\EndFor
	\end{algorithmic}
\end{algorithm}

\subsubsection{Beta$^*$}
In this section, we derive a closed-form expression for the constant control variate with lowest variance. As shown in \Cref{eq:intermediate_loorf}, if the control variate is of the form $h(\cdot) \nabla_{\btheta} \log p_{\btheta} (\cdot)$, as long as $h(\cdot)$ is not a function of the current sample $\bzs$, the correction summand is simply zero. We here consider only functions of the form $h(\cdot)=\beta$, but allow different $\beta_i$ for different dimensions. \Cref{thm:bStarOne} shows a closed-form expression for the optimal control variate in this case.
\begin{restatable}{theorem}{bStarOne}\label{thm:bStarOne}
	Consider the problem of estimating
	\begin{equation*}
		\frac{\partial}{\partial \theta_i} \E{\bz \sim p_{\btheta}(\cdot)}{J(\bz)} = \E{\bz \sim p_{\btheta}(\cdot)}{\frac{\partial \log p_{\btheta}(\bz)}{\partial \theta_i} J(\bz)}
	\end{equation*}
	by sampling $\bzs \overset{iid}{\sim} p_{\btheta}(\cdot)$, for $s \in \{1, \dots, n\}$ and using:
	\begin{equation}
		\gh_{\beta_i}((\bz^{(s)})_{s=1}^n; \btheta) = \frac{1}{n} \sum_{s=1}^n \Big(J(\bzs) - \beta_i \Big) \frac{\partial \log p_{\btheta}(\bzs)}{\partial \theta_i}.
		\label{eq:bstar_thm_estimator}
	\end{equation}
	The optimal $\beta_i^*$ such that
	\begin{equation*}
		\beta_i^* = \argmin{\beta_i \in \R} \Var{\gh_{\beta_i}((\bz^{(s)})_{s=1}^n; \btheta)}
	\end{equation*}	
	is given by
	\begin{align*}
		\beta_i^* = \E{\bz \sim q_i(\cdot; \btheta)}{J(\bz)}, &&\text{ where }q_i(\cdot; \btheta) \propto p_{\btheta}(\cdot) \left( \frac{\partial \log p_{\btheta}(\cdot)}{\partial \theta_i} \right)^2.
	\end{align*}
\end{restatable}
\begin{proof}
	See \Cref{sec:bstar_proofs}.
\end{proof}
The quantity $p_{\btheta}(\cdot) \left( \frac{\partial \log p_{\btheta}(\cdot)}{\partial \theta_i} \right)^2$ above is the integrand of the (diagonal) Fisher information metric for the corresponding $i$. The Fisher information corresponds to the integral of this function and quantifies how much a random variable $\bz$ distributed according to $p_{\btheta}(\cdot)$ is predictive of the parameter $\theta_i$. If this quantity is high then it should take less observations of $\bz$ to compute $\theta_i$ accurately via Monte Carlo. The value of $q_i(\Zeta_h; \btheta)$, for $h \in \{0, \dots, 2^d - 1\}$, weights the relative contribution of the value $\Zeta_h$ among the possible values $\bz$ can assume.

To exemplify, if $d=1$ and $\theta = 0.999$, sampling $z$ will yield multiple $1$ values, but it is the zeros that are going to enable better discernment of whether $\theta$ is $0.9$, $0.99$ or $0.999$. \Cref{thm:bStarTwo} shows what the optimal $\beta_i^*$ is when $\bz$ is sampled from a factorized Bernoulli.
\begin{restatable}{theorem}{bStarTwo}\label{thm:bStarTwo}
	For the same setting as in \Cref{thm:bStarOne}, but with the additional condition that $p_{\btheta}(\cdot) = \prod_{i=1}^d \Ber[\theta_i]$, the optimal $\beta_i^*$ becomes:
	\begin{equation*}
		\beta_i^* = \E{\bz \sim p_{\btheta}(\cdot)}{J(z_1, \dots, z_{i-1}, 1 - z_{i}, z_{i+1}, \dots)}.
	\end{equation*}
\end{restatable}
\begin{proof}
	See \Cref{sec:bstar_proofs}.
\end{proof}

This control variate changes the input on the $i$-th coordinate from $z_i$ to $1 - z_i$ while keeping the other dimensions as they were, bearing some resemblance to the alternative definition of derivative used in classical PB approaches from \Cref{eq:pd_der}.

Since it is not feasible to compute this estimator in closed-form for the larger experiments, we will only use it for the smaller ones. It should be seen as an upper-bound of variance reduction achievable by the other methods, as well as an indicator of the implications of such variance reduction when solving the desired PB optimization problems.

\subsection{Drawbacks}
\label{sec:mc_drawbacks}
We now overview issues with the methods presented throughout this section.
\subsubsection{Dependence on the Current Distribution}
When put into the perspective of solving a PB optimization problem, one of the main limitations of the methods presented in this section is their dependence on $p_{\btheta}(\cdot)$. To exemplify, we note that the REINFORCE expression
\begin{equation*}
	\frac{1}{n} \sum_{s=1}^n J(\bzs) \nabla_{\btheta} \log p_{\btheta}(\bzs),
\end{equation*}
corresponds to a weighted average of multiple $J(\Zeta_h)\nabla_{\btheta} \log p_{\btheta}(\Zeta_h)$ terms, for $h \in \{0, \dots, 2^{d}-1\}$, where some of the summands are likely to have higher weights if $p_{\btheta}(\Zeta_h)$ is higher. In other words, we can rewrite it as:
\begin{equation}
	\label{eq:mc_as_avg}
	\sum_{h=0}^{2^d -1} \Big( \frac{n_{\Zeta_h}}{n} \Big) J(\Zeta_h) \nabla_{\btheta} \log p_{\btheta}(\Zeta_h) =  \sum_{h=0}^{2^d -1} \Big( \frac{n_{\Zeta_h}}{n} \Big) \Big( \frac{1}{p_{\btheta}(\Zeta_h)} \Big) J(\Zeta_h) \nabla_{\btheta} p_{\btheta}(\Zeta_h).
\end{equation}
Where $n_{\Zeta_h}$ is the number of samples satisfying $\bzs = \Zeta_h$. In each summand, the term $\nabla_{\btheta} p_{\btheta}(\Zeta_h)$ is the direction that increases the probability of $\Zeta_h$ being sampled again, while $J(\Zeta_h)$ scales that direction according to the cost function and $(1/p_{\btheta}(\Zeta_h))$ up-scales rare values. We can see the sum as multiple $\Zeta_h$ values attracting $\btheta$ towards themselves (or repelling if $J(\cdot) < 0$), as illustrated in \Cref{sub:theta_dep_sampled}.

A big difference between MC methods and the ones using true gradients, is that, in the latter case, the sum becomes:
\begin{align}
	\nabla_{\btheta} \E{\bz \sim p_{\btheta}(\cdot)}{J(\bz)} &= \nabla_{\btheta} \sum_{h=0}^{2^d -1} p_{\btheta}(\Zeta_h) J(\Zeta_h) \nonumber \\
	&= \sum_{h=0}^{2^d -1} J(\Zeta_h) \nabla_{\btheta} p_{\btheta}(\Zeta_h).  \label{eq:true_as_avg}
\end{align}

Notably, it includes all $\Zeta_h$ values, whereas a value that is unlikely to be sampled may not be present in \Cref{eq:mc_as_avg} (see \Cref{sub:theta_dep_true}). Hence, while $\mathbb{E}[J(\cdot)]$ might possibly induce a smooth optimization surface leading towards the minimizing $\bz^*$, optimizing it with (non-stochastic) gradient descent accounts for the entirety of $\corners$. Conversely, MC implementations may never sample $\bz^*$ with low $p_{\btheta}(\cdot)$. Even in the rare occasion that they do, some optimizers normalize the gradient updates, diminishing the joint effect of $J(\Zeta_h)(1/p_{\btheta}(\Zeta_h))$ and making it less likely that $p_{\btheta}(\cdot)$ will get close to placing significant probability mass on $\bz^*$.
\begin{figure}[tb!]
	\centering
	\begin{subfigure}{0.3\textwidth}
		\centering
		\includegraphics[width=0.8\columnwidth]{\main/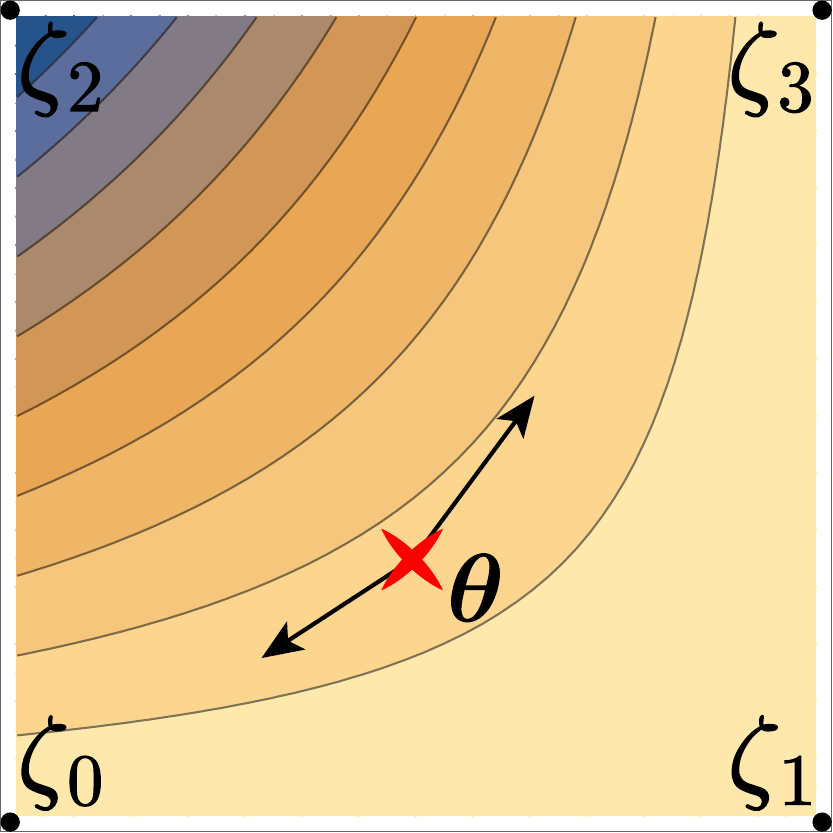}
		\caption{\Cref{eq:mc_as_avg}, $n=2$}
		\label{sub:theta_dep_sampled}
	\end{subfigure}%
	\begin{subfigure}{0.3\textwidth}
		\centering
		\includegraphics[width=0.8\columnwidth]{\main/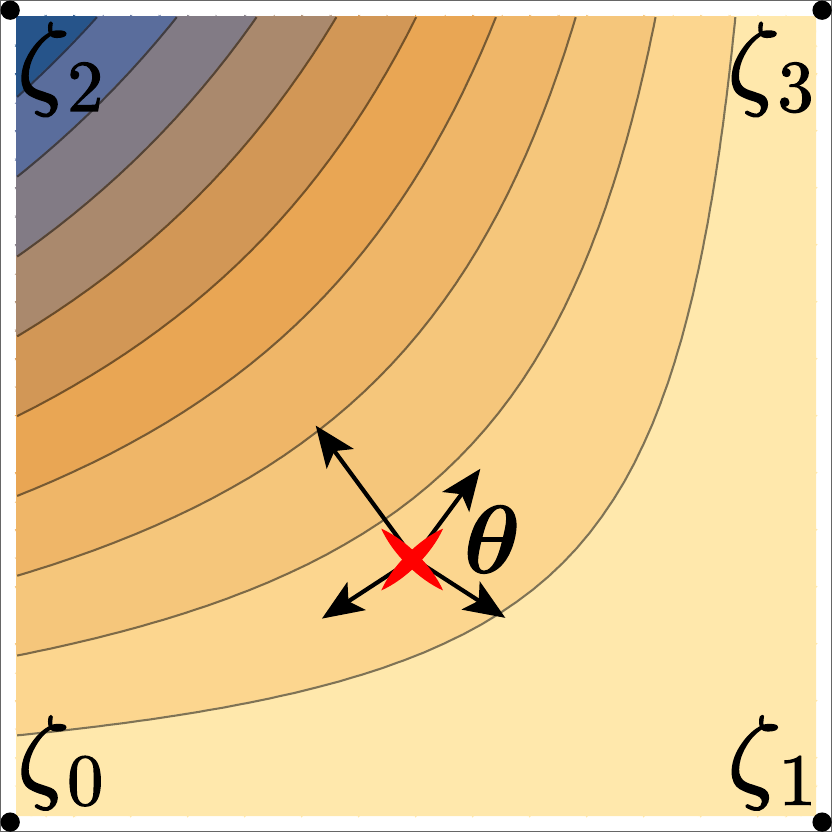}
		\caption{\Cref{eq:true_as_avg}}
		\label{sub:theta_dep_true}
	\end{subfigure}	
	\begin{subfigure}{0.3\textwidth}
		\centering
		\includegraphics[width=0.9\columnwidth]{\main/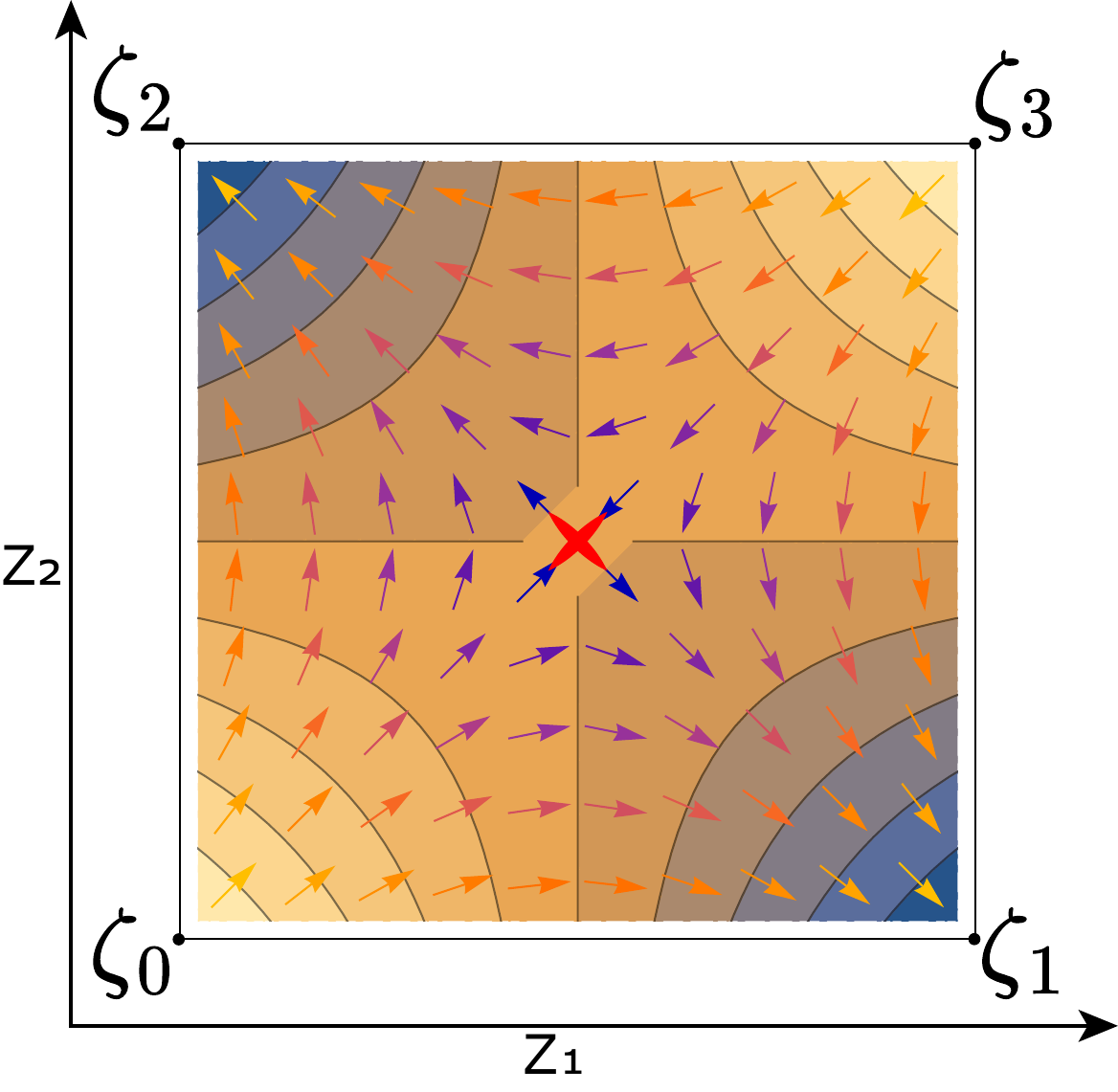}
		\caption{\Cref{ex:mc_bad_gen_1}}
		\label{sub:mc_bad_gen_1}
	\end{subfigure}		
	\caption{Contour plots of $\mathbb{E}[J(\cdot)]$ for $d=2$, blue regions correspond to lower values. Left (a,b): Arrows indicate the summands from the equation indicated in the caption. If $p_{\btheta}(\Zeta_2)$ is low, the gradient may not point towards $\bz^* = \Zeta_2$, such as in (a). Right (c): Illustration for \Cref{ex:mc_bad_gen_1}. Small arrows correspond to the gradient field. Darker arrows correspond to lower gradient magnitudes.}
	\label{fig:theta_dep_ill}
\end{figure}
Although one might consider using importance sampling to make the optimization more exploratory, a large $d$ severely limits its applicability. To exemplify, in case the importance distribution is a factorized Bernoulli with parameters $\btheta_q$, the sampling estimate will be
\begin{gather*}
	\frac{1}{n} \sum_{s=1}^n \left( \frac{p_{\btheta}(\bzTils)}{p_{\btheta_q}(\bzTils)} \right) J(\bzTils) \nabla_{\btheta} \log p_{\btheta}(\bzTils) = \\
	\frac{1}{n} \sum_{s=1}^n \left( \prod_{i=1}^{d} \frac{(\btheta)_i^{\zTils_i} (\Ntheta)_i^{1 - \zTils_i}}{(\btheta_q)_i^{\zTils_i} (\Ntheta_q)_i^{1-\zTils_i}} \right) J(\bzTils) \nabla_{\btheta} \log p_{\btheta}(\bzTils).
\end{gather*}

The IS ratios will be large whenever a $\bzTils$ that is more likely under $p_{\btheta}(\cdot)$ than under $p_{\btheta_q}(\cdot)$ is sampled. Since there are $d$ per-sample coordinates where this can happen, with $d$ assumed to be a high number, it might be infeasible to avoid exploding gradients without limiting $\btheta_q$ or resorting to some biased approach, such as weighted importance sampling \parencite[Section 9.2]{owen2013monte}.\footnote{One possibility to control the variance introduced by IS while retaining unbiasedness is to use defensive importance sampling, where the importance distribution is designed to be a mixture between the nominal distribution and some desired alternative \parencite[Section 9.11]{owen2013monte}.}

Generally it is a rule of thumb to not use IS for large dimensions \parencite{au2003important, li2005curse}. See \textcite[Example 9.3]{owen2013monte} for an example using mean zero Gaussian distributions where slight changes in variance from the importance distribution result in massive variance increments for large $d$. ARMS manages to bypass this problem because of the condition that the marginals remain unchanged, leading to some cancellation. Nonetheless, since the search space is so large, it is infeasible to cover it entirely, causing the design of an adequate importance distribution to either be heuristic or to rely on domain knowledge.

\subsubsection{Unwanted Generalization}
\label{sec:unwanted_gen}
The problems outlined in the previous section revolved mostly around how hard it can be for the sampling to get close to minimizing values. In this section, we show that issues with the probabilistic formulation from \Cref{sec:prob_framework} go beyond the stochasticity introduced by sampling. To understand why, we remind the reader that the problem involves finding the best solution among $2^d$ candidates and that, in general, the value of a single $J(\Zeta_h)$ does not imply anything about evaluations in the rest of the set.

A straightforward approach to solve this problem might be to store a logit vector with $2^d$ entries, one corresponding to each vertex, and increase or decrease them according to comparisons between multiple $J(\cdot)$ evaluations. Naturally, these logits could map to probabilities by the use of some function, a common choice being the softmax. That way, increasing the probability of some solution $\bz$ would not affect the probabilities of unseen $\bz'$ relative to each other. 

Yet, because we cannot store such a large vector, we resorted to using multiple Bernoulli in this section. Although this choice can make gradient-based optimization possible, it also means that increasing the probability of some solution $\bz$ will also increase $p_{\btheta}(\cdot)$ in points that are close to it. When we say close, we refer to the Hamming distance between both vectors, which corresponds to the number of dimensions in which they differ. In our particular case, the Hamming distance between $\bz$ and $\bz'$ is given by:
\begin{equation*}
	d_H(\bz, \bz') = \sum_{i=1}^{d} z_i (1 - z_i') + (1-z_i) z_i'.
\end{equation*}
In \Cref{ex:mc_bad_gen_1}, we show a simple case where this is problematic.
\begin{example}
	\label{ex:mc_bad_gen_1}
	For $d=2$, we take $J(\cdot)$ such that:
	\begin{align*}
		\begin{bmatrix}
			J(\zeta_2) & J(\zeta_3) \\ 
			J(\zeta_0) & J(\zeta_1) 
		\end{bmatrix} =  \begin{bmatrix}
			m & M \\ 
			M & m 
		\end{bmatrix}, &&\text{where }m < M.
	\end{align*}
	Assume the update rule is:\footnote{In general, $\btheta_t$ would additionally have to be projected to $\hypercube$ in case it falls outside of this region. In this section we assume that it will always stay there.}
	\begin{equation*}
		\btheta_t = \btheta_{t-1} - \alpha \left. \left(\nabla_{\btheta} \E{\bz \sim p_{\btheta}(\cdot)}{J(\bz)}\right) \right\rvert_{\btheta = \btheta_{t-1}}.
	\end{equation*}
	If we initialize $\btheta_0 = \begin{bmatrix}0.5 & 0.5\end{bmatrix}^\top$, we have that:
	\begin{equation*}
		\left. \left(\nabla_{\btheta} \E{\bz \sim p_{\btheta}(\cdot)}{J(\bz)}\right) \right\rvert_{\btheta = \btheta_{0}} = \zeros
	\end{equation*}
	and the optimization will be stuck at $\btheta_0$. In that case, the final solution will have $50\%$ change of sampling $M$ instead of $m$. \Cref{sub:mc_bad_gen_1} shows the contour plot for this problem.	
\end{example}
Initializing $\btheta$ in the middle of the hypercube is a reasonable choice when nothing else is known about the optimal solution, since it is equidistant to all vertices. For this simple example, using MC instead of true gradients would actually help getting the optimization away from the saddle point. Nonetheless, it illustrates how different $\Zeta_h$ can interfere with each other. \Cref{thm:mc_bad_gen_2} shows another example of undesired behaviour due to this generalization.
\begin{restatable}{theorem}{mcBadGen}\label{thm:mc_bad_gen_2}
	For some arbitrary $\bz^* \in \corners$, define $J:\corners \rightarrow \R$ as:
	\begin{equation*}
		J(\bz) = \left\{ \begin{aligned}[c]
			&m &&\hspace{1em}\text{if }\bz=\bz^*\\
			&M_0 - d_H(\bz, \bz^*) \frac{\Delta M}{d} &&\hspace{1em}\text{otherwise}
		\end{aligned},
		\right.
	\end{equation*}
	where $M_0 \in \R$, $\Delta M \in \R_{>0}$ and $m = \min_{\bz \in \corners}{J(\bz)}$ is unique. Particularly, this is only satisfied if:
	\begin{gather*}
		m < M_0 - \Delta M,
	\end{gather*}
	where $J(\ones - \bz^*) = M_0 - \Delta M$ is the second lowest value. For any arbitrary dimension $i$ and assuming $p_{\btheta}(\cdot)$ is a factorized Bernoulli with parameter $\btheta$, we have:
	\begin{alignat}{2}
		\label{eq:mc_bad_gen_inner}
		&{} & \left(-\nabla_{\btheta}\E{\bz \sim p_{\btheta}(\cdot)}{J(\bz)}\right)_i \left(\nabla_{\btheta}p_{\btheta}(\bz^*)\right)_i \geq 0&\\	
		&\iff & m\leq M_0 - \frac{\Delta M}{d \prod_{j \neq i} p_{\theta_j}(z_j^*)}.& \nonumber	
	\end{alignat}
\end{restatable}
\begin{proof}
	See \Cref{sec:mc_bad_gen_2}.
\end{proof}
The LHS of \Cref{eq:mc_bad_gen_inner} corresponds to the $i$-th summand of the inner product between the direction used in gradient descent (\ie, $-\nabla_{\btheta}\mathbb{E}[J(\bz)]$) and the direction that increases the probability of $\bz^*$ the most (\ie, $\nabla_{\btheta}p_{\btheta}(\bz^*)$). Their sign being the same implies the (negative) gradient is pointing towards the optimal solution in that dimension. \Cref{tab:mc_bad_gen_2} illustrates one $J(\cdot)$ following the description from the theorem. The intuition behind $J(\cdot)$ is that, if not for $\bz^*$, the optimal value would be $\ones-\bz^*$, the furthest point from the true solution. Additionally, the closer $\bz$ is to $\ones - \bz^*$, the better $J(\bz)$ gets. To overcome the joint effect where most of the gradients (in the PB perspective of \Crefalt{eq:pd_der}) evaluated on the set $\corners \setminus {\bz^*}$ point away from $\bz^*$, $m$ must be much less than $M_0 - \Delta M$.

\begin{table}
	\centering
	\begin{tabular}{ | c | c| c | } 
		\hline
		$\bz$ & $d_H(\bz, \bz^*)$ & $J(\bz)$ \\
		\hline
		$\begin{bmatrix}0 & 0 & 0 \end{bmatrix}^\top$ & $3$ & $M_0 - \Delta M$ \rule{0pt}{3ex} \\
		$\begin{bmatrix}0 & 0 & 1 \end{bmatrix}^\top$ ; $\begin{bmatrix}0 & 1 & 0 \end{bmatrix}^\top$ ; $\begin{bmatrix}1 & 0 & 0 \end{bmatrix}^\top$& $2$ & $M_0 - 2 \Delta M/3$ \\
		$\begin{bmatrix}0 & 1 & 1 \end{bmatrix}^\top$ ; $\begin{bmatrix}1 & 1 & 0 \end{bmatrix}^\top$ ; $\begin{bmatrix}1 & 0 & 1 \end{bmatrix}^\top$& $1$ & $M_0 - \Delta M/3$ \\
		$\begin{bmatrix}1 & 1 & 1 \end{bmatrix}^\top$ & $0$ & $m$ \rule[-1.5ex]{0pt}{0pt} \\
		\hline
	\end{tabular}
	\caption{Example of $J(\cdot)$ following the conditions from \Cref{thm:mc_bad_gen_2} for $d=3$ and $\bz^* = \ones$.}
	\label{tab:mc_bad_gen_2}
\end{table}

\begin{corollary}
	\label{cor:mc_bad_gen_2}
	For the setting described in \Cref{thm:mc_bad_gen_2}, define error $\epsilon$ such that
	\begin{equation*}
		\max_{i \in \{1,\dots,d\}}p_{\theta_i}(z_i^*) = 1-\epsilon
	\end{equation*}
	and assume $\epsilon $ is greater than zero.\footnote{Similar results can be derived if some coordinates have error zero, but we opt to use these assumptions to simplify the exposition.} \Cref{eq:mc_bad_gen_inner} is satisfied if and only if:
	\begin{align*}
		m \leq - \Theta\left(\frac{1}{d} \left(\frac{1}{1-\epsilon}\right)^d \right) = -\Theta\left(\frac{c^d}{d}\right), \quad \text{for }c>1.
	\end{align*}
\end{corollary}
In accordance to the previous discussion, \Cref{cor:mc_bad_gen_2} states that the minimum has to go to minus infinity at a rate exponential in $d$ for the problem to be solvable via the probabilistic formulation. It is worth it to mention that $m$, $M_0$ and $\Delta M$ are pre-defined values and none of them needs to have any dependence on $d$ for $J(\cdot)$ to be of the described form (\eg, $m$ does not need to go to $-\infty$ as $d \to \infty$ to be the minimum). Moreover, $J(\cdot)$ will always be bounded to the interval $[m, M_0)$. What happens as $d$ increases is that the differences between evaluations of neighbours (\eg, $J(\bz) - J(\bz') = \Delta M/d$) become smaller and multiple $(J(\bz), J(\bz'))$ become closer together. Because the number of vertices increases exponentially, but $\Delta M/d$ only decreases linearly, relative contribution of the $2^d-1$ points towards $\ones - \bz^*$ will increase at a large rate in this example.

From an optimization perspective, this means that, asymptotically, the solution will not converge to $\bz^*$, even if $p_{\btheta}(\cdot)$ is initialized close to it. Once more, we point out that this is all considering true gradients, which correspond to either using MC with $n \to \infty$ or using $2^d$ evaluations per timestep. This number of evaluations of $J(\cdot)$ would be enough to solve the problem by simply searching among solutions, without even needing any method. Therefore, the probabilistic approach has a hidden assumption that good solutions must also be close to each other according to the Hamming distance, or, alternatively, the best solution must be low enough to compensate for the difference. Although we presented results on a global scale, it is likely that the same effect will occur locally in the optimization landscape. The assumption could very well be broken in large neural networks, given that they operate in complex ways.

Another point that becomes apparent after looking at this example is the dependence of these methods on the scale of the loss. The relative ordering of $J(\bz)$ remains the same, regardless of which $m$, $M_0$ and $\Delta M$ are chosen. Still, values of $m$ that do not comply with the bound from \Cref{thm:mc_bad_gen_2} will cause the method to fail, which would not be the case had $m$ been small enough.

To finalize this discussion, we contrast our analysis with recent positive results published by \textcite{daulton2022bayesian} in the context of Bayesian optimization. Particularly, their paper focuses on smaller problems, with $d$ consisting of a few dozen units at most. Theorem 2 from that paper considers running stochastic gradient descent with unbiased gradient estimators in the probabilistic objective. It states that running infinitely many initializations for infinitely many steps yields the optimal solution for the PB optimization problem. The paper then presents some experimental results using a few parallel runs with different initializations and achieving good results.

Although their analysis was adequate for the problems they focus on, they are likely not reflective of the neural network problems from \Cref{sec:pb_nn_examples}. Namely, \Cref{cor:mc_bad_gen_2} and the subsequent discussion, investigate behaviour as $d \to \infty$ for a constant number of initializations, whereas \textcite{daulton2022bayesian} consider a fixed size problem with only the number of initializations and time steps going to infinity. In the overparametrized regime, we should have a $d$ much larger than the number of possible parallel runs, which is closer to our assumptions.

\subsection{Alternative Parametrizations}
\label{sec:alternative_parametr}
Throughout this section, we focused on the case where:
\begin{align*}
	p_{\btheta}(\bz) = \prod_{i=1}^{d} p_{\theta_i}(z_i), \hspace{0.1\textwidth}\text{for }p_{\theta_i}(z_i) =	
	\begin{cases}
		\theta_i & \text{if }z_i=1 \\
		1 - \theta_i & \text{if }z_i=0 \end{cases},
\end{align*}
which we alternatively denoted as:
\begin{align*}
	p_{\btheta}(\bz) = \prod_{i=1}^{d} \paddedtheta_i^{z_i} \ntheta_i^{\nz_i}.
\end{align*}
Note that, when updating $\btheta$, care must be taken so that it does not fall out of $\hypercube$. It is common practice to, instead of updating $\btheta$ directly, reparametrize it by writing:
\begin{equation*}
	\btheta = \theta(\br),\hspace{0.1\textwidth}\text{for }\br \in \sR,
\end{equation*}
where $\theta(\cdot)$ is applied element-wise and usually chosen such that its range is $[0,1]$. The most common choice for $\theta(\cdot)$ is the sigmoid function due to its simplicity.

Nevertheless, there are some results in the literature questioning the effectiveness of the sigmoid for this purpose. We note that it corresponds to the softmax function when we have only two outputs. \textcite{mei2020escaping} establish, both theoretically and empirically, two problems that appear whenever optimizing an expectation with respect to the softmax: high sensitivity to the initialization (``softmax gravity well'') and slow convergence (``softmax damping''). \textcite{li2020arm} use sigmoid gates as probability masks on NN pruning, where they point out the slow transition between ones and zeros, which they then attempt to reduce by using fixed temperature parameters. Similarly, \textcite{serra2018overcoming} apply annealed sigmoid gates directly as NN masks in sequential task learning, and also point out that the low gradient magnitudes of the sigmoid (with $\tau = 1$) harmed performance, which led them to add a compensation to their annealing schedule.

Based on these observations, our experimental sections will also include the following alternative choices for $\theta(\cdot)$:
\begin{itemize}
	\item Direct parametrization: where $\br = \btheta$, similarly to what was done previously in this section, with the caveat that each $\theta_i$ has to be clamped to $[0,1]$ after being updated.\footnote{In practice we had to clamp it to $[\epsilon, 1-\epsilon]$ instead, for some small $\epsilon$. This avoided the gradient becoming zero prematurely.}
	\item Sinusoid parametrization: we normalize a sinusoid to lie in the $[0,1]$ interval. Qualitatively, its main distinction from the sigmoid is that high values of $r_i$ oscillate between one and zero, rather than causing $z_i$ to get arbitrarily close to deterministic.
	\item Escort: \textcite{mei2020escaping} propose this as an alternative parametrization to the softmax to avoid the described problems. In this work, we simply use the version they propose in their paper, where, for a categorical random variable with two classes, $\theta(\cdot)$ is input two scalars (\ie, $(\br)_i = \br_i \in \R^2$), instead of a single scalar as in sigmoid, where one of the logits from the corresponding softmax is fixed at zero. Following the authors, we use $P=4$ in the experiments.
\end{itemize}
For all of the gradient estimators presented in this section, changing the parametrization simply amounts to changing the score function $\nabla_{\btheta} \log p_{\btheta}(\cdot)$ for $\nabla_{\br} \log p(\cdot; \theta(\br))$ in all $\hat{\bg}(\cdot)$ formulas. In our experiments, the methods will also require inverting given probabilities to initialize $\br$ to the desired values. Sometimes it is possible to invert the mapping in multiple ways, in which case we chose an arbitrary inverse map. \Cref{tab:alternative_param} summarizes all estimators used, as well as their expressions, inverse maps and score functions.

\newcommand{\myMcTableHeader}[1]{
	\begingroup
	\renewcommand{\arraystretch}{1.0}
	\begin{tabular}{@{}c@{}}
		#1 
	\end{tabular}
	\endgroup
}

\begingroup
\renewcommand{\arraystretch}{1.5}

\begin{table}
	\centering
	\resizebox{\textwidth}{!}{%
		\begin{tabular}{ | c | c| c | c | c | } 
			\hline
			Name & \myMcTableHeader{Domain \\ $\br_i \in$} & \myMcTableHeader{Expression \\ $\theta(\br_i)$} & \myMcTableHeader{Inverse map* \\ $\br^{-1}(\theta_i)$} & \myMcTableHeader{Score Function** \\ $(\nabla_{\br} \log p(\bz; \theta(\br)))_i$}
			\\
			\hline
			Cosine & $\R$ & $\frac{1}{2}(1-\cos(r_i))$ & $- \arccos(2\theta_i -1) + \pi$ & $\frac{\cos(r_i) + (2 z_i - 1)}{\sin(r_i)}$ \\
			Direct & $\R$ & $r_i$ & $\theta_i$ & $\frac{z_i - \theta(r_i)}{\theta(r_i)(1-\theta(r_i))}$ \\
			Sigmoid & $\R$ & $\frac{1}{1+\euler^{-r_i}}$ & $\log \theta_i - \log(1 - \theta_i)$ & $z_i - \theta(r_i)$ \\
			Escort & $\R^2$ & $\frac{|(\br_i)_1|^P}{|(\br_i)_1|^P + |(\br_i)_2|^P}$ & $\begin{bmatrix} \left( \frac{\theta_i}{1 - \theta_i} \right)^{\frac{1}{P}} & 1 \end{bmatrix}^\top$ &  $\begin{bmatrix} -\frac{(\theta(r_i) - z_i)P}{(\br_i)_1} & \frac{(\theta(r_i) - z_i)P}{(\br_i)_2} \end{bmatrix}^\top$\\
			\hline
		\end{tabular}
	}
	\caption{Summary of alternative parametrizations. (*) Inverse map assumes $\theta_i \notin \{0,1\}$. (**) If $\theta(\br_i) \in \{0,1\}$, we use zero instead of these formulas.}
	\label{tab:alternative_param}
\end{table}
\endgroup

\section{Hybrid Approaches}
\label{ch:other}
\Cref{ch:numerical_continuation} discussed CP approaches where a smooth function is iteratively annealed and minimized until it becomes close to the desired discrete function, whereas \Cref{ch:monte_carlo} discussed MC methods, which store a parametrized probability distribution that tracks regions with seemingly better performances using a gradient-based procedure. These gradients are essentially comparisons between evaluations at the corners. There, we mentioned the reparametrization trick and explained why it is not applicable here. Nonetheless, due to the aforementioned successes of pathwise gradient estimation in other contexts, the belief that incorporating them in PB optimization should be beneficial has led to attempts to do so. In this section, we briefly mention some of these attempts, while emphasizing that they are not the main focus of this study. Despite that, we include them in some of the experiments. Intuitively, they correspond to hybrid versions which use both the gradients, like CP approaches, but also rely on sampling and comparing evaluations on $\corners$, like MC methods.

One of the simplest and most widely used such method is straight-through (ST) estimation. Originally introduced by \textcite{bengio2013estimating} as a way of adding stochastic units to neural networks, the forward pass consists of a sampling operation,\footnote{Some works simplify this even further by using a threshold instead of sampling.} but the backward pass treats the sampling operation as if it had been the identity. Specifically, they propose to use:
\begin{alignat*}{3}
	&\text{Forward pass: }&  &\mathbb{E}[J(\bz)] \approx J(\bz) && \text{for }\bz \sim \prod_{i=1}^d \Ber[\sigma(r_i)]	\\
	&\text{Backward pass: }& &\frac{\partial \mathbb{E}[J(\bz)]}{\partial r_i} \approx \frac{\partial J(\bz)}{\partial z_i}.
\end{alignat*}

Note that, in addition to treating the sampling operation as if it had been an identity, the above steps ignore the derivative of the sigmoid with respect to $r_i$. \textcite{bengio2013estimating} claimed that including the additional derivative harmed performance. Despite this estimator being biased, many works have used it throughout the years, mainly motivated by its empirical performance rather than analytical results \parencite{zhou2019deconstructing, srinivas2017training, bulat2019improved, bethge2019back, bulat2020high, martinez2020training}. 

More recently, \textcite{shekhovtsov2021reintroducing} presented theoretical reasoning behind straight-through estimation considering a more general framework that includes the method above. Particularly, ST estimation corresponds roughly to approximating the PB derivative of \Cref{eq:pd_der} by using a first order Taylor expansion instead. Additionally, the paper proves that, for dimension $i$, if the expected $f_i(\bz) = \partial J(\bz)/\partial z_i$ has absolute value larger than the Lipschitz constant of $f_i(\bz)$, the corresponding (negative) straight-through gradient will be a descent direction.

Another very popular way of combining gradients and sampling in PB optimization is via the Gumbel-softmax estimator \parencite{jang2016categorical, maddison2016concrete}, a technique that is inspired by the Gumbel max trick. In general, sampling from a categorical distribution with probabilities $(\theta_c)_{c=1}^C$ can be equivalently done with this trick as follows:
\begin{enumerate}
	\item Sample $u_c \sim U[0,1]$.
	\item Transform it using inverse transform sampling: $T_c = \log \theta_c - \log(- \log u_c)$.
	\item Output $\argmax{c=\{1,\dots, C\}} T_c$ (sometimes in one-hot vector form).
\end{enumerate}
Samples from the Gumbel-softmax distribution, also called Concrete distribution, are obtained by changing the last step for $\exp(T_c/\tau)/\sum \exp(T_{c'}/\tau)$. Naturally, relaxing the estimation of an expectation with respect to a categorical distribution by instead using the Gumbel-softmax results in a biased estimator. On the other hand, this estimator can be promptly combined with the reparametrization trick. Furthermore, this distribution can be annealed towards the categorical as $\tau$ goes to zero.\footnote{In fact, the same logic used to motivate CP methods could also be used with Gumbel-softmax. We choose to focus on the deterministic algorithms in \Cref{ch:numerical_continuation} due to their more pronounced recent successes as well as to better separate sampling from gradients as two different sources of top-down information.} Some methods use this biased approach directly \parencite{louizos2017learning, paulus2020rao, zhou2021effective}. \textcite{maddison2014sampling} extend the Gumbel-softmax to continuous random variables and \textcite{paulus2020gradient} extend it to more general combinatorial spaces.

Some unbiased approaches also incorporate the Gumbel-softmax to the design of control variates. When computing the estimator following \Cref{eq:main_control_variates}, $h(\bz)$ is taken to be $J(\bz_G)$, where $\bz_G$ is dependent on $\bz$, but is ultimately sampled from the Gumbel-softmax. The correction is often done using a reparametrization-based estimator, instead of using closed-form $\bg_h$. Popular methods approaching the problem this way are REBAR \parencite{tucker2017rebar} and RELAX \parencite{grathwohl2017backpropagation}.

Other works use Taylor expansions as control variates for unbiased estimation, therefore also relying on first-order derivatives, such as Muprop \parencite{gu2015muprop} and the control-variate used in the experiments from \textcite{mohamed2020monte}, referred to as delta method. More recently \textcite{titsias2022double} propose combining a two-level application of LOORF and Taylor expansions, which was later improved by \textcite{shi2022gradient} by also incorporating Stein operators to the two-level estimator, but ultimately also relying on inclusion of $\partial J(\bz)/\partial z_i$ in the control variates.

Some of these methods have known pitfalls. \textcite{andriyash2018improved} point out that the raw Gumbel-sotmax estimators are highly dependent on tuning the temperature and that the apparent improvement obtained by such tuning can be replicated in other methods via simple entropy regularization. Additionally, \textcite{tucker2017rebar} show that these same Gumbel estimators fail in a very simple one dimensional problem where $J(z) = (z - C)^{2}$ and $C$ is chosen to be very close to $0.5$. We also verified empirically that the version of straight-through estimation presented above fails in this same problem setting.

Comparisons between unbiased hybrid methods and MC methods that do not use $\nabla_{\bz} J(\bz)$ (\ie, methods from \Cref{ch:monte_carlo}) show mixed results. In a non-exhaustive analysis of recent works, we found five papers reporting mostly superior results of variants of ARMS\footnote{\textcite{yin2018arm, dimitriev2021arms, yin2019arsm, shi2022gradient, dong2020disarm}.} while three reported superior performances of RELAX/REBAR.\footnote{\textcite{andriyash2018improved, lorberbom2019direct, dong2021coupled}.} We mention that both REBAR and RELAX compute three forward passes and two backward passes for each sample $s = \{1, \dots, n\}$. Perhaps this computation would have been better used with larger $n$ instead. In more general comparisons between pathwise gradient estimators and score function estimators, \textcite{mohamed2020monte} noted that the latter seem to benefit more from larger $n$. In our experiments, we will occasionally include some of the hybrid methods from this section.

\section{Experiments}
\label{ch:experiments}
Now that we understand the ideas behind CP and MC methods, the next step is to see how well the intuitions developed in the previous sections generalize to more practical settings. By first scaling experiments down to smaller $d$, we can compute closed-form expressions for expectations and variances for the MC methods. Furthermore, non-overparametrized settings can help us understand how reliant the described CP methods are on large neural networks. Accordingly, we start with two problems where $\bz$ is directly input to arbitrary $J(\cdot)$, for $d \leq 10$. Afterwards, we scale up by moving to a regression problem, where $d = 8,050$ and the multiple $z_i$ act as element-wise masks to the weights of a fixed backbone neural network. Finally, we scale up even more by comparing approaches in the problems of pruning and supermasks, where we additionally compare against magnitude-based pruning (MP). We note that, even though there are papers already comparing CP and MP, MC is not well represented in the pruning literature. Namely, Monte Carlo gradient estimation is an active area of research, lots of new approaches to reduce variance during estimation and improve convergence to the true gradient have been proposed in the last few years, as described in \Cref{ch:monte_carlo}. Studies comparing against these approaches for pruning, however, are scarce, and mostly focus on straight-through estimators from \Cref{ch:other}. Code to reproduce all experiments is publicly available on GitHub.\footnote{\url{https://github.com/hsilva664/discrete_nn}.}

\subsection{Microworld}
\label{sec:microworld}
In this section, we restrict ourselves to $d \in \{4,10\}$ and $n \in \{1,4,10\}$, as well as using true gradients, which corresponds to $n \to \infty$. Our experiments aim at first comparing the different estimators from \Cref{sec:mc_overview}, then comparing the different parametrizations proposed in \Cref{sec:alternative_parametr} and finally comparing MC and CP methods, as well as comparing both against some of the hybrid methods from \Cref{ch:other}. We always initialize $\btheta_0 = [0.5, \dots, 0.5]^\top$ in these experiments and optimize with SGD.

\subsubsection{Benchmarks}
Most MC gradient estimation papers start by comparing estimators in problems of the form $J(z) = (z - C)^{2}$, for $C$ close to $0.5$ and then moving to larger ones involving discrete Variational Autoencoders, usually combined with image data sets. The former setting serves to eliminate confounding factors present in the latter, while also allowing computation of closed-form expressions. Still, we believe it might be overly simplistic and as a result not bring to light the differences between the estimators. For example, a simple second order Taylor expansion is enough to completely reconstruct the quadratic $J(\cdot)$ everywhere.

On the other hand, examples from \Cref{sec:mc_drawbacks,sec:cp_drawbacks} were adversarial, designed specifically to highlight some of the main flaws from MC and CP methods. We propose new benchmarks that, while compatible with small-scale settings, also require some additional search, either because of the presence of local minima or because $\nabla_{\bz} J(\bz)$ might not be too informative of $J(\cdot)$ away from the current $\bz$.

\begin{itemize}
	\item \textbf{ExponentialTabularLoss:} we first sample $2^d$ values, where $E_h \overset{iid}{\sim} \Exp[1.5]$ and $h \in \{0,\dots, 2^{d}-1\}$. Then, we calculate the costs by normalizing $\{E_h\}_{h=0}^{2^d - 1}$ such as to map $\max_h E_h$ to $-1$ and $\min_h E_h$ to 1. Because of the exponential distribution, most $\bz$ points will have $J(\cdot)$ closer to $1$, with a few rare exceptions closer to $-1$, meaning the solutions will be harder to find. As sampling is agnostic to the Hamming distance, $d_H(\bz, \bz^*)$ will not be indicative of anything about $J(\cdot)$ in $\bz' \notin \{\bz, \bz^*\}$. In fact, sampling might result in multiple local solutions in different parts of $\corners$. Importantly, this benchmark is not differentiable, so we will only evaluate MC methods using it.
	\item \textbf{NNLoss:} we input $\bz$ directly to a fixed neural network with a scalar output corresponding to the cost. In addition to the input and output layers, this network has $9$ fully-connected hidden layers with $20$ neurons each. Normalization and LeakyRelu follow the respective linear operations of all but the output layer, which simply maps to a scalar that is then only normalized.
	
	We initialize weights to be either $1$ or $-1$ with $50\%$ chance and pre-process the input $\bz$ by mapping it to the $[-1,1]^d$ range. The normalization uses one dimensional ``batch norm'' layers \parencite{ioffe2015batch} without affine parameters. One should run these layers in evaluation mode only, otherwise the output $J(\bz)$ will depend on the whole input batch instead of only on $\bz$. Before fixing moving normalization statistics, we initialize them by running some forward passes on random uniform points from $\hypercube$.
	
\end{itemize}

In both cases, for any fixed $d$, we initialize the loss only once and reuse it across methods and runs from the same method.

\subsubsection{Estimators}
\label{sec:microworld_compare_estimators}
Experiments in this section use only sigmoid parametrization. We start by comparing the variances of different MC estimators. As mentioned when we introduced control variates in \Cref{sec:mc_overview}, they will only reduce $\Var{(\bgh(\cdot))_i}$ if $(J(\bz) - h(\bz)) \partial \log p_{\btheta}(\bz)/ \partial \theta_i$ has smaller variance than $J(\bz) \partial \log p_{\btheta}(\bz)/ \partial \theta_i$. Analytical guarantees of variance reduction only exist in limited contexts. \textcite{richter2020vargrad}, for example, show that LOORF can have lower variance than REINFORCE, but enumerate some conditions in their proof. Similarly, \textcite{dimitriev2021arms} prove that the ARMS gradient will have lower variance than LOORF for negatively correlated $z$ and $z'$, given that $d=1$ and $n=2$. This however, does not automatically imply that the same is true for general $n$ and $d$.\footnote{\textcite{dimitriev2021arms} show that, for general $n>2$ and $d=1$, both LOORF and ARMS correspond to an average of their respective $\bgh((z, z'); \theta)$ (\ie, $n=2$ estimates). Still, the variance for $n > 2$ also has to account for the correlation between these summands.}

With that in mind, we run a single trajectory starting from $\btheta_0$ in the middle of the hypercube and then following the true gradient of the expected cost for $10,000$ steps. This trajectory serves as a guide, providing $\btheta_t$ for different timesteps. For each of them, we calculate the variances for all estimators and compare their across-dimension sums. We consider two settings: $d=4$ with $n=4$ and $d=10$ with $n=10$. In the first, we compute the variances in closed-form after leveraging some combinatorial analysis, whereas in the second we estimate them by using $10,000$ per-iteration evaluations of $\bgh(\cdot)$. More details can be found in \Cref{sec:app_true_var}. Importantly, computing ARMS variance in closed-form requires knowing the IS weights, which we derive algebraically.

\begin{figure}[!htb]
	\centering
	\includegraphics[height=1.5em]{\main/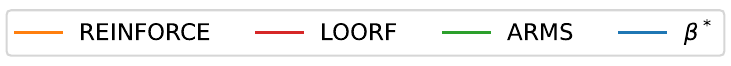}							
	\includegraphics[width=0.3\textwidth]{\main/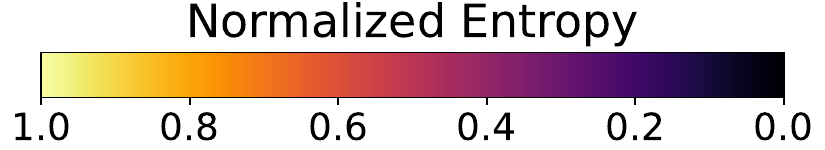}
	\par\vspace{1em}
	\includegraphics[width=0.5\textwidth]{\main/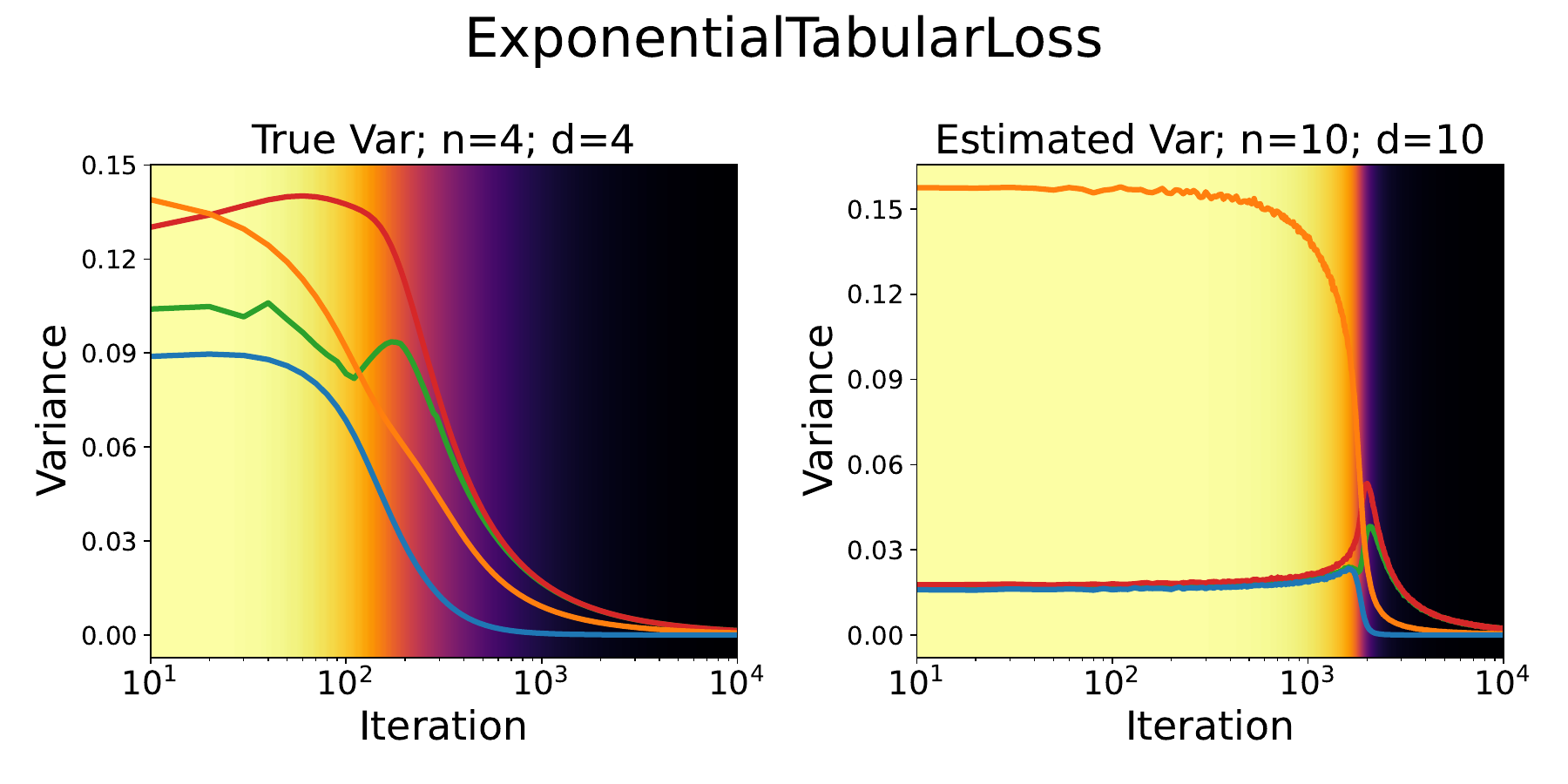}%
	\includegraphics[width=0.5\textwidth]{\main/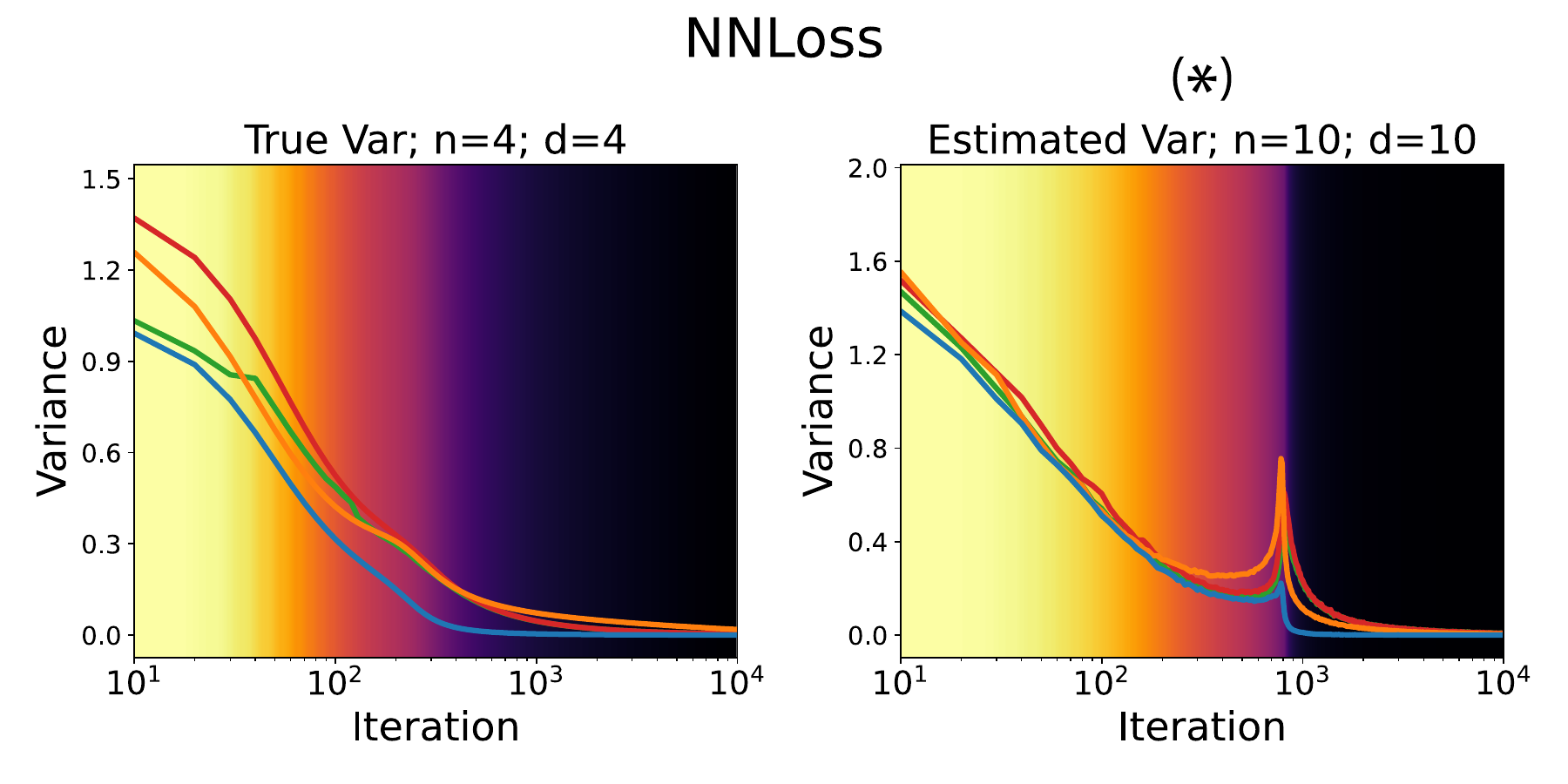}				
	\caption{Comparing estimator variances following a fixed trajectory on Microworld domains. The model was not able to solve the setting marked with (*).}
	\label{fig:microworld_variance}
\end{figure}

\Cref{fig:microworld_variance} shows the results for both settings and benchmarks, considering only learning rate $0.1$. We can see that REINFORCE indeed has lower variance than ARMS or LOORF in some cases. By overlaying the plots with the entropy, which is shared by all methods, we see that REINFORCE tends to perform better when $p_{\btheta}(\cdot)$ is almost deterministic.

Despite analytical proof of ARMS superior variance reduction being restricted to $n=2$ and $d=1$, here it indeed seems to perform better than LOORF, even when using closed-form results instead of estimates. As expected, $\beta^*$ always had the lowest variance, indicating perhaps the lower bound achievable by using control variates without further problem-specific information.

\begin{figure}[htb!]
	\centering
	\includegraphics[height=1.5em]{\main/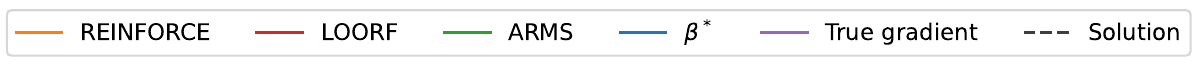}	
	\par\vspace{1em}
	\includegraphics[width=0.5\textwidth]{\main/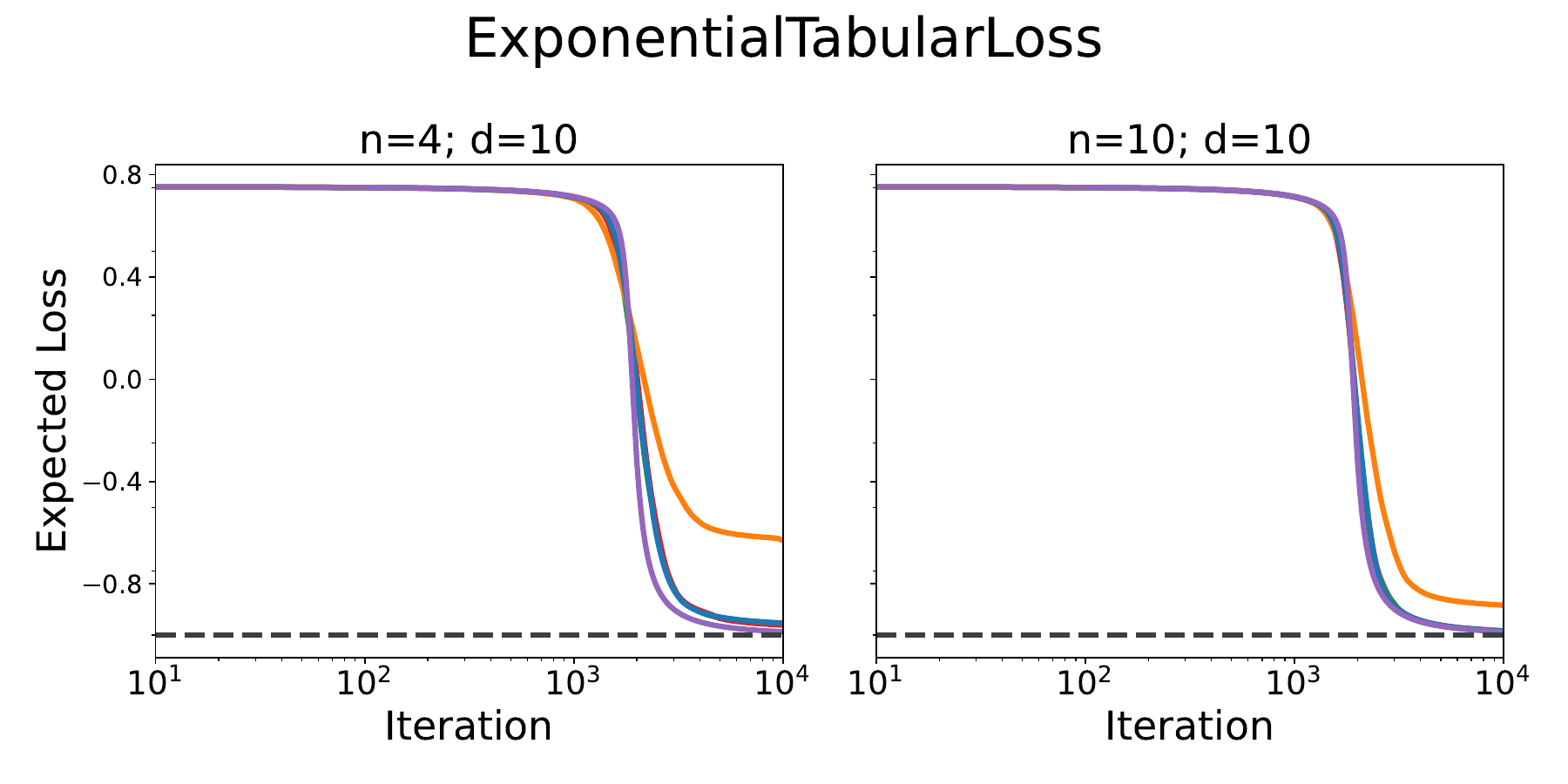}%
	\includegraphics[width=0.5\textwidth]{\main/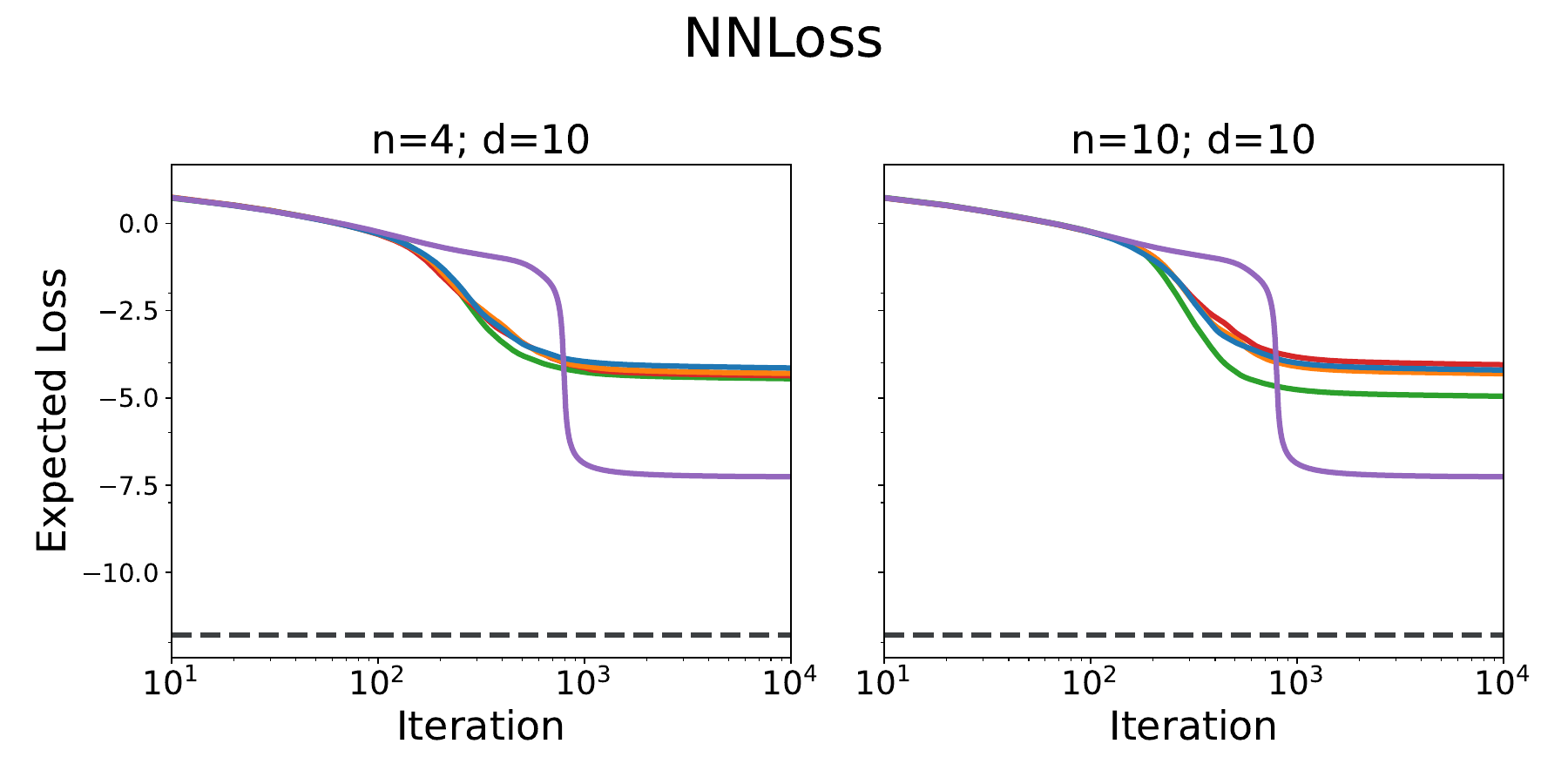}			
	\caption{Comparing estimators on microworld domains.}
	\label{fig:microworld_estimators}
\end{figure}

After analyzing the variances, the next question is how that impacts the solution of the PB optimization problems on self-generated trajectories. Since the setting with $d=4$ is not too challenging and all methods perform similarly well on it, we restrict ourselves to $d=10$ and compare the estimators for varying $n$. \Cref{fig:microworld_estimators} shows the results averaged across $100$ seeds.

REINFORCE now performs clearly worse than the other methods. We recall that it had high variance whenever entropy was high in \Cref{fig:microworld_variance} and that, because of the initialization, entropy is high at first. Since trajectories are no longer supplied, this higher initial variance was likely responsible for the model ending up far from where it would have gone had it followed the true gradient, justifying its poor performance. For the other estimators, variance reduction relative to each other did not seem to matter much. Despite using the optimal control variate, $\beta^*$ failed to perform better than LOORF and ARMS in both problems. On NNLoss, even the model using true gradients failed to reach the correct solution, although eliminating the variance still led it to a better local minimum than the MC estimators. We can explain this observation using the generalization discussion from \Cref{sec:unwanted_gen}, where interactions between $\Zeta_h$ impede the model from reaching the correct solution.

\subsubsection{Parametrizations}

\begin{figure}[htb!]
	\centering
	\includegraphics[height=1.5em]{\main/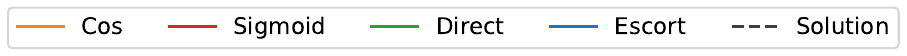}	
	\par\vspace{1em}
	\includegraphics[width=\textwidth]{\main/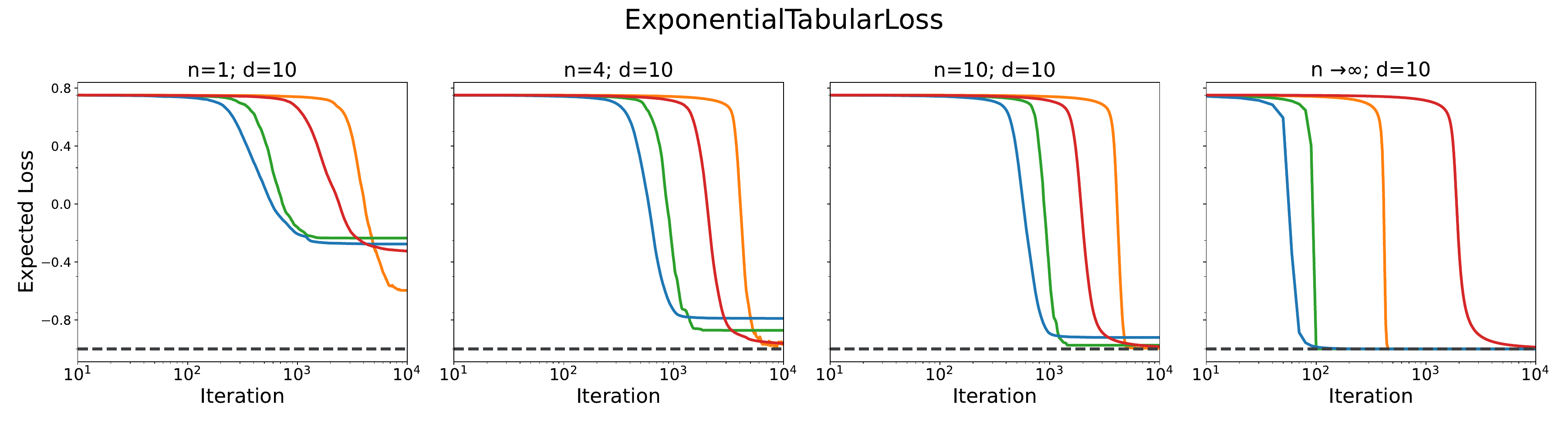}	
	\includegraphics[width=\textwidth]{\main/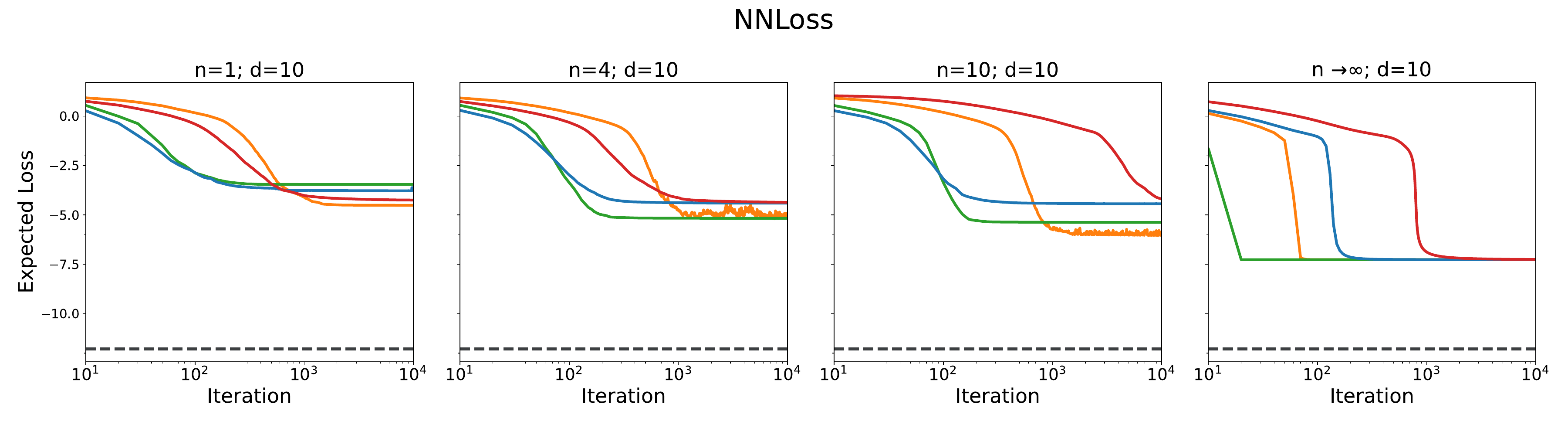}			
	\caption{Comparing parametrizations on microworld domains.}
	\label{fig:microworld_params}
\end{figure}

In this section, we run the same experiments as before, but keep the estimator fixed and vary the parametrization instead. Namely, we use REINFORCE for $n=1$ and LOORF for $n \in \{4, 10\}$. We sweep learning rates in $\{0.1, 0.01\}$ and select the best one based on the final loss. \Cref{fig:microworld_params} shows the results averaged over $100$ seeds.

By inspecting the figure, we note that the parametrizations can roughly be categorized in two groups: direct and escort; sigmoid and cosine. The first converges much faster, but often to worse values, whereas the second converges slower, but to better final values. As more samples are used, the faster group seems to maintain superior convergence speed while arriving at better final solutions.

\begin{figure}[htb!]
	\centering
	\includegraphics[width=\textwidth]{\main/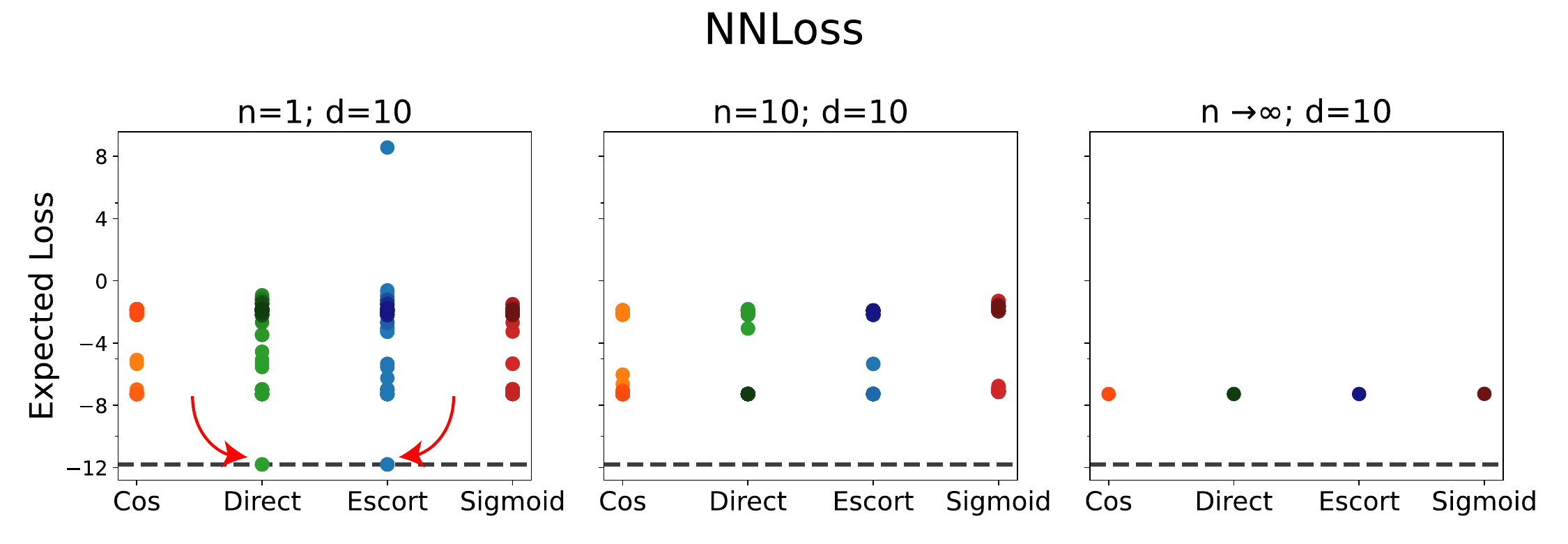}			
	\caption{Comparing parametrizations on microworld domains. Alternative visualization of NNLoss from \Cref{fig:microworld_params} showing the final loss for each seed.}
	\label{fig:microworld_params_scatter}
\end{figure}

\Cref{fig:microworld_params_scatter} depicts NNLoss scatter plots for this same experiment, where each circle corresponds to the final per-seed loss. Losses are more spread out for low $n$, but converge towards $-8$ as samples increase. For $n=1$, the higher variance caused the model to be more exploratory and even find $\bz^*$ in some runs, indicated with arrows. However, these ouliers stopped appearing as variance was reduced with increasing n and the generalization problem stopped the model from ever reaching $\bz^*$.

\subsubsection{Approaches}
\label{sec:microworld_approaches}
\begin{figure}[htb!]
	\centering
	\includegraphics[height=1.5em]{\main/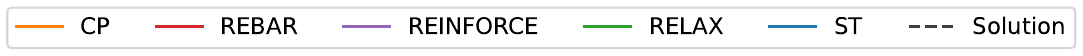}			
	\par\vspace{1em}
	\includegraphics[width=0.7\textwidth]{\main/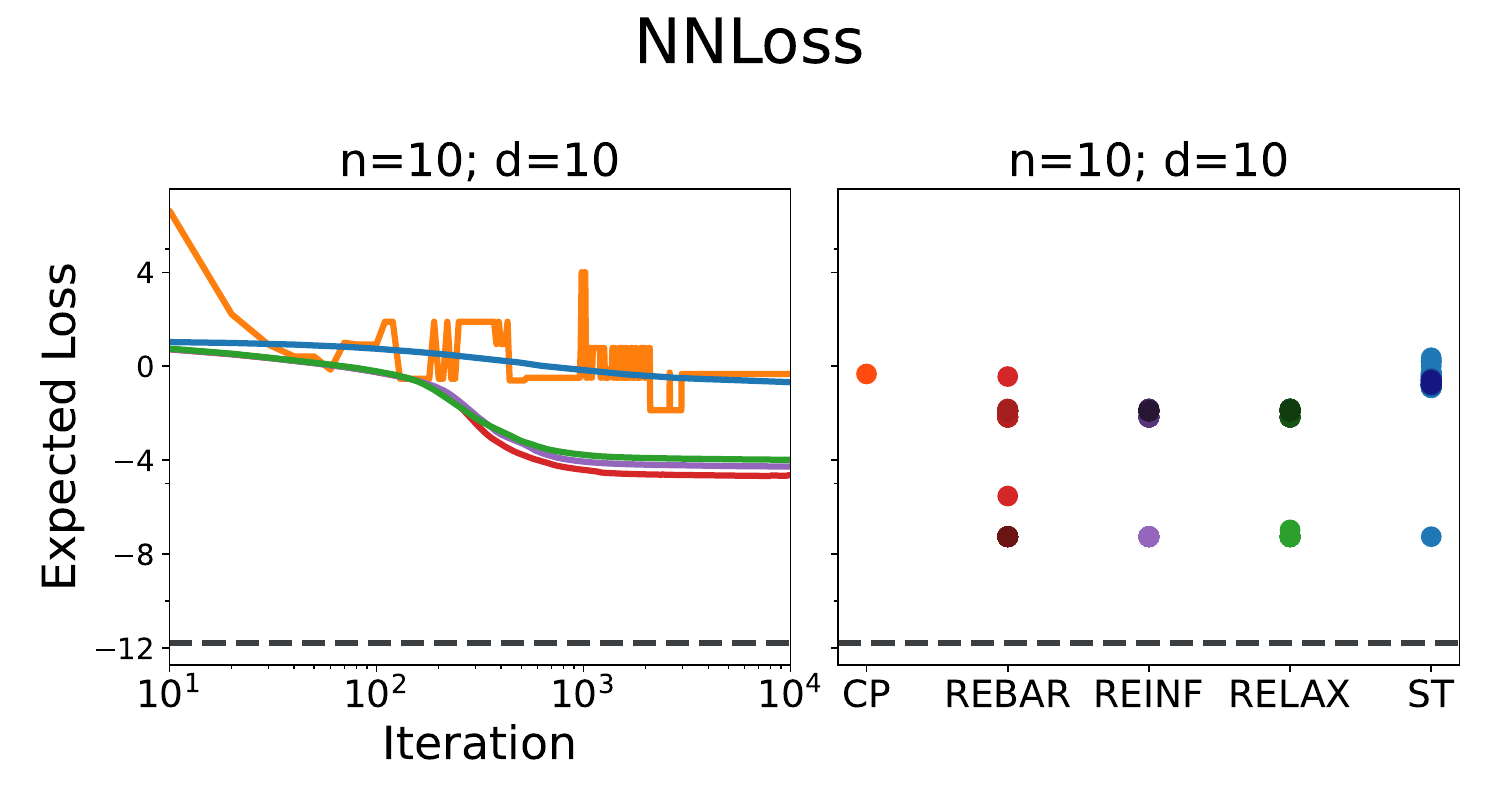}
	\caption{Comparing approaches on microworld domains.}
	\label{fig:microworld_grad}
\end{figure}
\label{sec:microworld_exp_approaches}
Finally, we compare MC and CP in \Cref{fig:microworld_grad}. We also include hybrid approaches from \Cref{ch:other}, where ST is a biased method and both REBAR and RELAX are unbiased, but use $\nabla_{\bz} J(\bz)$ in their control variates. Similarly to the above, we sweep learning rates in $\{0.1, 0.01\}$ and use $100$ seeds on non-deterministic algorithms. Additional details about hyperparameters particular to each method are deferred to \Cref{sec:app_microworld_hypers}. We parametrize REINFORCE with sigmoid. Differently from the previous experiments, selecting hyperparameters based on the lowest final loss, led to poor choices. Settings often had similar final losses and tie-breaking sometimes chose models that converged much slower. Therefore, we select by comparing the average value from iteration $10$ onward instead.

By inspecting the plots, we can see that CP has a ``jagged'' format, likely due to sudden temperature changes, causing the model to reach a new solution quickly and stay there until $\tau$ changes again. Overall, it did not seem that $\nabla_{\bz} J(\bz)$ was helpful here: ST and CP, which arguably rely on it the most, were stuck in poor local solutions, whereas the methods that combine MC and gradients had similar performances to merely using REINFORCE. Alternative parametrizations and estimators already outperformed this version of REINFORCE for this same benchmark in the previous experiments.

We can again put these results in light of previous discussions. In particular, \Cref{sec:cp_drawbacks} mentioned that, when applied to PB optimization in the described manner, numerical continuation methods tend to extrapolate local information from inside the hypercube to differences between its corners. The same is true for the other methods relying on the gradient of $J(\cdot)$. This extrapolation is only heuristic and appears to be inaccurate in this benchmark.

\subsection{Neural Network Regression}
The goal of this section is to perform comparisons similar to the previous ones, but on a larger scale, although not as large as in image classification tasks. This time, it is not feasible to compute expectations in closed-form, so we restrict ourselves to $n \in \{2,10,100,1000\}$ and report losses with respect to a sampled $\bz$. Just as before, we initialize $\btheta_0 = [0.5, \dots, 0.5]^\top$, but now we optimize parameters with RMSprop, the only exception being CP methods. For CP methods, we found it particularly beneficial to sweep over both RMSprop and SGD, which we do from this section onwards (the stochastic methods were often harder to train with SGD). Results are averaged over $10$ seeds and $d$ is equal to $8,050$. We select hyperparameters by comparing the average loss in the second half of training.

\subsubsection{Benchmark}
For these experiments, we use a single benchmark, called MaskedNNRegression, consisting of two fixed neural networks: a backbone NN and a target NN. $\bz$ is a binary mask applied weight-wise to the backbone NN and the goal is to learn the input-output mapping from the target NN while only changing the masks. Furthermore, the target has higher capacity than the backbone and a more complexity-inducing initialization, so these experiments are not in the overparametrized regime yet.

First, we sample some random input data uniformly from $[-1,1]^{10}$ and map each $10$-dimensional vector to a scalar output using the target network. We fix this data and split it into training and validation data sets. Importantly, we sample data sets, backbone and target networks only once and re-use them across different random seeds. The loss is the expected (absolute) difference between target and output of the masked backbone network. More details about the architectures used and experimental setup are in \Cref{sec:app_maskednnregression}

\subsubsection{Estimators}
\begin{figure}[htb!]
	\centering
	\includegraphics[height=1.5em]{\main/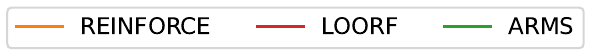}			
	\par\vspace{1em}
	\includegraphics[width=\textwidth]{\main/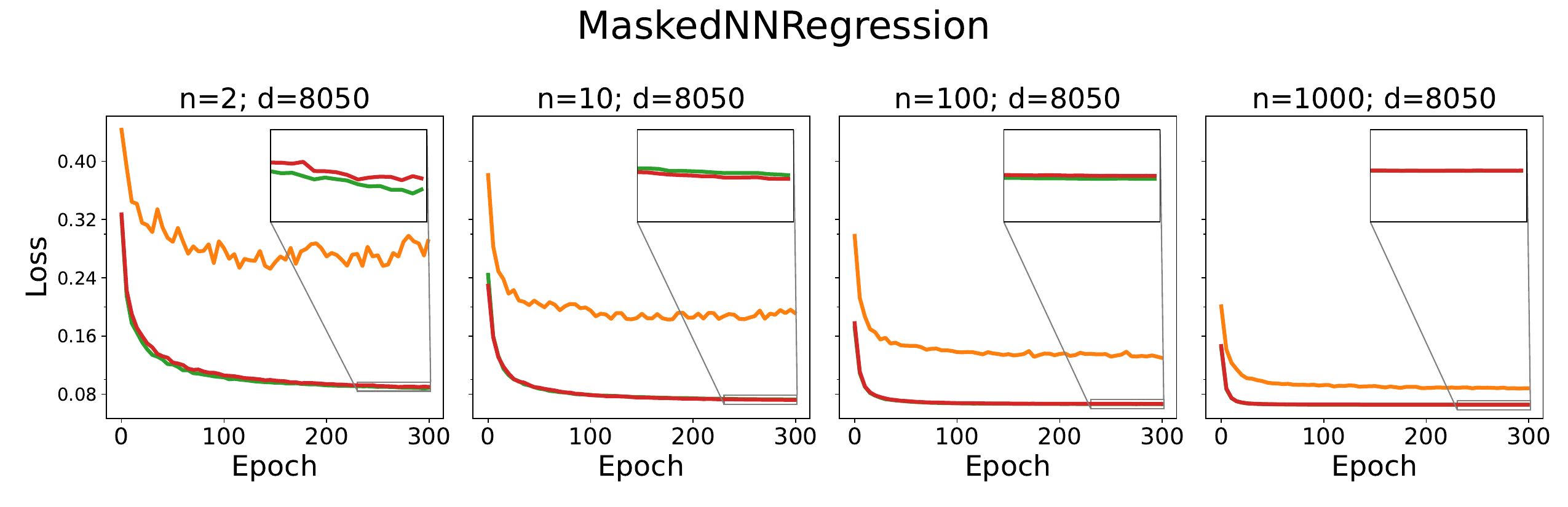}
	\caption{Comparing estimators on masked NN regression.}
	\label{fig:maskednnregression_est}
\end{figure}

As we show in \Cref{fig:maskednnregression_est} (for learning rate $=0.1$ and sigmoid parametrization), REINFORCE still performs worse than ARMS and LOORF, but this time the difference is much higher. Recall that the entropy of a random variable that can assume $2^d$ values is at most $d \log 2$, which is also the initial entropy in these experiments. Considering the previous REINFORCE results, where entropy, variance and loss were positively correlated, this more entropic initial model should indeed lead to relative worsening of its performance.

Despite that, lower entropy does not necessarily imply lower variance. The former corresponds roughly to the spread of a random variable with respect to the set of values it can assume, whereas the latter concept is tied to a scalar quantity measuring how much it changes as a consequence of the changes from this variable.

ARMS and LOORF again performed similarly. The ratio between $n$ and samples needed to solve the PB optimization (\ie, $n/2^d$) will get even worse in the later experiments. Perhaps ARMS being superior for $n=2$ in \Cref{fig:maskednnregression_est} and also having lower variance in \Cref{fig:microworld_variance} is indicative that it will perform slightly better than LOORF in the larger experiments.

\subsubsection{Parametrizations}
\begin{figure}[htb!]
	\centering
	\includegraphics[height=1.5em]{\main/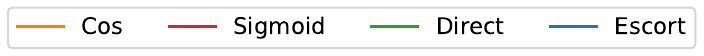}
	\par\vspace{1em}
	\includegraphics[width=\textwidth]{\main/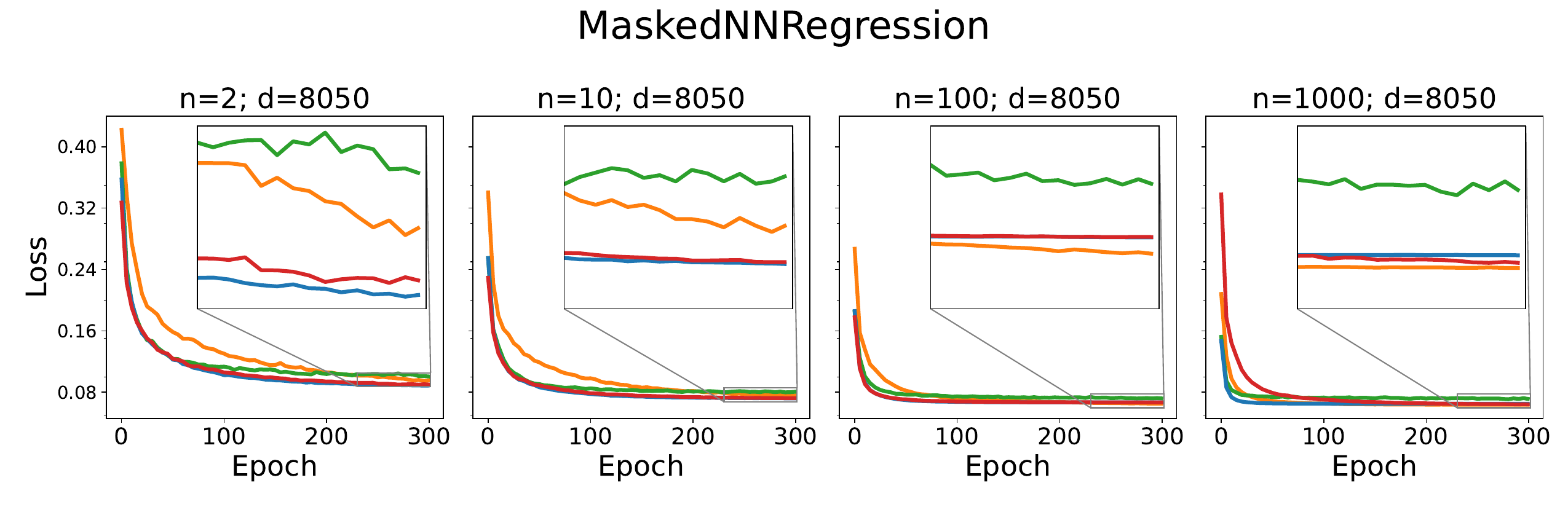}
	\caption{Comparing parametrizations on masked NN regression.}
	\label{fig:maskednnregression_param}
\end{figure}

\Cref{fig:maskednnregression_param} depicts the results of comparing parametrizations, where we sweep learning rates in $\{0.1, 0.01\}$ and use LOORF (approach-specific parameters are as before). Similarly to the previous experiments, direct parametrization is quick to converge, but to sub-optimal values. This effect is even more pronounced now. Cosine parametrization is still slow, but, in contrast with the previous section, its final values are worse than those of the other estimators. Escort and sigmoid were the best performing parametrizations. The first maintained its convergence speed, but reached better final values, while the latter got faster. On the larger experiments, we will only use sigmoid and escort.

\subsubsection{Approaches}
\begin{figure}[htb!]
	\centering
	\includegraphics[height=1.5em]{\main/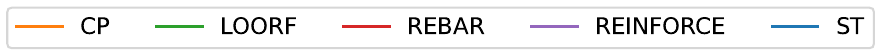}
	\par\vspace{1em}
	\includegraphics[width=0.7\textwidth]{\main/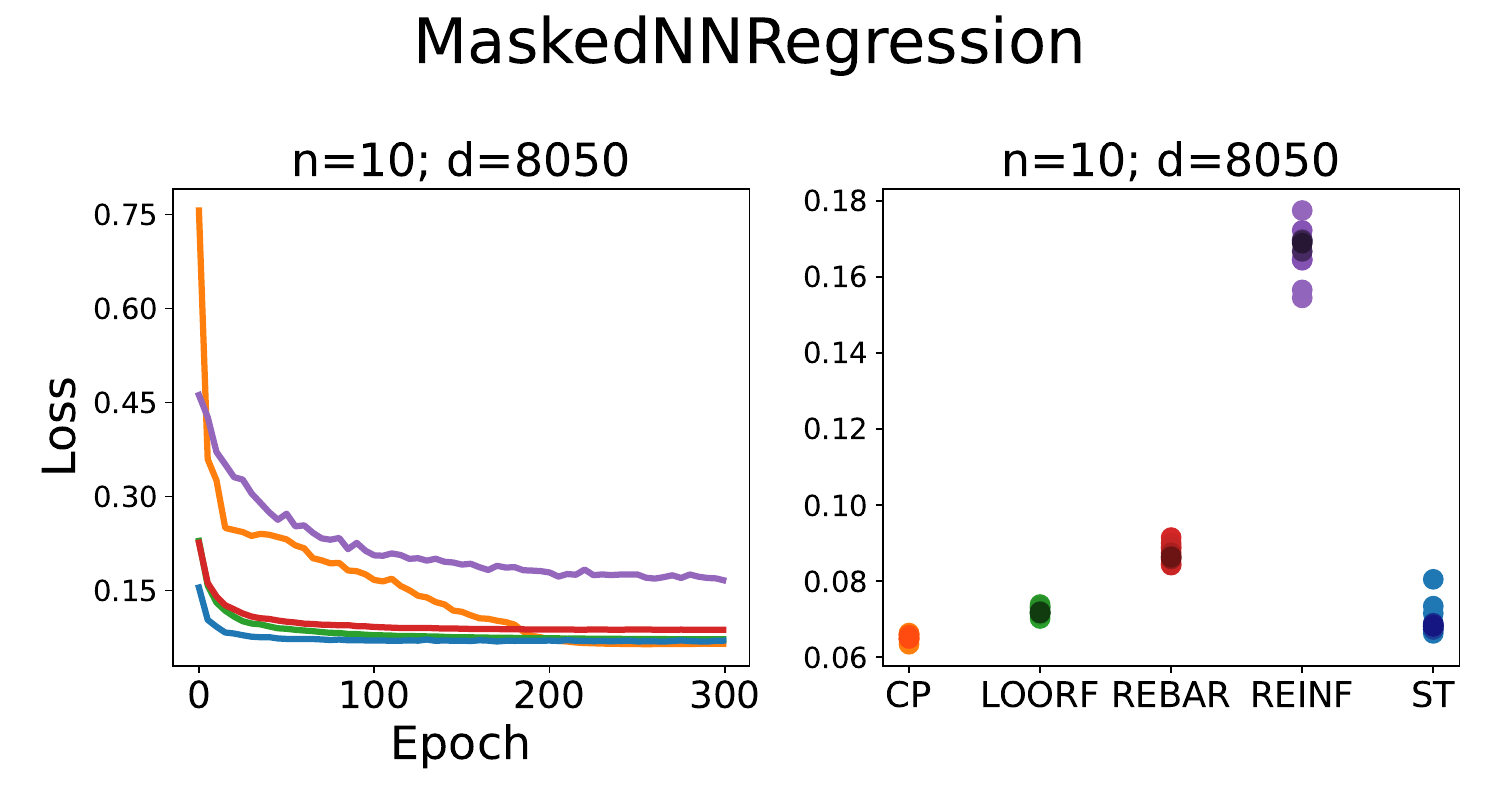}
	\caption{Comparing approaches on masked NN regression.}
	\label{fig:maskednnregression_approaches}
\end{figure}

As seen in \Cref{fig:maskednnregression_approaches} (for learning rates in $\{0.1, 0.01\}$, $n=10$ for the non-deterministic methods and sigmoid parametrization), results from this experiment were different from those of \Cref{sec:microworld_approaches}. The combination of larger $d$ and neural network structure is somewhat compatible with approaches reliant on $\nabla_{\bz} J(\bz)$. Differently from last time, simply using REINFORCE was much worse than combining MC and gradients, such as in REBAR and ST. The best performance was from CP, followed by ST and LOORF. ST converged the fastest, but CP and LOORF had more runs converging to the lower loss values (see the scatter plot). CP, however, converged slower than the other methods.

The depicted CP runs use SGD. Although we do not report results with varying optimizers, we note that all MC methods usually had better performances when using RMSprop compared to SGD. This was consistent across experimental settings, even on the larger ones from the next parts. For CP methods, on the other hand, the best optimizer varied significantly between settings. For MaskedNNRegression, for example, the best performance of CP with SGD was almost as bad as REINFORCE. For that reason, we recommend sweeping over optimizers for CP methods.

\subsection{Pruning}
\subsubsection{Benchmark}

As explained in \Cref{ex:pruning}, in pruning we have two often-conflicting objectives: to improve performance while making the network sparse. The latter often serves to regularize the network, as well as to reduce memory and computation. In this work we will focus on unstructured weight sparsity, where pruning happens in a indiscriminate manner. This is opposed to structured sparsity, where the structure often comes in the form of entire layers or filters and has the benefit of being more simply translated to practical performance benefits. One may not immediately see such performance on unstructured sparsity, since general-purpose hardware and implementations will often store the data structure as dense regardless. Still, it serves as an upper bound on the possible sparsity and as a way to advance pruning research more easily by focusing more on the algorithmic aspects and less on the deployment limitations.

Pruning methods attempt to solve the objective from \Cref{eq:pruning_l0_raw}, either implicitly or explicitly. As \Cref{eq:pruning_l0_deferred} consists of a PB optimization in the search for masks, we can apply the methods discussed here to solve it. We run experiments in both the case that the backbone network is fixed throughout training (supermask) as well as the case where masks and main weights are trained jointly (joint pruning). When optimizing the objective with MC, we only estimate the main objective, as we found using the regularization in closed-form performs better. We do not include hybrid approaches in these experiments, while noting that multiple works already investigate them (see below).

One advantage of using the $L_0$ norm as opposed to $L_1$ or $L_2$ norms, is that it more directly represents what we desire from the optimization. Other approaches, such as lasso \parencite{tibshirani1996regression}, rarely cause weights to be exactly zero, requiring subsequent thresholding steps for pruning to occur. Furthermore, $L_1$ and $L_2$ regularization are both incompatible with batch normalization layers \parencite{ioffe2015batch}, since the affine parameters can simply undo the regularization.

To give some context, the combination of the $L_0$-regularized objective and CP approaches has been successfully implemented by \textcite{savarese2020winning,  yuan2020growing, luo2020autopruner}, while \textcite{azarian2020learned} also suggested it. In fact, \textcite[Section 8.7]{hoefler2021sparsity} later surveyed some results from the past years and classified it as one of the most parameter-efficient approach on Resnet-50. Most works that address the problem using (approximate) $L_0$-regularization directly, however, rely on the hybrid methods from \Cref{ch:other}. Some examples are the studies by \textcite{louizos2017learning, mccarley2019structured, zhou2021effective, srinivas2017training, zhou2019deconstructing}. Some works report instabilities when training some of these hybrid methods \parencite{savarese2020winning, gale2019state, li2020arm}. \textcite{savarese2020winning}, however, recently performed experiments where straight-through performed reasonably well when used as a baseline, although worse than CP. Finally, literature using pure MC approaches for this problem is scarce, the main work being done by \textcite{li2020arm} using ARM with 2 samples.

Still, for pruning problems, the most popular approach is to heuristically use the magnitude of the weights as a proxy for their importance, thus choosing the smallest ones for removal.\footnote{\textcite[figure 18a]{hoefler2021sparsity} categorize 157 papers published between 1988 and 2020, concluding that magnitude pruning is the most common category, comprising 48 of the surveyed papers.}. Although there are some simple examples where the logic behind magnitude pruning does not hold (see \Crefalt{sec:mp_failure}), when taken in conjunction with gradient-based training, it yields good empirical results. In addition to its simplicity, the absence of need for storing additional parameters makes them specially convenient for very large models, which can occupy most of the memory available already. We include two magnitude-based baselines in our experiments: the vanilla approach from \textcite{han2015learning} (MP) and the gradual magnitude schedule from \textcite{zhu2017prune} (GMP).

We make the transition between smaller and larger scale settings more gradual by first running supermask experiments \parencite{zhou2019deconstructing} on simpler architectures and then train both backbone weights and masks together. We run experiments for $200$ epochs and optimize $\br$ with RMSprop for MC methods, while again initializing $\btheta$ in the middle of the hypercube. For CP methods, we again found it beneficial to include the choice between RMSprop and SGD as a hyperparameter when optimizing $\br$. Results reported for each setting correspond to averages over 5 seeds. We report results on the test set, while noting that the training results were similar.

\subsubsection{Supermask}

We run experiments on two architecture-data set combinations: Lenet \parencite{lecun1998gradient} on MNIST and a 6 layer convolutional neural network on CIFAR-10 \parencite{krizhevsky2009learning}. We exclude MP methods from the comparison, as they rely on training the main network. Further, we only show results for ARMS and escort, since they performed slightly better than sigmoid and LOORF. Additional experimental details can be found in \Cref{sec:app_pruning_setup,sec:app_supermask_setup}.

\begin{figure}[htb!]
	\centering
	\includegraphics[height=1.5em]{\main/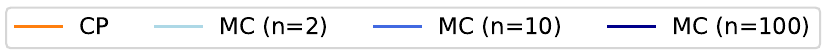}
	\par\vspace{1em}
	\includegraphics[width=0.85\textwidth]{\main/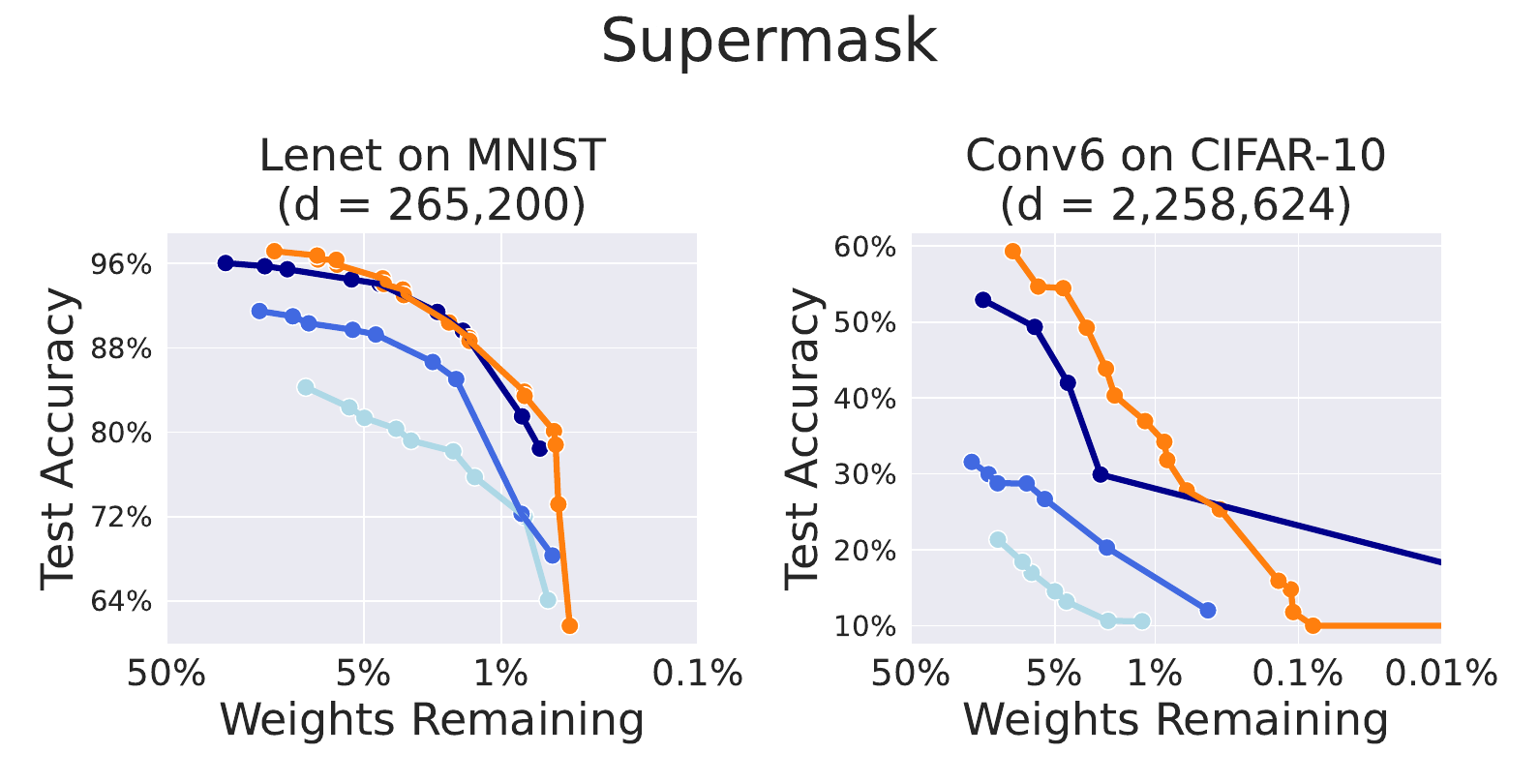}
	\caption{Pareto front for supermask experiments.}
	\label{fig:supermask_results}
\end{figure}

In \Cref{fig:supermask_results}, we show the Pareto front for CP and MC. For multiple objectives and a set of solutions $\sS$, this is the largest subset where no element surpasses another in all objectives. Visually, it is a curve with the best performances for each sparsity level.

These results reaffirm those of MaskedNNRegression: it seems that large neural networks change the discrete problem in a way that benefits approaches reliant on $\nabla_{\bz} J(\bz)$. In fact, the effect is even more pronounced now. Note, for example, how $n=10$, which was competitive with CP before, now performs much worse. Extrapolations from evaluations outside of $\corners$, which were misleading before (\Cref{sec:cp_drawbacks,sec:microworld_exp_approaches}), now produce good results. On top of that, this happens even without the co-adaptation of the backbone network. On the other hand, it seems that the difficulty of finding better solutions by mere search, such as in MC methods, increases significantly with larger $d$. Even using $100$ samples is now worse than using a single combination of forward and backward passes on CP.

The contrast between these results and those from \Cref{sec:microworld} can be connected to previous discussions. \Cref{ch:numerical_continuation} showed the unreliability of the gradient extrapolations from CP methods on simple examples, later confirmed by microworld experiments. However, there were also signs that performance of such methods might be more correlated with properties of $J(\cdot)$ other than $d$. As an example, when introducing pathwise gradients from \Cref{sec:mc_overview}, we noted that the reparametrization estimators had their variance linked to the Lipschitz constant of $J(\cdot)$ instead. Similarly, in \Cref{ch:other}, we also cited analytical results linking the Lipschitz constant and ``correctness'' of ST.

Conversely, the drawbacks discussed in \Cref{sec:mc_drawbacks} (\ie, dependence on the current $\btheta$ and generalization between $\Zeta_h$), are both worsened by high $d$. The introduction of NN, combined with the initializations, activations and losses used, might have helped control how fast $J(\cdot)$ changes in $(0,1)^d$, which benefits CP methods, but not MC. Nonetheless, this does not explain what exactly changes, nor how overparametrization leads CP to good solutions or how far these solutions are from the best possible. Finally, we should mention that CP with RMSprop was significantly better than with SGD in these experiments, differently from the MaskedNNRegression results.

\subsubsection{Joint Pruning}
Next, we move to deeper convolutional networks and start training the main weights. Since running sweeps with $n=100$ on these larger models is challenging, we reuse the best hyperparameters from $n \in \{2,10\}$. See \Cref{sec:app_supermask_hypers} for alternative visualizations of data from supermask experiments suggesting this is a valid extrapolation. Other experimental details are in \Cref{sec:app_joint_pruning_setup}.

\begin{figure}[htb!]
	\centering
	\includegraphics[height=1.5em]{\main/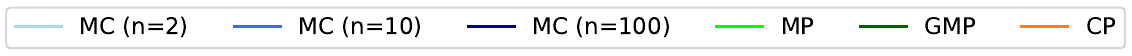}
	\par\vspace{1em}	
	\begin{subfigure}{264.99pt}
		\includegraphics[height=132.495pt]{\main/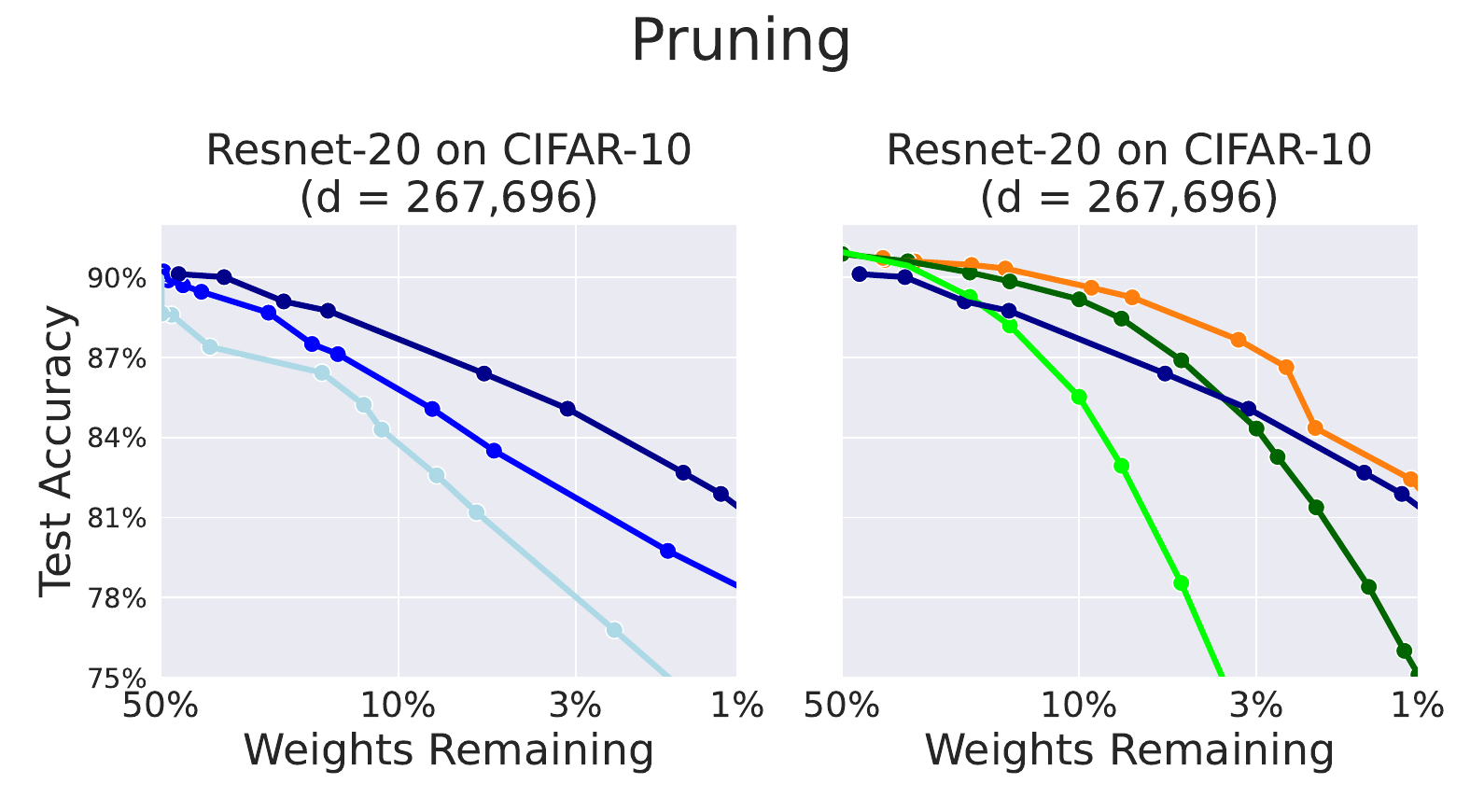}
		\includegraphics[height=132.495pt]{\main/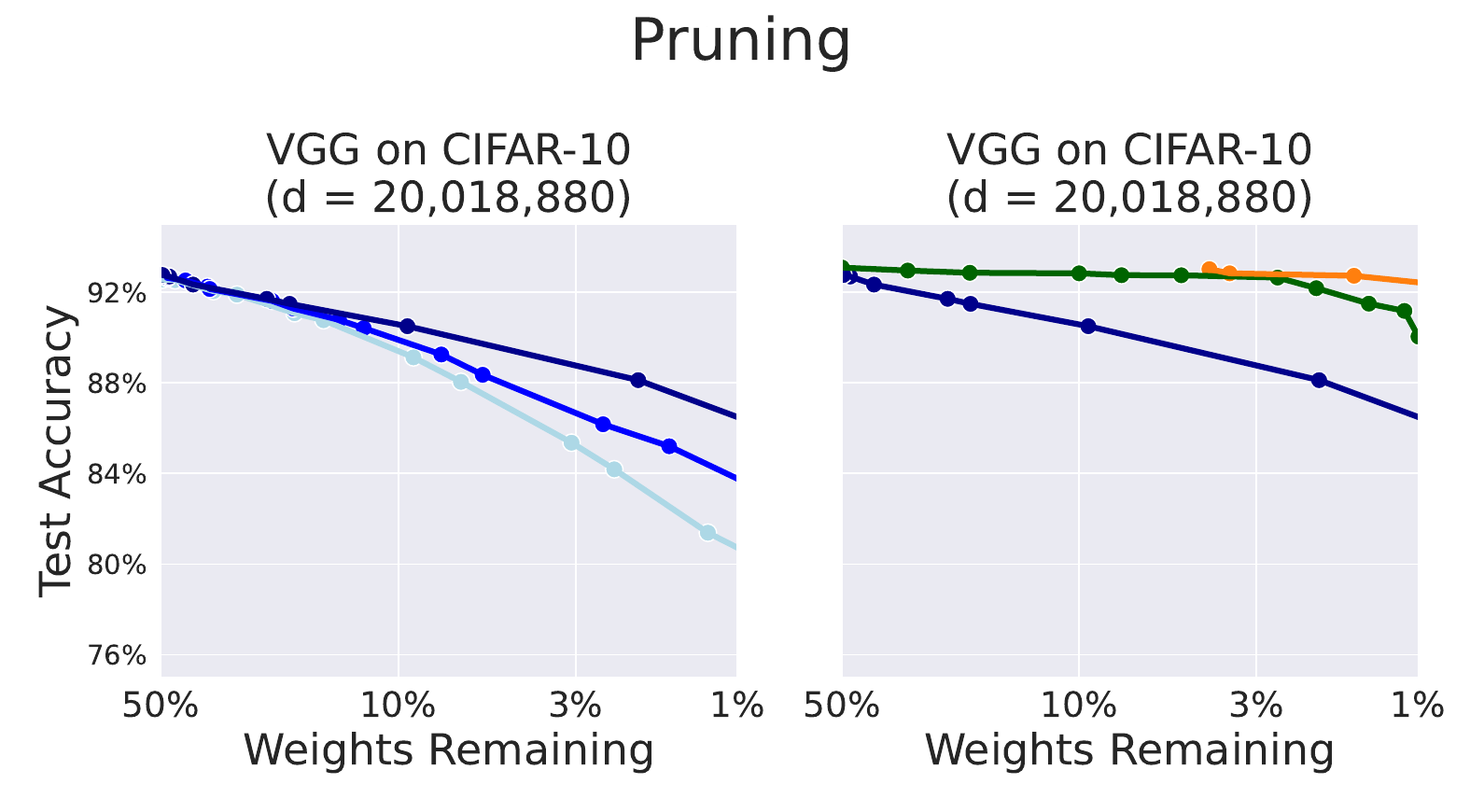}
		\caption{Moderate sparsity}
	\end{subfigure}%
	\begin{subfigure}{132.495pt}
		\includegraphics[height=132.495pt]{\main/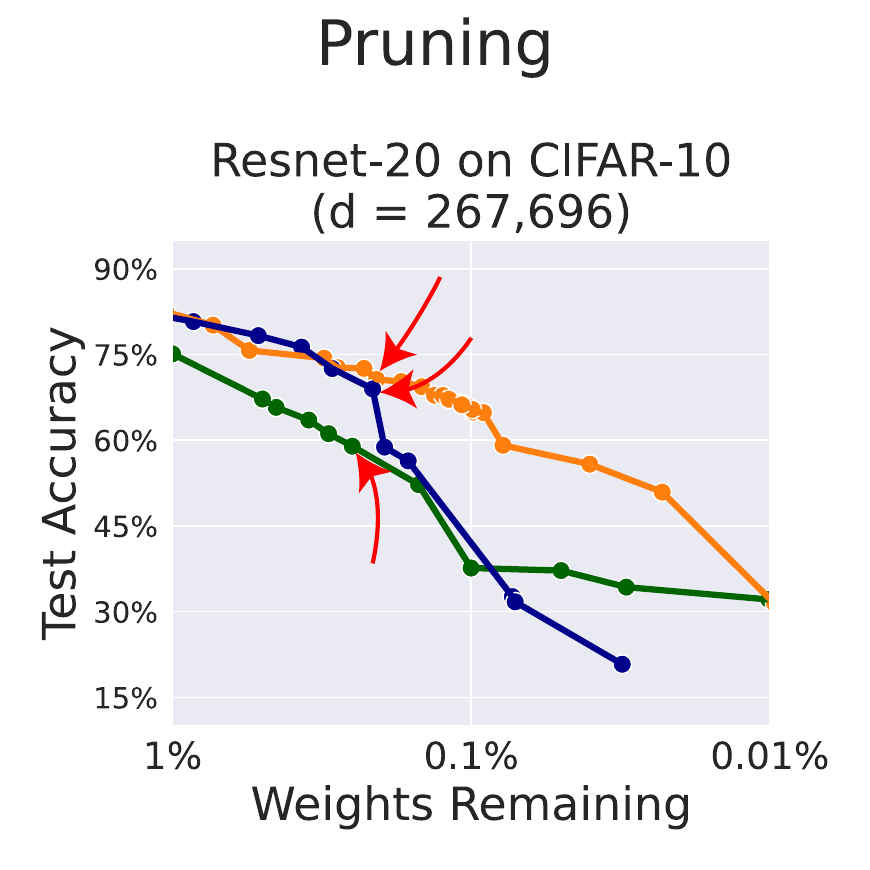}
		\includegraphics[height=132.495pt]{\main/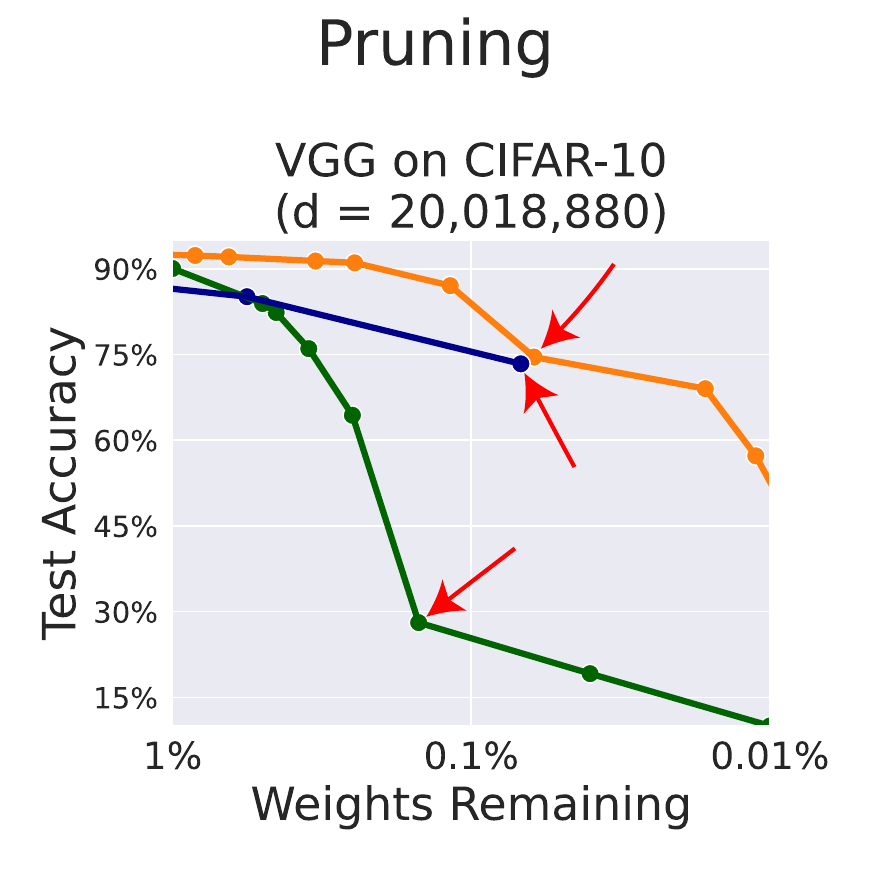}
		\caption{High sparsity}
	\end{subfigure}
	\caption{Pareto front for joint pruning experiments. We analyzed settings marked with arrows more deeply in \Cref{sec:sparsities_throughout}.}
	\label{fig:joint_pruning_results}
\end{figure}

In \Cref{fig:joint_pruning_results} we first plot MC settings only, then we show MC (n=100) together with CP and the magnitude-based approaches. Lastly, we plot only the best performing methods from each category on extreme sparsity ranges. MP does one-shot pruning and GMP follows a cubic interpolation.

In all these experiments, CP was very dominant against the other methods. It should be noted, however, that this was only when using SGD, in accordance to MaskedNNRegression and in contrast with the supermask experiments. Although we do not plot it, results with RMSprop were significantly worse (but still better than MC). It seems that the best optimizer is problem-dependent for CP approaches, which, we repeat, is in contrast with our observations for MC, where RMSprop was always better (except for the microworlds, where SGD was reasonable for all methods used due to the problems being much smaller). In practice, it seems beneficial to try both optimizers for CP, but we need further understanding of its inner workings.

In \Cref{sec:sparsities_throughout}, we visualize the sparsity allocation of the runs indicated with arrows in \Cref{fig:joint_pruning_results} throughout training. After inspecting the plots, it becomes evident that GMP and CP have much more freedom in changing the initial (\ie, dense) configuration. MC, perhaps due to difficulty of the search, does not change the structure as much: the majority of the computation from the dense network happens on high level layers initially, MC methods mostly keep it this way.

\FloatBarrier

\section{Conclusion}
\label{ch:conclusion}
Perhaps one of the main difficulties of adapting PB optimization to neural network contexts is that most of their knowledge base was built around smooth optimization. Techniques such as backpropagation, initialization schemes, optimizers, regularization, normalization, architectures, activations and even loss functions were designed with continuous optimization in mind. In multivariate calculus, the gradient of the loss is already indicative of the best solution in a small neighbourhood. At least locally, a short step in that direction should give the desired outcome. Then, researchers and practitioners formulate their objectives in terms of such functions, while attempting to make the problem as ``convex''-like as possible. Finally, neural network design connects this objective with a computational graph that takes advantage of this differentiable framework.

In contrast, a similar set of techniques is lacking when combining larger-scale PB optimization and neural networks. Problem formulation will not necessarily involve functions that interact particularly well with discrete parameters. The unavailability of equivalent closed-form expressions for the best immediate direction signifies searching exhaustively for such a solution. Still, PB optimization problems resurface when attempting to improve certain aspects of machine learning.

As a consequence of the referred disparity, approaches trying to extend smooth optimization to discrete contexts have to overcome the incompatibility between the two settings. One such approach involves using the gradient of the loss and annealing the problem, as described in this study. Unfortunately, most such methods are afterthoughts and still poorly understood. Perhaps one future direction for improving the PB optimization solutions is to adapt problem and architecture designs by also considering the discrete variables.

\Cref{tab:app_results_summ} summarizes all experimental results from this work. By analyzing them and the background discussion, CP and MC may not be the best choices for general PB optimization. Recall that a simple greedy search will definitely find the optimal solution if given enough computation. The same may not be true for MC methods, as per the discussion on \Cref{sec:unwanted_gen}. These methods are greatly harmed by larger dimensionality. In addition to the inherent difficulty of searching in a large space, exploding importance sampling ratios and unwanted generalization between evaluations make it even less applicable. Despite renewed interest in proposing new estimators with increasingly lower variance, our underwhelming results using optimal control-variates or even true gradients suggest this might be insufficient. Notably, even using $100$ samples was not enough to outperform CP in the pruning experiments.

In spite of concerns regarding the sigmoid, our results indicate that using it is unlikely to be the bottleneck for this use case. Widely different parametrizations appear to have the same limitations. Alternative ideas for improving MC estimation include modeling dependency between the discrete variables and mapping a smaller parameter vector to $\bz$, instead of working directly with $d$ dimensions.\footnote{As a sidenote, depending on how one implement these changes, the equivalence of the PB and MC objectives discussed in \Cref{sec:prob_framework} might be lost.} As none of these avoid the limitations discussed, we believe they are unlikely to lead to long-term advances, serving at best to slightly improve results on some benchmark. Clearly, regardless of having lower-dimensional parameters or dependent variables, the resulting method will still ultimately be searching a complex function in an enormous space by mere trial-and-error.

On the other hand, there is no specific reason why adapting numerical continuation ideas to bridge the gap between continuous and discrete domains should work. Studies often rely on the general intuition of curriculum learning instead. As explained in \Cref{sec:cp_basics}, homotopy methods were designed to find roots of a non-linear system of equations. Given some assumptions, which we overviewed then, a path between the simplified problem and the original one will exist. Gradient-based optimization has a similar motivation: the system equations correspond to the dimensions of the gradient vector and the roots are the critical points. Thus, when CP methods follow the the negative gradient of the smooth loss, there is an implicit attempt to find such points for increasingly lower temperatures. Still, it is not clear why finding a critical point is desirable for infinitely small temperature, or how that connects to the differences between multiple $\Zeta_h$. The original assumptions from numerical continuation are likely false in this context.

What ends up happening in these methods is the extrapolation of local behaviour. Although reconstructing $J(\cdot)$ away from the current point is no new endeavour, with Taylor expansions and Padé approximant being some examples, precision decays for far-away regions, which are exactly the ones needed here. Moreover, accurate results require higher-order terms, whose computation is prohibitive if $d$ is large. Our results indicate that these extrapolations are indeed unreliable.

\begin{figure}[tb!]
	\centering
	\begin{tabular}{c}
		\includegraphics[height=1.9em]{\main/img/numerical_continuation/legend.pdf}
	\end{tabular}
	
	\begin{tabular}{c c}
		\begin{subfigure}{0.3\columnwidth}
			\includegraphics[width=\columnwidth]{\main/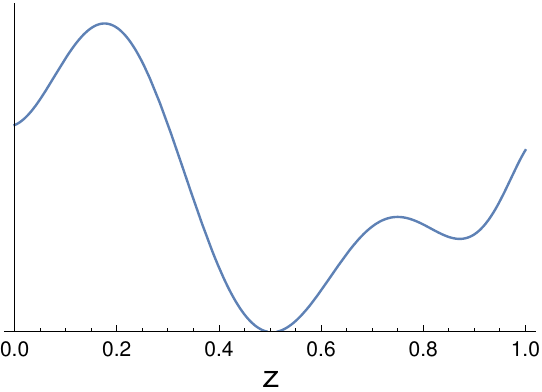}
			\caption{$J(z)$}
		\end{subfigure}
		&
		\begin{subfigure}{0.3\columnwidth}
			\includegraphics[width=\columnwidth]{\main/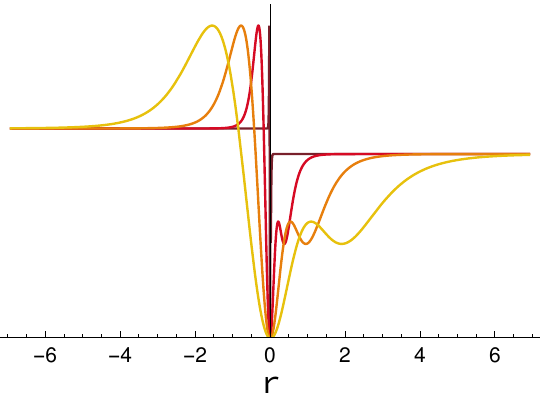}
			\caption{$J\left(\sigmoid \Big(\frac{r}{\tau}\Big)\right)$}
		\end{subfigure}		
	\end{tabular}
	\caption{Temperature annealing collapses $J(\cdot)$ at the origin.}
	\label{fig:collapse_of_j}
\end{figure}

In addition to that, CP methods resulted in solutions oscillating considerably throughout training in some experiments (see \Cref{fig:deceiving_piecewise_runs,fig:microworld_grad,fig:joint_pruning_per_layer}) and one optimizer performing much better than the other depending on the problem. To attempt to understand why, we can take $d=1$ and note that the chained function $J(\sigma(\cdot/\tau))$ has an active range $(-\Delta r,\Delta r)$, where behaviour is similar to that of $J(\cdot)$ when restricted to $[0,1]$. For $r$ outside of this range, the chained function will be constant: $J(\zeta_0)$ for lower values or $J(\zeta_1)$ for higher values. As $\tau \to 0$, the whole active range collapses towards the origin, causing $r_t$ to traverse many of the critical points from $J(\cdot)$. The parameter will then change direction depending on whether its per-timestep neighbourhood is ascending or descending. The interplay between these dynamics and the optimizer might have resulted in the sensitivity we observed. See \Cref{fig:collapse_of_j} for an illustration.

We did not discuss magnitude-based approaches much in this work. Since the (implicit) masks come from training the main weights, which use the gradient of the loss, they are naturally closer to CP methods than to MC. The connection between the two goes beyond that. Considering a single weight $w$ from the neural network, CP rewrites it as:
\begin{equation*}
	w' = w \sigma\left(\frac{r}{\tau}\right)
\end{equation*}
and weights with $r>0$ by the end of training are not pruned. The chain rule in this case gives
\begin{equation*}
	-\frac{\partial J}{\partial r} = w \frac{\partial \sigma(r/\tau)}{\partial r} \left(-\frac{\partial J}{\partial w'}\right).
\end{equation*}
The derivative of the sigmoid is never negative, regardless of $\tau$, whereas the rightmost factor corresponds to the feedback signal from the problem to $w'$. The LHS is the direction used when updating $r$. From this expression, $r$ increases in two situations: when $w$ is positive and the gradient feedback indicates it should go up, or when $w$ is negative and the same signal indicates it should go down.

Interestingly, a positive weight that always goes up or a negative weight that always goes down will tend to have larger magnitude, which explains the connection. When used explicitly, this exact criterion: pruning weights by removing those moving toward zero, was shown to work by itself multiple times \parencite{zhou2019deconstructing, sanh2020movement, bellec2017deep}. This observation can perhaps help explain recent successes when generating sparse masks via CP approaches.

One remaining open question is what would happen to joint training if computation was not a problem and we could find the optimal solution of the PB optimization. Using pruning as an example, on the one hand, this perhaps could lead to major improvements and neural networks significantly smaller than the current state-of-the-art would perform closely to dense models. In contrast, it could also be the case that training of the main network ends up being robust to worse mask choices, which could then imply that using the optimal $\bz^*$ does not improve the current state-of-the-art much.

One key point for future work is answering what changed for the methods using $\nabla_{\bz}J(\bz)$ when combined with neural networks. Despite the main NN being fixed, MaskedNNRegression and supermask had them performing well. We believe studying similar phenomena more analytically is paramount for future progress, as it provides understanding that might not be achievable by more practical investigation on larger models. Broadly speaking, the scale of such systems and the excess of confounding factors limit the range of possible tools when interpreting results.

Of equal importance is the design of smaller experiments that, despite the scale, still capture the essence of the problem. Relative computational feasibility allows for more trial-and-error of new ideas, while also benefiting from a wider range of tools. Using pruning as an example once more, if the input-output mapping is part of the experimental design, we can compare the size of solutions found to the actual number of degrees of freedom, which is not possible for image classification problems.

Better understanding the aforementioned behaviour is a more immediate step that perhaps can lead to more conscious model and algorithmic design. These advances, in turn, can likely be adapted to a wide array of different fields and problems, including continual learning, multi-task learning, pruning, transfer-learning, ticket search, Bayesian optimization, training binary/ternary NN and mixed-precision NN quantization.

% Acknowledgements and Disclosure of Funding should go at the end, before appendices and references

\acks{We gratefully acknowledge funding from the Natural Sciences and Engineering Research Council of Canada (NSERC), 
	the Canada CIFAR AI Chair program, and the Alberta Machine Intelligence Institute (Amii).}


\newpage

\appendix

\section{Additional Discussion}
This section provides more details about points mentioned in the main text.

\subsection{Additional PB Optimization Examples}
\label{sec:pb_examples}

Here, we complement \Cref{sec:pb_nn_examples} by presenting some general PB optimization examples. Sometimes the problem will be presented in terms of a maximization. Recall, however, that we can always find an equivalent minimization, since

\begin{equation*}
	\max_x J(x) = - \min_x -J(x).
\end{equation*}

\begin{example}[Maximum Satisfiability]
	This example is adapted from Section 3 of \textcite{boros2002pseudo}. Maximum satisfiability is a very frequently studied problem in applied mathematics/theoretical computer science. The input is a family $\mathscr{C}$ of clauses $C$, where each of them is a set $C \subseteq \{z_1,\nz_1,\cdots,z_d,\nz_d\}$. A clause is satisfied when any of its literals evaluate to one. In logic notation, considering $u_i \in C$ as the literals in an arbitrary clause, the clause is satisfied when:
	\begin{equation*}
		\prod_{u \in C} \overline{u} = 0,
	\end{equation*}
	the problem corresponds to the following minimization:
	\begin{equation*}
		\argmin{\bz \in \corners} \sum_{C \in \mathscr{C}} \prod_{u \in C} \overline{u}.
	\end{equation*}	
\end{example}

\begin{example}[Maximum Independent Set]
	The following is also adapted from Section 3 of \textcite{boros2002pseudo}. Considering a graph $\sG=(\sV,\sE)$, with vertex set $\sV$ and edge set $\sE$, where $\sV = \{1,\cdots, d\}$, an independent set is one in which no two vertices belong to the same edge $e \in \sE$. The problem becomes finding:
	\begin{equation*}
		\argmax{\sU \subseteq \sV} |\sU| \qquad \text{s.t.\ } N(v) \cap \sU = \emptyset,\  \forall v \in \sU.
	\end{equation*}	
	Where $N(v)$ represents the neighbours of vertex $v$. \textcite[Theorem 2.1]{tavares2008new} shows that the maximum of this problem is equal to:
	\begin{equation*}
		\max_{\bz \in \corners} \sum_{i \in \sV} z_i - \sum_{(i,j) \in \sE} z_i z_j.
	\end{equation*}
	Conversion of an arbitrary $\bz = \bz^*$ of this size to an independent set will take at most $(|V| - \sum z_i^*)$ iterations.
\end{example}

\begin{example}[Image segmentation]
	Computer vision problems often involve assigning discrete labels for every pixel in an image. For example, one may want to discern foreground from background, or to identify pixels corresponding to human skin in an image. In tracking, the goal is to follow some object on a sequence of frames, which also involves labeling pixels as either pertaining to the object or not \parencite{szeliski2022computer}.
	
	Assuming we have two labels, such as in binary image segmentation, the loss can often be formulated using energy minimization \parencite{lecun2006tutorial, szeliski2008comparative}, and the problem becomes finding the following quantity:
	\begin{equation*}
		\argmin{\bz \in \corners} E_d(\bz) + \lambda E_s(\bz).
	\end{equation*}
	Where $E_d$ is the data energy, which penalizes deviation from the data, whereas $E_s(\bz)$ is the smoothness energy, which enforces spacial coherence, and $\lambda$ is just a tradeoff parameter.
\end{example}

\subsection{Additional Properties of the Multilinear Objective}
\label{sec:pb_basics_appendix}
We here give some additional details of the objective from \Cref{sec:pb_basics}. The following property gives us an additional tool to recognize functions of this form.

\begin{restatable}{theorem}{mth}\label{thm:multilinear_hessian}
	An arbitrary twice differentiable function $f:\R^d \rightarrow \R$ will have a Hessian whose diagonal is equal to $\zeros = [0,\cdots,0]^\top$ if and only if it can be represented in the form of \Cref{eq:multi_poly_set}.
\end{restatable}
\begin{proof}
	See \Cref{sec:multilinear_hessian_proof}.
\end{proof}	
Interestingly, the theorem means that any function with Hessian diagonal zero implicitly corresponds to this canonical form of the PB optimization objective. Inspecting \Cref{eq:multi_poly}, we can see that its Hessian has diagonal equal to zero, so, by \Cref{thm:multilinear_hessian}, it can indeed be expressed in the form of \Cref{eq:multi_poly_set}. Additionally, for an arbitrary pseudo-Boolean function $J(\cdot)$, its multilinear polynomial form is unique \parencite[Section 4.1]{boros2002pseudo}. We present an example of these forms below.

\begin{example}
	\label{ex:pb_example}
	Considering the following pseudo-Boolean function for $d=2$:
	\begin{equation*}
		J(\Zeta_0) = A ; J(\Zeta_1) = B ; J(\Zeta_2) = C ; J(\Zeta_3) = D.
	\end{equation*}
	
	Expressing $\mP_J(\bz)$ as in \Cref{eq:multi_poly} corresponds to:
	\begin{align*}
		\mP_J(\bz) &= (1 - z_1) (1 - z_2) A + z_1 (1 - z_2) B + (1 - z_1) z_2 C + z_1 z_2 D\\
		&= \nz_1 \nz_2 A + z_1 \nz_2 B + \nz_1 z_2 C + z_1 z_2 D,
	\end{align*}%
	whereas \Cref{eq:multi_poly_set} corresponds to:
	\begin{align*}
		\mP_J(\bz) &= A + (B-A) z_1 + (C-A) z_2 + (A-B-C+D) z_1 z_2.
	\end{align*}	
\end{example}

\subsection{Algorithmic Approches for PB optimization}
\label{sec:pb_algs_detailed}
Here we give some explanation of the approaches listed in \Cref{sec:pb_algs}. By reason of the exponential number of solutions, it may be necessary to waive the search for the global optimum and instead stop it when a local solution is reached. Local search methods are a broad class of iterative approaches that look for the best solution in a neighbourhood $N(\bz)$ around the current solution $\bz$ (see, for example, \citealt{schuurmans2001local} or sections 4.3 and 4.6 of \citealt{boros2002pseudo}). Considering, for instance, a neighbourhood using the Hamming distance, which here corresponds to the number of indices where two binary vectors differ, the method may return a local solution $\bzh^*$ satisfying:

\begin{equation*}
	J(\bzh^*) \leq J(\bz), \qquad \text{for } \bz \in N(\hat{\bz}^*).
\end{equation*}

The high level approach of the deterministic version of these methods consists in repeatedly choosing the best performing neighbour as the next solution. Nevertheless, a greedy search around the current value is likely to get stuck in a poor local optimal, justifying the emergence of alternative approaches that allow some degree of exploration. As an example, tabu search methods \parencite{glover1989tabu, glover1990tabu} accept changing the solution for another whose value of $J(\cdot)$ is higher, in case it is no longer possible to get improvement in the current neighbourhood. Moreover, they discourage re-visiting previous solutions by means of a technique they call \emph{prohibitions}.

Similarly, simulated annealing \parencite{salamon2002facts, van1987simulated} allows some degree of exploration, controlled by a temperature parameter. When the temperature is high, the model will have a higher probability of choosing a next solution worse than the current one. As the temperature goes to zero, the probability of the worse choice becomes arbitrarily close to zero. The method then follows a temperature schedule based on the computational budget. The name annealing comes from an technique from metallurgy, where the material goes through heating and then controlled cooling to alter its physical properties.

Yet another popular approach surfaces by interpreting some problems as graph optimization. In computer vision, for example, \textcite{szeliski2008comparative, boykov2001fast} re-cast problems involving pixel labelling as finding the minimum cut of a graph. The \emph{max-flow min-cut theorem} \parencite{dantzig2003max} draws a connection between this problem and the maximum value of the flow in a flow network, allowing the reuse of efficient algorithms from graph theory to speed-up the search for approximate solutions \parencite{edmonds1972theoretical, dinic1970algorithm}.

Alternatively, genetic algorithms take inspiration from natural selection to find the approximate solution while avoiding getting stuck on poor local optima \parencite{kramer2017genetic}. They start with a population of randomly initialized solutions (called individuals) which correspond to an iterate (called generation) of the algorithm. The method assigns each candidate solution a fitness score, which, in our case, is based on the function $J(\cdot)$. Then, selection of a subset of the individuals happens according to this fitness score\footnote{Sometimes it may be wise to also include less fit candidates to make the population more diverse.} and is succeeded by a merging phase. Merging follows a problem-specific formulation and outputs the candidates of the next generation. In our case, to illustrate, one possibility is to select a cutoff index and swap the values before and after the index between the two parents. Moreover, to further increase stochasticity, randomly changing some part of the children via mutation is common practice. The algorithm ends when a desired fitness level is achieved, the fitness plateaus or the computational budget is used up. Notably, this method can help avoid poor local minima due to the multiple distinct initializations.

Meanwhile, branch-and-bound algorithms \parencite{lawler1966branch} are also prominent. The essence of such methods is to break down the search, recursively splitting the original search space into smaller subspaces. For each subspace $\sS$, the algorithm will compute a lower bound on $J(\bz)$ for $\bz \in \sS$ and will use its value to discard subspaces where it is certain that the desired solution will not be found. The problem instance (and the sub-problem instances generated thereafter) must define three operations: (1) \emph{branching}, which breaks the instance into sub-instances; (2) \emph{bounding}, which computes the lower bound on the function for the current instance (3); \emph{returning candidate}, which returns a candidate solution from the set of values in the current instance.

Finally, quadratization methods reduce the degree of the original multilinear polynomial from \Cref{eq:multi_poly_set}, rewriting it as:
\begin{equation*}
	\mP'(\bz') = a_\emptyset + \sum_{i \in \sD'} b_i z'_i + \sum_{\substack{i,j \in \sD' \\ i < j}} c_{ij} z'_i z'_j.
\end{equation*}
Where $d \leq d'$, $\sD \subseteq \sD' = \{1, \cdots, d'\}$ and $\bz' \in \{0,1\}^{d'}$. One can always transform the higher-degree polynomial iteratively by increasing the size of the set by one each time. Adding a new element can, for example, be performed by selecting a pair from the previous iteration set and having the new member correspond to their product. Weights from the old $\mP(\cdot)$ can then be changed in a way that reduces the degree of some of the summands by one while also enforcing that the new element is indeed the product of the chosen pair. For more details of the complete example, refer to \textcite[Section 4.4]{boros2002pseudo}. \textcite{rosenberg1975reduction} shows that this procedure can be completed in polynomial time. Furthermore, if the following submodularity condition is satisfied:
\begin{equation*}
	\mP'(\ones^A) + \mP'(\ones^B) \geq \mP'(\ones^{A \cup B}) + \mP'(\ones^{A \cap B}),
\end{equation*}
where $\ones^\sS$ is the vector with ones in the positions contained in the set $\sS$, finding the solution $\mP(\bz^*)$ in polynomial time is possible \parencite{grotschel1981ellipsoid}. However, this condition is not always satisfied and the general problem is NP-hard. Nonetheless, development of algorithms tailored to the quadratic case is an active area of research. \textcite{tavares2008new} overviews some quadratic PB optimization methods.

To conclude, sometimes deriving sub-optimality bounds on an approximation of the true solution and finding this approximation in polynomial time is feasible. Algorithms that search for solutions this way are called approximation algorithms. In general, not all PB optimization problems admit polynomial-time algorithms with reasonable approximation guarantees as discussed in \textcite[Chapter 16]{williamson2011design}.

\subsection{Failure Cases of Magnitude Pruning}
\label{sec:mp_failure}

\begin{figure}[htb!]
	\centering
	\begin{subfigure}[b]{0.3\columnwidth}
		\centering
		\includegraphics[width=0.9\columnwidth]{\main/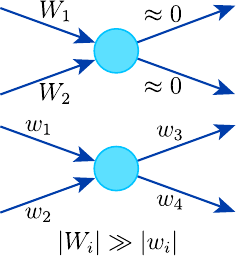}
		\caption{Output connections}
		\label{sub:mp_counter_a}
	\end{subfigure}%
	\begin{subfigure}[b]{0.7\columnwidth}
		\centering
		\includegraphics[width=0.8\columnwidth]{\main/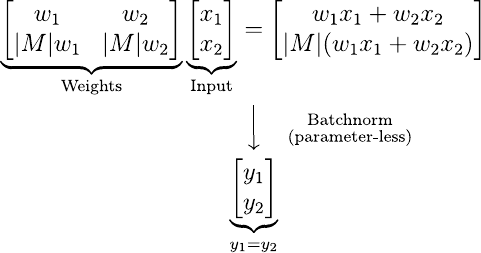}%
		\caption{Batch normalization}
		\label{sub:mp_counter_b}
	\end{subfigure}
	\caption{Examples where weight magnitudes do not correlate with their importance. On \Cref{sub:mp_counter_a}, the larger weights have less influence on the model output than the smaller ones due to their outgoing connections. In \Cref{sub:mp_counter_b}, one row is a scaled version of another. Despite the difference in magnitude, subsequent batch normalization equates their outputs.}
	\label{fig:mp_counter}
\end{figure}

\Cref{fig:mp_counter} depicts some simple examples where the assumption that a larger weight magnitude implies higher importance does not hold.

\section{Proofs}
This section contains some proofs from results in the main text.
\subsection{Multilinear Form}
\label{sec:multilinear_hessian_proof}
Although we believe that this results is known in the literature, we could not find it stated anywhere, so we derive it below. For these proofs, considering $\sD = \{1,2,...,d\}$, we will use the following form, which is a common way to represent the multilinear polynomial from \Cref{eq:multi_poly_set} \parencite[see][Section 2]{boros2002pseudo}:
\begin{equation*}
	\mP_J(\bz) = \sum_{\sS \subseteq \sD} \left( \prod_{i=1}^d z_i^{(\ones^\sS)_i} \nz_i^{(\overline{\ones^\sS})_i} \right) J(\ones^\sS),
\end{equation*}
where:
\begin{equation}
	\label{eq:index_set_notation}
	(\ones^\sS)_i = \begin{cases} 1 & \text{if } i \in \sS\\ 0 & \text{otherwise} \end{cases}.
\end{equation}
We start with the following lemma, then proceed to the main proof.
\begin{lemma}
	\label{lem:multilinear_hessian_proof_lemma}
	As in the main text, we denote $\sD = \{1,2,...,d\}$. With $\bx \in \R^d$ and arbitrary $i$, consider the following:
	\begin{equation*}
		g(\bx) = \sum_{\sS \subseteq \sD \setminus \{i\}} q_{\sS}(x_i) \prod_{j \in \sS} x_j.
	\end{equation*}
	$g(\bx)$ being independent of $x_i$ implies $q_{\sS}(x_i)$ is also independent of $x_i$ for all $\sS \subseteq \sD \setminus \{i\}$.
\end{lemma}
\begin{proof}
	To simplify exposition, we will consider only $\bz \in \corners$ instead of $\bx$. Extension to $\R^d$ follows a similar argument, which we mention in the end. Using notation from \Cref{eq:index_set_notation}, $g(\bz)$ independent of $z_i$ means:
	\begin{equation*}
		g(\ones^\sS) = g(\ones^{\sS \cup \{i\}}), \hspace{1.5em}\text{for }\sS \subseteq \sD \setminus \{i\}.
	\end{equation*}
	Also, for $\sS \subseteq \sD \setminus \{i\}$, we have:
	\begin{equation*}
		g(\ones^\sS) = \sum_{\sT \subseteq \sS} q_{\sT}(0) \hspace{1em}\text{ and }\hspace{1em} g(\ones^{\sS \cup \{i\}}) = \sum_{\sT \subseteq \sS} q_{\sT}(1) .
	\end{equation*}
	We want to show that:
	\begin{equation*}
		q_{\sS}(0) = q_{\sS}(1), \hspace{1.5em}\text{for }\sS \subseteq \sD \setminus \{i\},
	\end{equation*}
	which we prove by induction:
	\begin{itemize}
		\item \textbf{Base case:} we consider the empty set
		\begin{equation*}
			g(\ones^{\emptyset}) = q_{\emptyset}(0) \hspace{1em}\text{ and }\hspace{1em} g(\ones^{\{i\}}) = q_{\emptyset}(1).
		\end{equation*}
		By our main assumption from the statement, $g(\ones^{\emptyset}) = g(\ones^{\{i\}})$. Therefore:
		\begin{equation*}
			q_{\emptyset}(0) = q_{\emptyset}(1).
		\end{equation*}
		\item \textbf{General case:} assume that, for $\sS' \subseteq \sD \setminus \{i\}$ and $|\sS'| \leq d - 1$, we have $q_{\sS'}(1) = q_{\sS'}(0)$. Then, for $\sS \subseteq \sD \setminus \{i\}$ and $|\sS| = d$:
		\begin{align*}
			g(\ones^{\sS \cup \{i\}}) &= \sum_{\sT \subseteq \sS} q_{\sT}(1) \\
			&= q_{\sS}(1) + \sum_{\sT \nsubseteq \sS} q_{\sT}(1) \\
			&= q_{\sS}(1) + \sum_{\sT \nsubseteq \sS} q_{\sT}(0)
		\end{align*}
		and, by our main assumption, we know that:
		\begin{equation*}
			g(\ones^{\sS \cup \{i\}}) = g(\ones^\sS) = \sum_{\sT \subseteq \sS} q_{\sT}(0).
		\end{equation*}
		Combining both:
		\begin{align*}
			q_{\sS}(1) + \sum_{\sT \nsubseteq \sS} q_{\sT}(0) &= \sum_{\sT \subseteq \sS} q_{\sT}(0) \\
			q_{\sS}(1) &= q_{\sS}(0).
		\end{align*}
	\end{itemize}
	To extend this proof to $\R^d$, simply use $z_i = x'$ and $z_i = x''$ in the comparisons, instead of $0$ and $1$ and allow arbitrary $z_j$ for $j \neq i$ in place of $z_j = 1$. The additional terms due to the product of the coordinates $\neq i$ will cancel.
\end{proof}

\ifdefined\mth{
	\mth*
}
\fi

\begin{proof}
	We consider $x \in \R^d$ and start by proving:
	\begin{equation*}
		f(\bx) = \mP_J(\bx) \Rightarrow \frac{\partial^2 f(\bx) }{\partial x_i^2} = 0, \hspace{1.0em}\text{for all }i.
	\end{equation*}
	Which follows from the fact that no $x_i$ appears squared for $\mP_J(\bx)$ following \Cref{eq:multi_poly_set}. Then, we prove the reverse implication:
	\begin{equation*}
		\frac{\partial^2 f(\bx) }{\partial x_i^2} = 0 \hspace{1.0em}\text{for all }i \Rightarrow f(\bx) = \mP_J(\bx).
	\end{equation*}
	By induction:
	\begin{itemize}
		\item \textbf{Base case:} We assume $f: \R \rightarrow \R$. Using integration we have
		\begin{align*}
			\frac{\partial^2 f(x) }{\partial x^2} &= 0 \\
			f(x) &= ax + b \\
			&= w_{1} x + w_{\emptyset}.
		\end{align*}
		\item \textbf{General case:} we assume that, for $f':\R^{d-1} \rightarrow \R$
		\begin{align}
			\frac{\partial^2 f'(\bx) }{\partial x_i^2} = 0, \hspace{1.0em}\text{for }1 \leq i \leq d-1 \Leftrightarrow f'(\bx) &= \mP_{J'}(\bx) \nonumber \\
			& = \sum_{\sS \subseteq \{1,\dots,d-1\}} w'_{\sS}\ \prod_{i \in \sS} x_i. \label{eq:app_mlp_proof_induction}
		\end{align}
		Then, we consider a function $f^d: \R \rightarrow \R$. Assuming that:
		\begin{equation*}
			\frac{\partial^2 f(\bx) }{\partial x_i^2} = 0, \hspace{1.0em}\text{for all }1 \leq i \leq d,
		\end{equation*}
		we want to show that this implies $f(\cdot)$ can be written in the form of \Cref{eq:multi_poly_set}. Fix an arbitrary $i$, and note that, for any value of $x_i$, since the other $d-1$ coordinates have corresponing entries on the diagonal of the Hessian equal to zero, we can write $f(\cdot, x_i, \cdot)$ as in the RHS of \Cref{eq:app_mlp_proof_induction} (with indices possibly remapped). The specific $d-1$ multilinear polynomial from \Cref{eq:app_mlp_proof_induction} will depend on the value of $x_i$. Therefore, we can write $f(\cdot)$ as:
		\begin{equation*}
			f(\bx)	= \sum_{\sS \subseteq \sD \setminus \{i\}} w_{\sS}(x_i) \prod_{j \in \sS} x_j.
		\end{equation*}
		Taking the derivative:
		\begin{equation*}
			\frac{ \partial f(\bx)}{\partial x_i}	= \sum_{\sS \subseteq \sD \setminus \{i\}} \frac{\partial w_{\sS}(x_i)}{\partial x_i} \prod_{j \in \sS} x_j.
		\end{equation*}	
		In addition, by assumption, we also have $\partial^2 f(\bx)/\partial x_i^2 = 0$. Consequently, $\partial f(\bx)/\partial x_i$ is constant with respect to $x_i$. Using \Cref{lem:multilinear_hessian_proof_lemma}, with $g(\bx) = \partial f(\bx)/\partial x_i$ and 
		$q_{\sS}(x_i) = \partial w_{\sS}(x_i)/\partial x_i$, we have that $\partial w_{\sS}(x_i)/\partial x_i$ is also independent of $x_i$. Meaning:
		\begin{equation*}
			w_{\sS}(x_i) = w'_{\sS}	+ x_i w''_{\sS}.
		\end{equation*}
		Finally, this implies:
		\begin{align*}
			f(\bx) &= \sum_{\sS \subseteq \sD \setminus \{i\}} w_{\sS}(x_i) \prod_{j \in \sS} x_j \\
			&= \sum_{\sS \subseteq \sD \setminus \{i\}} (w'_{\sS}	+ x_i w''_{\sS}	) \prod_{j \in \sS} x_j \\
			&= \sum_{\sS \subseteq \sD} w_{\sS}\ \prod_{i \in \sS} x_i.
		\end{align*}
		Where, for any $\sS \subseteq \sD \setminus \{i\}$, we have $w_{\sS} = w'_{\sS}$ and $w_{\sS \cup \{i\}} = w''_{\sS}$.
	\end{itemize}
\end{proof}

\subsection{Iterative ARMS Procedure}
\label{sec:arms_iter_proof}
To prove that the sampling procedure is equivalent, since dimensions are independent, we restrict ourselves to $d=1$ and omit the index $i$. Assume:
\begin{equation*}
	E^{(s)} = - \log U^{(s)}, \hspace{1.5em}\text{for }U^{(1)},\dots, U^{(n)}\overset{i.i.d.}{\sim}U[0,1].
\end{equation*}
Note also that the original ARMS procedure from \Cref{alg:orig_arms} samples the Dirichlet by using:
\begin{equation}
	\label{eq:app_arms_dirichlet}
	d^{(s)} = \frac{\log U^{(s)}}{\sum_{s'=1}^{n}\log U^{(s')}} = \frac{E^{(s)}}{\sum_{s'=1}^{n}E^{(s')}}.
\end{equation}
Furthermore, for $E^{(s)}$ as defined above and considering the exponential and Gamma distributions, the following holds:
\begin{equation*}
	E^{(s)} \sim \mathrm{Exp}[1] = \mathrm{Gamma}[1,1].
\end{equation*}
The PDF of $E^{(s)}$ is:
\begin{equation*}
	f_{E^{(s)}}(x) = \euler^{-x}, \hspace{1.5em}\text{for }x \geq 0.
\end{equation*}
Additionally, the sum of these multiple independent Gamma random variables with common scale parameter satisfies
\begin{align}
	\label{eq:arms_iter_app_eq_gamma}
	\sum_{s'=1}^{n} E^{(s')} \sim \mathrm{Gamma}[n,1], && f_{\sum E^{(s')}}(x) = \frac{x^{n-1}\euler^{-x}}{(n-1)!},\hspace{1.0em}\text{for }x \geq 0.
\end{align}
The procedure involves sampling the whole $\sum_{s'=1}^{n} E^{(s')}$ sum at once and then iteratively sampling each of its composing summands $E^{(s)}$. We therefore want to know 
the PDF of $E^{(s)}$ given the sum. By Bayes rule:

\begin{align*}
	f_{E^{(s)}|\sum E^{(s')}=R}(x) &= \frac{f_{E^{(s)}}(x) f_{\sum E^{(s')}|E^{(s)}=x}(R)}{f_{\sum E^{(s')}}(R)} \\
	&= \frac{f_{E^{(s)}}(x) f_{\sum_{s' \neq s} E^{(s')}|E^{(s)}=x}(R-x)}{f_{\sum E^{(s')}}(R)} \\
	&= \frac{f_{E^{(s)}}(x) f_{\sum_{s' \neq s} E^{(s')}}(R-x)}{f_{\sum E^{(s')}}(R)} \tag*{$\triangleright$ \parbox[t]{2.0cm}{\raggedright Since $E^{(s)}$ and $E^{(s')}$ are i.i.d.}} \\
	&= \frac{\euler^{-x} \left( \frac{(R-x)^{n-2}\euler^{-(R-x)}}{(n-2)!} \right)}{\frac{R^{n-1}\euler^{-R}}{(n-1)!}}, \hspace{1.5em}\text{for }0\leq x \leq R \\
	&= \frac{n-1}{R} \left( 1 - \frac{x}{R} \right)^{n-2}.
\end{align*}
The CDF of this r.v. is:
\begin{align}
	F_{E^{(s)}|\sum E^{(s')}=R}(x) &= \int_{0}^{x} f_{E^{(s)}|\sum E^{(s')}=R}(x') \diff x' \nonumber \\
	&= 1 - \left(1-\frac{x}{R}\right)^{n-1},\hspace{1.5em}\text{for }0\leq x \leq R, \label{eq:app_arms_proof_cdf}
\end{align}
and its inverse is given by
\begin{align}
	(F_{E^{(s)}|\sum E^{(s')}=R})^{-1}(y) = -R(1-y)^{\frac{1}{n-1}}+R,\hspace{1.5em}\text{for }0\leq y \leq 1.
	\label{eq:app_arms_proof_inv_cdf}
\end{align}
By inverse transform sampling, we can sample a random variable with CDF from \Cref{eq:app_arms_proof_cdf} by first sampling $U\sim U[0,1]$ and then applying \Cref{eq:app_arms_proof_inv_cdf} to it. In that case, $1-U$ is also distributed according to $U[0,1]$, so we can use it instead. Therefore, sampling $E^{(s)}$ given $\sum_{s'=1}^{n}E^{(s')}=R$ is the same as sampling $U \sim U[0,1]$, then doing:
\begin{equation}
	E^{(s)} = -R U^{\frac{1}{n-1}} + R.
	\label{eq:arms_iter_app_cond_sample}
\end{equation}
The iterative procedure outlined in \Cref{alg:iter_dir} then consists of the following steps:
\begin{enumerate}
	\item Sample $\sum_{s'=1}^{n}E^{(s')}$ (\Crefalt{eq:arms_iter_app_eq_gamma}).
	\item Sample $E^{(s)}$ given this sum (\Crefalt{eq:arms_iter_app_cond_sample}).
	\item Compute the Dirichlet $d^{(s)}$ (\Crefalt{eq:app_arms_dirichlet}).
	\item Compute uniform $\tilde{u}$ using the marginal Dirichlet CDF (similar to \Cref{alg:orig_arms}).
	\item Repeat this process for the next iteration, but, instead of sampling $\sum_{s'=1}^{n}E^{(s')}$ again, use $R$ minus the previous $E^{(s)}$, with the caveat that \Cref{eq:app_arms_proof_inv_cdf} has to be adapted to consider one less sample.
\end{enumerate}
The final algorithm, resulting from combining \Cref{alg:iter_loorf,alg:orig_arms,alg:iter_dir} is outlined in \Cref{alg:iter_arms}.
\begin{algorithm}
	\caption{Iterative ARMS with Dirichlet copulas}
	\label{alg:iter_arms}
	\begin{algorithmic}
		\State $\triangleright$ Start accumulators:
		\State $\Jh \leftarrow 0$;
		\State $\Lambda_{\nabla \log} \leftarrow \zeros$;
		\State $\Lambda_{J \nabla \log} \leftarrow \zeros$;
		\State $\triangleright$ Start variables for iterative Dirichlet sampling:
		\State Sample $\sum_{s'}E_i^{(s')} \sim \mathrm{Gamma}[n,1]$ for $i \in \{1, \cdots, d\}$
		\State $R_i \leftarrow \sum_{s'}E_i^{(s')}$
		\State $\triangleright$ Compute correlations:
		\For{$i \in 1 \dots d$}
		\If{$\theta_i > 0.5$}
		\State  $\rho_i \leftarrow \frac{\max(0, 2(1 - \theta_i)^{\frac{1}{n-1}} - 1)^{n - 1} - (1 - \theta_i)^2}{\theta_i(1 - \theta_i)}$
		\Else
		\State $\rho_i \leftarrow \frac{\max(0, 2\theta_i^{1 / (n-1)} - 1)^{n - 1} - \theta_i^2 }{\theta_i(1 - \theta_i)}$
		\EndIf		
		\EndFor
		\State $\triangleright$ Main loop:
		\For{$s \in 1 \dots n$}
		\State $\triangleright$ The loop below can be parallelized with vectorized implementations
		\For{$i \in 1 \dots d$}
		\If{$s < n$}
		\State Sample $U \sim U[0,1]$
		\State $E_i^{(s)} \leftarrow -R_i U^{\frac{1}{n-s}} + R_i$
		\Else
		\State $E_i^{(s)} \leftarrow R_i$
		\EndIf
		\State $R_i \leftarrow R_i - E_i^{(s)}$
		\State $\triangleright$ Get single Dirichlet r.v. (recall that $\sum_{s'}E_i^{(s')}$ was already computed)
		\State $\dis \leftarrow \frac{E_i^{(s)}}{\sum_{s'}E_i^{(s')}}$
		\State $\uTils_i \leftarrow 1 - (1-\dis)^{n-1}$ \Comment{Apply marginal CDF}
		\If{$\theta_i > 0.5$}  \Comment{Additional steps from ARMS paper}
		\State $\zTils_i \leftarrow \ind[\uTils_i \leq \theta_i]$
		\Else
		\State $\zTils_i \leftarrow \ind[1 - \uTils_i \leq \theta_i]$
		\EndIf
		\EndFor
		\State $\triangleright$ Accumulate values, similarly to LOORF
		\State $\Jh \leftarrow \frac{s-1}{s}\Jh + \frac{1}{s} J(\bzTils) $
		\State $\Lambda_{\nabla \log} \leftarrow \frac{\mymax{}(s-2, 1)}{\mymax{}(s-1, 1)}\Lambda_{\nabla \log} + \frac{1}{\mymax{}(s-1, 1)} \nabla_{\btheta} \log p_{\btheta}(\bzTils)$
		\State $\Lambda_{J \nabla \log} \leftarrow \frac{\mymax{}(s-2, 1)}{\mymax{}(s-1, 1)}\Lambda_{J \nabla \log} + \frac{1}{\mymax{}(s-1, 1)} J(\bzTils) \nabla_{\btheta} \log p_{\btheta}(\bzTils)$			
		\EndFor
		\State $\bgh \leftarrow \frac{1}{1-\brho}\Big(\Lambda_{J \nabla \log} - \Lambda_{\nabla \log} \Jh\Big)$
		\State \textbf{Return:} $\bgh$
	\end{algorithmic}
\end{algorithm}

\FloatBarrier
\subsection{Beta$^*$}
\label{sec:bstar_proofs}

\ifdefined\bStarOne{
	\bStarOne*
}
\fi

\begin{proof}
	For arbitrary $f(\cdot)$ and $d=1$, consider estimating $\mathbb{E}[f(z)]$ by using control variates as follows:
	\begin{equation}
		\E{}{f(z)} \approx \left( \frac{1}{n} \sum_{s=1}^n f(z^{(s)}) - \beta h(z^{(s)}) \right) + \beta \E{}{h(z)}.
		\label{eq:app_bstar_proof_owen_form}
	\end{equation}
	From \textcite[Section 8.9]{owen2013monte}, the optimal $\beta$ is given by
	\begin{equation*}
		\beta^* = \frac{\Cov{f(z), h(z)}}{\Var{h(z)}}.
	\end{equation*}
	Estimation from \Cref{eq:bstar_thm_estimator} is equivalent to \Cref{eq:app_bstar_proof_owen_form} for the following choices:
	\begin{equation*}
		f(z) = \frac{\partial \log p_{\btheta}(\bz)}{\partial \theta_i} J(\bz) \hspace{1em}\text{ and }\hspace{1em} h(z) = \frac{\partial \log p_{\btheta}(\bz)}{\partial \theta_i}.
	\end{equation*}
	We have:
	\begin{align*}
		\beta^* &= \frac{\Cov{\frac{\partial \log p_{\btheta}(\bz)}{\partial \theta_i} J(\bz), \frac{\partial \log p_{\btheta}(\bz)}{\partial \theta_i}}}{\Var{\frac{\partial \log p_{\btheta}(\bz)}{\partial \theta_i}}} \\
		&= \frac{\E{}{\left(\frac{\partial \log p_{\btheta}(\bz)}{\partial \theta_i}\right)^2 J(\bz)}}{\E{}{\left(\frac{\partial \log p_{\btheta}(\bz)}{\partial \theta_i}\right)^2}} \tag*{$\triangleright$ \parbox[t]{5.0cm}{\raggedright Since $\E{\bz \sim p_{\btheta}(\cdot)}{\frac{\partial \log p_{\btheta}(\bz)}{\partial \theta_i}} = 0$}} \\
		&= \frac{\int p_{\btheta}(\bz) \left(\frac{\partial \log p_{\btheta}(\bz)}{\partial \theta_i}\right)^2 J(\bz) \diff \bz}{\int p_{\btheta}(\bz') \left(\frac{\partial \log p_{\btheta}(\bz')}{\partial \theta_i}\right)^2  \diff \bz'} \\
		&= \int \left(\frac{p_{\btheta}(\bz) \left(\frac{\partial \log p_{\btheta}(\bz)}{\partial \theta_i}\right)^2}{\int p_{\btheta}(\bz') \left(\frac{\partial \log p_{\btheta}(\bz')}{\partial \theta_i}\right)^2  \diff \bz'}\right) J(\bz) \diff \bz \\
		&= \E{\bz \sim q_i(\cdot; \btheta)}{J(\bz)}.
	\end{align*}
	Where:
	\begin{align*}
		q_i(\bz; \btheta) &= \frac{p_{\btheta}(\bz) \left( \frac{\partial \log p_{\btheta}(\bz)}{\partial \theta_i} \right)^2}{\int p_{\btheta}(\bz') \left(\frac{\partial \log p_{\btheta}(\bz')}{\partial \theta_i}\right)^2  \diff \bz'} \\
		&\propto p_{\btheta}(\bz) \left( \frac{\partial \log p_{\btheta}(\bz)}{\partial \theta_i} \right)^2.
	\end{align*}
\end{proof}

\ifdefined\bStarTwo{
	\bStarTwo*
}
\fi

\begin{proof}
	Since we have $p_{\btheta}(\bz) = \prod_{i=1}^d p_{\theta_i}(z_i)$, $q_i(\bz; \btheta)$ becomes:
	\begin{align*}
		q_i(\bz; \btheta) &= \frac{p_{\btheta}(\bz) \left( \frac{\partial \log p_{\btheta}(\bz)}{\partial \theta_i} \right)^2}{\int p_{\btheta}(\bz') \left(\frac{\partial \log p_{\btheta}(\bz')}{\partial \theta_i}\right)^2  \diff \bz'} \\
		&= \frac{p_{\btheta}(\bz) \left( \frac{\partial \log p_{\btheta}(z_i)}{\partial \theta_i} \right)^2}{\int p_{\btheta}(z_i') \left(\frac{\partial \log p_{\btheta}(z_i')}{\partial \theta_i}\right)^2  \diff z_i'}.
	\end{align*}
	Which, for $p_{\btheta}(\cdot) = \prod_{i=1}^{d} \Ber[\theta_i]$, becomes:
	\begin{align*}
		q_i(\bz; \btheta) &= \prod_{j\neq i} p_{\theta_j}(z_j) \frac{\frac{1}{p_{\theta_i}(z_i)}}{\frac{1}{p_{\theta_i}(z_i)}+\frac{1}{1-p_{\theta_i}(z_i)}} \\
		&= \left(\prod_{j\neq i} p_{\theta_j}(z_j)\right)(1 - p_{\theta_i}(z_i)).
	\end{align*}
	Plugging this in $\mathbb{E}_{\bz \sim q_i(\cdot; \btheta)}[J(\bz)]$ and using $\bz_{\setminus i} \sim p_{\btheta}(\cdot)$ as shorthand for $z_j \sim p_{\theta_j}(\cdot)$ for $j \neq i$:
	\begin{align*}
		\beta_i^* &= \E{\underset{z_i \sim 1 - p_{\theta_i}(\cdot)}{\bz_{\setminus i} \sim p_{\btheta}(\cdot)}}{J(\bz)} \\
		&= \E{\bz \sim p_{\btheta}(\cdot)}{J(z_1, \dots, z_{i-1}, 1 - z_{i}, z_{i+1}, \dots)},
	\end{align*}
	where the last step follows from the fact that, in general, we can sample $b' \sim \Ber[1-\theta]$ as $b' = 1-b$, for $b \sim \Ber[\theta]$.
\end{proof}

\subsection{Generalization of Stochastic Formulation}
\label{sec:mc_bad_gen_2}

\ifdefined\mcBadGen{
	\mcBadGen*
}
\fi
\begin{proof}
	As mentioned in the main text, $J(\cdot)$ is bounded to the interval $[m, M_0)$, regardless of $d$ . Since $p_{\theta_i}(z_i)$ is a factorized Bernoulli, we have:
	\begin{equation*}
		p_{\theta_i}(z_i) = \begin{cases}
			\theta_i &\text{if }z_i = 1 \\
			1-\theta_i &\text{if }z_i = 0
		\end{cases} \hspace{1.5em}\text{and}\hspace{1.5em}p_{\btheta}(\bz) = \prod_{i=1}^{d}p_{\theta_i}(z_i).
	\end{equation*}
	We also have:
	\begin{equation*}
		\frac{\partial p_{\btheta}(\bz)}{\partial \theta_i} = (2 z_i - 1)\prod_{j \neq i} p_{\theta_j}(z_j)
	\end{equation*}
	and the direction that increases $p_{\btheta}(\bz^*)$ at the $i$-th coordinate is:
	\begin{equation*}
		\frac{\partial p_{\btheta}(\bz^*)}{\partial \theta_i} = (2 z_i^* - 1)\prod_{j \neq i} p_{\theta_j}(z_j^*).
	\end{equation*}
	Whereas the negative of the true gradient at the $i$-th coordinate is:
	\begin{align*}
		-\frac{\partial}{\partial \theta_i} \E{\bz \sim p_{\btheta}(\cdot)}{J(\bz)} &= -\sum_{h=0}^{2^d-1}\frac{\partial p_{\btheta}(\Zeta_h)}{\partial \theta_i}J(\Zeta_h) \\
		&= -\sum_{h=0}^{2^d-1}\left((2 (\Zeta_h)_i - 1)\prod_{j \neq i} p_{\theta_j}((\Zeta_h)_j)\right)J(\Zeta_h).
	\end{align*}
	Therefore, the $i$-th coordinate of $-\nabla_{\btheta}\mathbb{E}_{\bz \sim p_{\btheta}(\cdot)}{J(\bz)}$ and that of $\nabla_{\btheta} p_{\btheta}(\bz^*)$ will point in the same direction if the following is $\geq 0$:
	\begin{align*}		
		\frac{\partial p_{\btheta}(\bz^*)}{\partial \theta_i} \left(-\frac{\partial}{\partial \theta_i} \E{\bz \sim p_{\btheta}(\cdot)}{J(\bz)}\right) = (2 z_i^* - 1)\prod_{j \neq i} p_{\theta_j}(z_j^*) \begin{aligned}[t]
			&\Bigg(-\sum_{h=0}^{2^d-1}\Bigg((2 (\Zeta_h)_i - 1)\\
			&\hspace{1.7em}\prod_{j \neq i} p_{\theta_j}((\Zeta_h)_j)\Bigg)J(\Zeta_h)\Bigg).
		\end{aligned}
	\end{align*}
	Since $\prod_{j \neq i} p_{\theta_j}(z_j^*)$ does not change the sign of the expression, we will ignore it. The remaining terms can be written as:
	\begingroup
	\allowdisplaybreaks
	\begin{align}
		& (2 z_i^* - 1)\Bigg(-\sum_{h=0}^{2^d-1}\left((2 (\Zeta_h)_i - 1) \prod_{j \neq i} p_{\theta_j}((\Zeta_h)_j)\right)J(\Zeta_h)\Bigg) \nonumber\\
		={}& (2 z_i^* - 1) \Big(-\sum_{z_i=0}^{1}(2 z_i - 1) \mathbb{E}_{\{z_j \sim p_{\theta_j}(\cdot)\}_{j \neq i}}[J(\bz)] \Big) \nonumber \\
		={}&  - \sum_{z_i=0}^{1} (-1)^{\ind[z_i \neq z_i^*]} \mathbb{E}_{\{z_j \sim p_{\theta_j}(\cdot)\}_{j \neq i}}[J(\bz)]  \nonumber \\
		={}& - \begin{aligned}[t]
			&\Big(\mathbb{E}_{\{z_j \sim p_{\theta_j}(\cdot)\}_{j \neq i}}[J(z_1, \dots, z_{i-1}, z_{i}^*, z_{i+1}, \dots)]\ - \\
			&\hspace{3em}\mathbb{E}_{\{z_j \sim p_{\theta_j}(\cdot)\}_{j \neq i}}[J(z_1, \dots, z_{i-1}, 1 - z_{i}^*, z_{i+1}, \dots)]\Big)
		\end{aligned} \nonumber \\
		={}& - \begin{aligned}[t]
			&\mathbb{E}_{\{z_j \sim p_{\theta_j}(\cdot)\}_{j \neq i}}\big[J(z_1, \dots, z_{i-1}, z_{i}^*, z_{i+1}, \dots)\ - \\
			&\hspace{8em}J(z_1, \dots, z_{i-1}, 1 - z_{i}^*, z_{i+1}, \dots)\big].
		\end{aligned}
		\label{eq:app_true_grad_bound_alignment}
	\end{align}
	Note that the following implication is true for the given $J(\cdot)$:
	\begin{align}
		&(d_H(\bz, \bz') = 1) \land (d_H(\bz, \bz^*) - d_H(\bz', \bz^*) = -1) \Rightarrow \nonumber \\
		&\hspace{7.5em} \begin{aligned}
			\left((\bz = \bz^*) \land \left(J(\bz) - J(\bz') = m - \left(M_0 - \frac{\Delta M}{d}\right)\right) \right) \lor&\\
			\left((\bz \neq \bz^*) \land \left(J(\bz) - J(\bz')  = \frac{\Delta M}{d}\right)\right)&
		\end{aligned}
		\label{eq:app_true_bound_logic_statement}
	\end{align}
	and the left statement from \Cref{eq:app_true_bound_logic_statement} is satisfied for all pairs inside the expectation from the RHS of \Cref{eq:app_true_grad_bound_alignment}. Therefore \Cref{eq:app_true_grad_bound_alignment} becomes:
	\begin{align*}
		&-\begin{aligned}[t]
			&\mathbb{E}_{\{z_j \sim p_{\theta_j}(\cdot)\}_{j \neq i}}\big[J(z_1, \dots, z_{i-1}, z_{i}^*, z_{i+1}, \dots)\ - \\
			&\hspace{8em}J(z_1, \dots, z_{i-1}, 1 - z_{i}^*, z_{i+1}, \dots)\big]
		\end{aligned} \\
		={}& - \Bigg(\prod_{j \neq i} p_{\theta_j}(z_j^*)\Bigg(m - \Bigg(M_0 - \frac{\Delta M}{d}\Bigg)\Bigg) + \Bigg(1-\prod_{j \neq i} p_{\theta_j}(z_j^*)\Bigg)\frac{\Delta M}{d}\Bigg).
	\end{align*}
	Finally, after some algebraic manipulation:
	\begin{align*}
		& - \Bigg(\prod_{j \neq i} p_{\theta_j}(z_j^*)\Bigg(m - \Bigg(M_0 - \frac{\Delta M}{d}\Bigg)\Bigg) + \Bigg(1-\prod_{j \neq i} p_{\theta_j}(z_j^*)\Bigg)\frac{\Delta M}{d}\Bigg) \geq 0 \\
		\iff{}& m \leq M_0 - \frac{\Delta M}{d \prod_{j \neq i} p_{\theta_j}(z_j^*)}. \hspace{2em}\triangleright\text{Assuming }\prod_{j \neq i} p_{\theta_j}(z_j^*) \neq 0
	\end{align*}
	\endgroup	
\end{proof}

\section{Experimental Details}
This section details some choices made in experiments from the main text.
\subsection{Microworlds}
This section further details some experiments from \Cref{sec:microworld}.
\subsubsection{Variance Experiments}
\label{sec:app_true_var}

Here, we further explain the variance experiments performed in \Cref{sec:microworld_compare_estimators}. We write the following in terms of $\btheta$ instead of $\br$ to avoid clutter. First, we note that none of the estimators used here require $\nabla_{\bz} J(\bz)$, so we can simply compute all $J(\Zeta_h)$ values and store them in a table to speed up computations. The variance of the gradient estimators is given by:
\begin{align}
	\Var{\bgh((\bzs)_{s=1}^n; \btheta)} &=\E{}{\bgh((\bzs)_{s=1}^n; \btheta)^2} - \E{}{\bgh(\bzs)_{s=1}^n; \btheta)}^2 \nonumber \\
	&= \E{}{\bgh((\bzs)_{s=1}^n; \btheta)^2} - \left( \nabla_{\btheta} \E{\bz \sim p_{\btheta}(\cdot)}{J(\bz)}\right)^2.
	\label{eq:app_main_var_eq}
\end{align}

For these smaller scale experiments, computing the right summand can simply be done by using \Cref{eq:true_as_avg}, which requires $2^d$ evaluations. Here, this is feasible to compute, but not so much for the left summand, which is a sum of $2^{dn}$ terms. To reduce its complexity, we leverage the structure of the expressions involved. Notably, $\bgh(\cdot)$ is order invariant with respect to $\bzs$ in all the estimators we use. Therefore, we only consider the combinations of $n$ elements from $\{\Zeta_h\}_{h=0}^{2^d-1}$ instead of all possible $n$-tuples. Since each of them is an $n$-combination of $2^d$ elements with repetition, the number of different values $\bgh(\cdot)$ can take is given by:
\begin{equation}
	\multiset{2^d}{n} = {2^d - 1 + n \choose n}.
	\label{eq:n_multisets}
\end{equation}
Therefore, to compute the left summand of \Cref{eq:app_main_var_eq}, we need to evaluate $\bgh(\cdot)^2$ only in these combinations, then multiply each evaluation by the probability of the respective combination and sum all of these terms. The probability of a combination will be the sum of the probabilities of all the permutations corresponding to it.

Since, as mentioned before, these combinations may have repetition, each of them corresponds to a set that can have repeated elements, sometimes called multisets. The number of times each value appear in the multiset is called its multiplicity. Considering a multiset $\sM$ and denoting as $m_h$ the multiplicity of $\Zeta_h$, where $m_h \geq 0$ and $\sum_{h=0}^{2^d-1} m_h = n$, we have that the number of $n$-permutations of $\sM$ is given by:
\begin{equation}
	{n \choose m_0, \dots, m_{2^d-1}} = \frac{n!}{m_0!\dots m_{2^d-1}!}.
	\label{eq:multiset_permutations}
\end{equation}

For REINFORCE, LOORF, $\beta^*$, the $n$ samples are iid, therefore, the probability of an arbitrary set of $n$ samples is simply:
\begin{equation}
	p((\bzs)_{s=1}^n; \btheta) = \prod_{s=1}^n p_{\btheta}(\bzs).
	\label{eq:app_var_z_probs}
\end{equation}

Which will be the same for all permutations corresponding to each multiset. To conclude, the left summand from \Cref{eq:app_main_var_eq} can be computed by iterating the combinations from \Cref{eq:n_multisets} and, for each of them, computing the quantity from \Cref{eq:app_var_z_probs}, multiplying by the one in \Cref{eq:multiset_permutations} and by $\bgh(\cdot)^2$, then summing all of these results. For $d=4$ and $n=4$, this reduces the number of per-step evaluations from $65,536$ to $3,876$ when calculating $\mathbb{E}[\bgh(\cdot)^2]$.

For ARMS, however, sampling follows $\Cref{alg:iter_arms}$, which introduces dependence between different $\bzs$, causing \Cref{eq:app_var_z_probs} to be no longer valid. To show how we calculate the new probabilities, we use $d=1$, $n=2$ and $(\zTil^{(1)}, \zTil^{(2)})=(1,0)$ as an example. The algorithm has two cases: $\theta > 0.5$, where the Dirchlet copula $(\uTils)_{s=1}^n$ is used; $\theta \leq 0.5$, where $(1-\uTils)_{s=1}^n$ is used. For the first case, we have:
\begin{align}
	p(\zTil^{(1)} = 1, \zTil^{(2)} = 0; \theta) &= p(\uTil^{(1)} < \theta, \uTil^{(2)} > \theta) \nonumber \\
	&= p(d^{(1)} < F^{-1}(\theta), d^{(2)} > F^{-1}(\theta)),
	\label{eq:app_var_red_d_2}
\end{align}
where $F^{-1}(\cdot)$ is the inverse of the marginal CDF of the Dirichlet distribution used to get the copula in \Cref{alg:iter_arms}:
\begin{equation*}
	F^{-1}(\theta) = 1 - (1 - \theta)^{1/(n-1)}.
\end{equation*}
This function is monotonically increasing in $[0,1]$. By the law of total probability, \Cref{eq:app_var_red_d_2} corresponds to:
\begin{equation}
	\int p(d^{(1)}, d^{(2)}) \ind[(d^{(1)} < F^{-1}(\theta))] \ind[(d^{(2)} > F^{-1}(\theta))] \diff d^{(1)} \diff d^{(2)}.
	\label{eq:mathematica_integral}
\end{equation}
We can input this integral to a symbolic equation solver, such as Mathematica, by using the PDF of the Dirichlet and multivariate integration, yielding a closed-form expression in terms of $\theta$.

To generalize the above (still on $d=1$ for now), we note that these steps are also order invariant with respect to $\zTil^{(s)}$. Therefore, the number of possible integrals such as the one in \Cref{eq:mathematica_integral} is simply the number of $n$-combinations of $\{0,1\}$ with repetition, given by:
\begin{equation*}
	\multiset{2}{n} = {2 - 1 + n \choose n} = n + 1.
\end{equation*}
Once $\theta_t$ is known, we can compute these $n+1$ values in closed-form and store them in a table to then be repeatedly consulted for each possible set $\{\zTil^{(1)}, \dots, \zTil^{(n)}\}$ when calculating the expectation $\mathbb{E}[\gh_{ARMS}(\cdot)^2]$ (taking the place of \Crefalt{eq:app_var_z_probs}), since they will not change until $\theta_t$ changes. In the second case, where $\theta \leq 0.5$, ARMS uses a different copula, so the above steps can be slightly changed to:
\begin{align*}
	p(\zTil^{(1)} = 1, \zTil^{(2)} = 0; \theta) &= p(1 - \uTil^{(1)} < \theta, 1 - \uTil^{(2)} > \theta) \\
	&= p(\uTil^{(1)} > 1 - \theta, \uTil^{(2)} < 1- \theta) \\
	&= p(d^{(1)} > F^{-1}(1-\theta), d^{(2)} < F^{-1}(1-\theta)).
\end{align*}
$\theta$ is simply exchanged by $1 - \theta$ when computing the table and the inequalities change sides when consulting it. In other words, it is the same result as running the previous steps, but using $1-\theta$ instead of $\theta$ and $(\zTil^{(1)}, \zTil^{(2)})=(0,1)$ instead of $(1,0)$.

For $d > 1$, independence permits per-dimension parallelization of these steps. The closed-form integral solutions will be the same across different $i \in \{1, \dots, d\}$, requiring only changing $\theta$ for $\theta_i$. To allow reproducibility, \Cref{tab:integral_closed} gives closed-form expressions for integrals like the one in \Cref{eq:mathematica_integral} for the case where $n=4$ and $d=1$, where  we denote $\oplus$ to be the exclusive disjunction (\ie, XOR).

For the other experiments, where $n=10$ and $d=10$, we use estimated variance rather than closed-forms. This simply amounts to changing the expectation in the left summand from \Cref{eq:app_main_var_eq} for a Monte Carlo estimate, where we computed the mean over $10,000$ evaluations of $\bgh(\cdot)^2$ for each timestep, the right summand is still used in closed-form.
\begin{table}
	\centering
	\begin{tabular}{ | c | l| } 
		\hline
		Multiset &  $\int p((d^{(s)})_{s=1}^n) \prod_{s=1}^n \ind[(d^{(s)} < F^{-1}(\theta)) \oplus (\zTil^{(s)} = 0)] \diff d^{(s)} $ \\
		\hline
		$\{0,0,0,0\}$ & 
		$
		f(\cdot; \theta) = \begin{cases}
			-(4 \theta -1)^3 & 0 \leq \theta < \frac{1}{4} \\
			0 & \text{otherwise}
		\end{cases}
		$ \\
		$\{0,0,0,1\}$ & 
		$
		f(\cdot; \theta) = \begin{cases}
			\theta  \left(37 \theta ^2-21 \theta +3\right) & 0 \leq \theta <\frac{1}{4} \\
			-(3 \theta -1)^3 & \frac{1}{4}\leq \theta <\frac{1}{3} \\
			0 & \text{otherwise}
		\end{cases}
		$ \\
		$\{0,0,1,1\}$ & 
		$f(\cdot; \theta) =  \begin{cases}
			6 (1-3 \theta ) \theta ^2 & 0\leq \theta <\frac{1}{4} \\
			46 \theta ^3-42 \theta ^2+12 \theta -1 & \frac{1}{4}\leq \theta <\frac{1}{3} \\
			-(2 \theta -1)^3 & \frac{1}{3}\leq \theta <\frac{1}{2} \\
			0 & \text{otherwise}
		\end{cases}$
		\\
		$\{0,1,1,1\}$ &
		$f(\cdot; \theta) = \begin{cases}
			6 \theta ^3 & 0\leq \theta <\frac{1}{4} \\
			-58 \theta ^3+48 \theta ^2-12 \theta +1 & \frac{1}{4}\leq \theta <\frac{1}{3} \\
			23 \theta ^3-33 \theta ^2+15 \theta -2 & \frac{1}{3}\leq \theta <\frac{1}{2} \\
			-(\theta -1)^3 & \frac{1}{2}\leq \theta \leq 1
		\end{cases}$
		\\
		$\{1,1,1,1\}$ & 
		$f(\cdot; \theta) = \begin{cases}
			0 & 0 \leq \theta < \frac{1}{4} \\
			(4 \theta -1)^3 & \frac{1}{4}\leq \theta <\frac{1}{3} \\
			-44 \theta ^3+60 \theta ^2-24 \theta +3 & \frac{1}{3}\leq \theta <\frac{1}{2} \\
			4 \theta ^3-12 \theta ^2+12 \theta -3 & \frac{1}{2} \leq \theta \leq 1 \\
		\end{cases}$
		\\
		\hline
	\end{tabular}
	\caption{Closed-form expressions for probabilities of $\{\zTil^{(s)}\}_{s=1}^n$, for $n=1$ and $d=4$. If $\theta > 0.5$, $p(\{\zTil^{(s)}\}_{s=1}^n; \theta) = f(\{\zTil^{(s)}\}_{s=1}^n; \theta)$, otherwise $p(\{\zTil^{(s)}\}_{s=1}^n; \theta) = f(\{1 - \zTil^{(s)}\}_{s=1}^n; 1 - \theta)$.}
	\label{tab:integral_closed}
\end{table}

\FloatBarrier

\subsubsection{Comparing Methods}
\label{sec:app_microworld_hypers}
Whenever possible, we follow the recommendations from the authors, either directly from the papers or at least from the provided code. For REBAR, we initialized $\log \tau$ to $0.5$ and $\eta$ to $1$, whereas for RELAX we also initialized $\log \tau$ to $0.5$, but the method does not use $\eta$. We train these parameters using Adam and the same learning rate as $\br$. Additionally, the RELAX auxiliary neural network had $1$ hidden layer of size $10$ and used tanh activations. We trained it using Adam, with learning rate $1$ and weight decay $0.001$. For CP, the temperature starts at one and follows an exponential schedule, being updated every $100$ iterations until arriving at the final value of $1/200$. We naturally show the loss for $\ind[\cdot]$ instead of $\sigma(\cdot/\tau)$ when reporting its results. Moreover, we ran CP without learning rate decay, as it hindered performance.

\subsection{MaskedNNRegression}
\label{sec:app_maskednnregression}

The backbone network contained four fully connected hidden layers of size $50$ with the linear operator followed by ``batch norm'' and then LeakyReLU. We did not normalize outputs from the last layer. Weights used Xavier normal initialization.

The target network, on the other hand, had a more complex design, consisting of five fully connected hidden layers of size $500$ with the same sequence of per-layer operations as above. Here, we normalized outputs from the final layer to $[0,1]$ by using the corresponding maximum and minimum from the training data set. We initialized weights to either $-1$ or $1$ with $50\%$ chance.

To avoid adding trainable variables, we used batch normalization without affine parameters. We also excluded moving statistics from the implementation. Although \textcite{ioffe2015batch} mention that this form can reduce the expressiveness of the unnormalized layer, we observed major improvements when including it. The trained models had much lower error for the same targets and the target NN could generate much more complex mappings, which we confirmed by experimenting with one dimensional inputs and plotting the target maps. When constructing the data sets via forward-passes, we inputted each of them in its entirety, producing good average statistics for the normalization. Although we generated them separately, our comparisons to joint generation indicated this did not impact the results.

The training data set consisted of $10,000$ samples and we used batch size $100$, whereas the validation data set consisted of $5,000$ samples, which were input at once when validating. Reported results correspond to validation data, as we saw no need to generate more data for testing. To avoid relying on a single random mask when validating MC methods, we sample a batch of $500$ data points from the training set and compare the average loss over five random masks, selecting the best one.

\subsection{Pruning}
\label{sec:app_pruning_setup}
Similarly to \textcite{zhou2019deconstructing} we apply dynamic weight rescaling to MC (but not CP, as it overall hindered performance), with division of the weights by the mean of corresponding layer masks during forward passes. This quotient is treated as a constant in the backward pass. \textcite{hinton2012improving} also used DWR in the dropout paper. We noticed that methods tend to prune too aggressively without this addition. Although not using it also leads to reasonable results, sensitivity to $\lambda$ gets higher.

When training stochastic masks with SGD, \textcite{zhou2019deconstructing} had to resort to unusual learning rates such as $20$ or $100$. We did not observe this issue with RMSprop. Our implementation of CP is very similar to the one from \textcite{savarese2020winning}, the main difference being the initialization and the $\lambda$ sweeps: their mask initialization was lower and they set $\lambda$ close to zero, whereas we used $\btheta_0 = [0.5, \dots, 0.5]^\top$ and larger $\lambda$, which we believe to be more consistent with the PB optimization discussion.

In all experiments, we divided data sets in training, validation and test, with the test sets following the default split from CIFAR-10 and MNIST, while the validation sets consisted of $5,000$ random samples. Similarly to MaskedNNRegression, to avoid relying on a single random mask, we randomly sample 10 batches from the training set and 10 masks, selecting the one with the best performance on these examples to be fixed for testing and validation.

\subsubsection{Supermask}
\label{sec:app_supermask_setup}
\begin{table}[htb!]
	\centering
	\begin{tabular}{ | c | l | l |} 
		\hline
		Method & Parameter & Values \\
		\hline
		Shared & Batch size & \{128\} \\
		& Epochs & \{200\} \\	       
		& $\btheta_0 = \theta(\br_0)$ & $\{[0.5, \dots, 0.5]^\top\}$ \\	
		& Random seed & $\{0, 1, 2, 3, 4\}$ \\	       
		\hline
		MC & Optimizer (\br) & \{RMSprop\} \\	
		& Parametrization & \{Escort, Sigmoid\} \\
		& Estimator & \{ARMS, LOORF\} \\
		& Learning rate ($\br$) & \{0.1, 0.01, 0.001\} \\
		& Learning rate schedule ($\br$) & \{[60\%]\} \\
		& (\% of training) &  \\
		& Learning rate schedule ($\br$) & \{[0.5]\} \\
		& (multipliers) &  \\		   
		& $L0$-regularization ($\lambda$) & $\{1\scE-1, 5\scE-2,$ \\
		& & \hspace*{\fill}$\dots, 1\scE-5\}$ \\
		& $n$ & $\{2, 10, 100\}$ \\
		\hline
		CP & Optimizer (\br) & \{RMSprop, SGD\} \\	
		& Learning rate ($\br$) & \{0.1, 0.01, 0.001\} \\
		& Learning rate schedule ($\br$) & \{[40\%, 60\%]\} \\
		& (\% of training) &  \\
		& Learning rate schedule ($\br$) & \{[0.1, 0.1]\} \\
		& (multipliers) &  \\		   
		& $L0$-regularization ($\lambda$) & $\{1\scE-1, 5\scE-2,$ \\
		& & \hspace*{\fill}$\dots, 1\scE-5\}$ \\
		\hline 	                    		                    		                    
	\end{tabular}
	\caption{Supermask hyperparameter sweep.}
	\label{tab:app_supermask_hypers}
\end{table}
\Cref{tab:app_supermask_hypers} summarizes the hyperparameter sweep used. As mentioned before, sigmoid and LOORF had similar, but slightly worse results than escort and ARMS, so we only include these last two in the results. Remaining details for architectures and data sets are as described in \textcite[section S1]{zhou2019deconstructing}, where they refer to Lenet as MNIST-FC.

\FloatBarrier
\subsubsection{Joint Pruning}
\label{sec:app_joint_pruning_setup}
\begin{table}[htb!]
	\centering
	\begin{tabular}{ | c | l | l |} 
		\hline
		Method & Parameter & Values \\
		\hline
		Shared & Batch size & \{128\} for Resnet-20 \\
		 &  & \{64\} for VGG \\
		& Epochs & \{200\} \\
		& Finetune only & $\{80\%\}$ \\
		& (\% of training) &  \\
		& Random seed & $\{0, 1, 2, 3, 4\}$ \\		
		\hline
		MC & Parametrization & \{Escort\} \\
		& Estimator & \{ARMS\} \\
		& Optimizer (\br) & \{RMSprop\} \\		       
		& $\btheta_0 = \theta(\br_0)$ & $\{[0.5, \dots, 0.5]^\top\}$ \\
		& Learning rate ($\br$) & \{0.1, 0.01, 0.001\} \\
		& Learning rate schedule ($\br$) & \{[60\%]\} \\
		& (\% of training) &  \\
		& Learning rate schedule ($\br$) & \{[0.5]\} \\
		& (multipliers) &  \\		   
		& $L0$-regularization ($\lambda$) & $\{1\scE-1, 5\scE-2,$ \\
		& & \hspace*{\fill}$\dots, 1\scE-5\}$ \\
		& $n$ & $\{2, 10\}$ \\
		& Start training $\br$ & \{[10\%]\} \\
		& (\% of training) &  \\				
		\hline
		MC ($n=100$) & Learning rate ($\br$) & \{0.01\} \\
		& $n$ & $\{100\}$ \\
		&(Rest is the same as MC) & \\
		\hline		
		CP 	& Optimizer (\br) & \{RMSprop, SGD\} \\		       
		& $\btheta_0 = \theta(\br_0)$ & $\{[0.5, \dots, 0.5]^\top\}$ \\			 
		& Learning rate ($\br$) & \{0.1, 0.01, 0.001\} \\
		& Learning rate schedule ($\br$) & \{[40\%, 60\%]\} \\
		& (\% of training) &  \\
		& Learning rate schedule ($\br$) & \{[0.1, 0.1]\} \\
		& (multipliers) &  \\		   
		& $L0$-regularization ($\lambda$) & $\{1\scE-1, 5\scE-2,$ \\
		& & \hspace*{\fill}$\dots, 1\scE-8\}$ \\ 
		& & $\cup\ \{0.2, 0.3, \dots, 1\}$ \\ \hline
		GMP & Final weights remaining & $\{50\%,\ 10\%, \dots$\\
		& & \hspace*{\fill}$\dots,\ 0.1\%,\ 0.05\%\}$\\
		&(Rest is the same as & \\
		& \textcite{zhu2017prune}) & \\	
		\hline
		MP & Final weights remaining & $\{50\%,\ 10\%, \dots$\\
		& & \hspace*{\fill}$\dots,\ 0.1\%,\ 0.05\%\}$\\
		& Global prune & \{True\}\\
		\hline				
	\end{tabular}
	\caption{Joint pruning broad sweep.}
	\label{tab:app_pruning_broad}
\end{table}
%
%\begin{table}[htb!]
%	\centering
%	\begin{tabular}{ | c | l | l | l |} 
%		\hline
%		Method & Parameter & Resnet-20 & VGG \\
%		\hline
%		MC ($n=100$) & $L0$-regularization ($\lambda$) & $\{0.075, 0.045, 0.04,$ \\
%		& & $0.035, 0.03, 0.025,$ \\ 
%		& & $0.02, 0.015, 0.0035,$ \\
%		& & $0.0025\}$ \\
%		\hline
%		CP & Learning rate ($\br$) & \{0.01\} \\
%		& $L0$-regularization ($\lambda$) & $\{0.7, 0.6, 0.5, 0.4$ \\
%		& & $0.3, 0.2, 0.045, 0.035$ \\ 
%		& & $0.025, 0.015\}$ \\
%		\hline
%		CP & Learning rate ($\br$) & \{0.1\} \\
%		& $L0$-regularization ($\lambda$) & $\{1.0, 0.9, 0.8, 0.7, 0.6$ \\
%		& & $0.5, 0.4, 0.3 ,0.2\}$ \\ 
%		\hline
%		GMP and MP & Final weights remaining & $\{32\%,21\%,16\%,7.5\%$\\
%		& & $3\%,2.6\%,2\%,1.4\%,$\\
%		& & $1.1\%\}$\\
%		\hline
%		GMP only & Final weights remaining & $\{0.45\%,0.35\%, 0.3\%$\\
%		& & $0.25\%,\ 0.15\%\}\cup$\\
%		& & $\{0.04\%,0.03\%,0.02\%$\\
%		& & $0.01\%\}$\\
%		\hline
%	\end{tabular}
%	\caption{Joint pruning ``specialized'' sweep. Other parameters are the same as \Cref{tab:app_supermask_hypers}.}
%	\label{tab:app_pruning_specialized}
%\end{table}

\begin{table}[htb!]
	\centering
	\begin{tabular}{ | c | l | p{0.2\textwidth} | p{0.2\textwidth} |} 
		\hline
		Method & Parameter & Resnet-20 & VGG \\
		\hline
		MC & n & \{100\} & \{2, 10, 100\} \\
		   & Learning rate ($\br$) & \{0.01\}  & \{0.01\} \\
		   & $L0$-regularization ($\lambda$) &  \{0.075, 0.045, 0.04, 0.035, 0.03, 0.025, 0.02, 0.015, 0.0035, 0.0025, $1\scE{-6}$\} & \{$5\scE{-6}$, $1\scE{-6}$, $\dots$, $1\scE{-8}$\} \\ \hline
		CP & Optimizer (\br) & N/A & \{SGD\} \\	
		& Learning rate ($\br$) & N/A  & \{0.1\} \\
		& $L0$-regularization ($\lambda$) & N/A  & \{$5\scE{-9}$, $1\scE{-9}$, $\dots$, $1\scE{-10}$\} \\
		\hline
		GMP and MP & Final weights remaining & \{32\%, 21\%, 16\%, 7.5\%, 3\%, 2.6\%, 2\%, 1.4\%, 1.1\%\}  & Same, but no MP \\
		\hline
		GMP only & Final weights remaining & \{0.45\%, 0.35\%, 0.3\%, 0.25\%,\ 0.15\%\} $\cup$ \{0.04\%, 0.03\%, 0.02\%, 0.01\%\} & Same \\
		\hline
	\end{tabular}
	\caption{Joint pruning ``specialized'' sweep. Other parameters are the same as \Cref{tab:app_supermask_hypers}.}
	\label{tab:app_pruning_specialized}
\end{table}

Differently from the supermask experiments, in joint pruning we use image augmentation on the training set in the form of random horizontal flips and random crops. Our backbone experimental settings mostly follow the Resnet-20 and VGG descriptions from \textcite[figure 2]{frankle2018lottery} (they call the former Resnet-18), except for the learning rate schedule and the total number of epochs, which are based on \textcite{savarese2020winning}.

Inspired by \textcite{hoefler2021sparsity}, MC trains only the dense network in the first few epochs, the joint-training starts later. In the original GMP paper \parencite{zhu2017prune}, they also train this same way. Since CP uses smooth masks, joint training is already easier on the first epochs and there is no need for freezing masks. Similarly to \textcite{savarese2020winning} masks are frozen in the last epochs and only the main weights are fine-tuned.

As mentioned in the text, tentative hyperparameter values for $n=100$ were based on results for $n=10$. Our sweep was performed in two stages. First we used a broader range of hyperparameters, which we show in \Cref{tab:app_pruning_broad}. Then, after analyzing its results, we ran a second ``specialized'' sweep (\Cref{tab:app_pruning_specialized}) to better cover some sparsity ranges. On the main paper, our plots correspond to both sweeps combined.

\FloatBarrier
\clearpage
\begingroup
\def\conclusionParSkip{5pt}
\def\conclusionWidthOne{2.7cm}
\def\conclusionWidthTwo{2.75cm}
\def\conclusionWidthThree{8cm}
\newcommand{\conclusionParBox}[3][t]{\parbox[#1]{#2}{\vspace{.4em}\setlength{\parskip}{\conclusionParSkip}#3\vspace{.8em}}}
\subsection{Summary of Results}
For convenience, \Cref{tab:app_results_summ} summarizes the results from the main text.

\begin{longtblr}[caption = {Summary of all experiments performed}, label={tab:app_results_summ}]
	{|Q[m,c,\conclusionWidthOne]|Q[m,c,\conclusionWidthTwo]|Q[m,c,\conclusionWidthThree]|}
	\hline
	Experiment &  Comparison & Summary \\ \hline
	\SetCell[r=1]{m}{\centering Microworlds, \\ Variance \\ $(d \leq 10)$}
	& Estimators &
	\conclusionParBox[c]{\conclusionWidthThree}{
		$\bullet\ $REINFORCE sometimes had lower variance than ARMS/LOORF \par
		$\bullet\ $ARMS always had lower variance than LOORF \par
		$\bullet\ $ $\beta^*$ had significantly lower variance than the others
	}
	\\ \hline
	\SetCell[r=3]{m}{\centering Microworlds \\ $(d = 10)$}
	& Estimators &
	\conclusionParBox[c]{\conclusionWidthThree}{
		$\bullet\ $ REINFORCE performed the worst \par
		$\bullet\ $ LOORF/ARMS/$\beta^*$ performed similarly \par
		$\bullet\ $ True gradient failed to reach the correct solution in NNLoss
	}
	\\ \cline{2-3}
	& Parametrization &
	\conclusionParBox[c]{\conclusionWidthThree}{
		$\bullet\ $ Escort and direct were faster, but converged to worse final solutions \par
		$\bullet\ $ Sigmoid and cosine were slower, but converged to better final solutions \par
	}
	\\ \cline{2-3}
	& Approaches &
	\conclusionParBox[c]{\conclusionWidthThree}{
		$\bullet\ $ CP and ST performed the worst \par
		$\bullet\ $ $\nabla_{\bz}J(\bz)$ did not seem to help hybrid methods 
	}
	\\ \hline
	\SetCell[r=3]{m}{\centering NN regression \\ $(d \approx 8,000)$}
	& Estimators &
	\conclusionParBox[c]{\conclusionWidthThree}{
		$\bullet\ $ REINFORCE performed much worse than the others
	}
	\\ \cline{2-3}
	& Parametrization &
	\conclusionParBox[c]{\conclusionWidthThree}{
		$\bullet\ $ Cosine was still slower, but converged to worse values \par
		$\bullet\ $ Direct was still fast and still converged to poor values \par
		$\bullet\ $ Escort and sigmoid performed relatively well
	}
	\\ \cline{2-3}
	& Approaches &
	\conclusionParBox[c]{\conclusionWidthThree}{
		$\bullet\ $ $\nabla_{\bz}J(\bz)$ started becoming helpful for hybrid methods \par
		$\bullet\ $ CP performed the best, followed by ST and LOORF \par
		$\bullet\ $ CP performed well with SGD, but not with RMSprop
	}
	\\  \hline
	\SetCell[r=3]{m}{\centering Supermask $(d \approx 300,000$ and $2,000,000)$}
	& Estimators &
	\conclusionParBox[c]{\conclusionWidthThree}{
		$\bullet\ $ ARMS performed marginally better than LOORF
	}
	\\  \cline{2-3}
	& Parametrization &
	\conclusionParBox[c]{\conclusionWidthThree}{
		$\bullet\ $ Escort performed marginally better than sigmoid
	}
	\\ \cline{2-3}
	& Approaches &
	\conclusionParBox[c]{\conclusionWidthThree}{
		$\bullet\ $ CP performed better than MC, even when $n=100$ \par
		$\bullet\ $ CP performed much better with RMSprop than with SGD
	}
	\\  \hline		
	\SetCell[r=1]{m}{\centering Pruning \\ $(d \approx 300,000$ and $20,000,000)$}
	& Approaches &
	\conclusionParBox[c]{\conclusionWidthThree}{
		$\bullet\ $ CP performed the best \par
		$\bullet\ $ Per-method solutions were very different qualitatively \par
		$\bullet\ $ MC methods could not change the per-layer allocation as much as CP or GMP \par
		$\bullet\ $ CP performed much better with SGD than with RMSprop
	}
	\\  \hline		
\end{longtblr}

\endgroup

\FloatBarrier

\section{Additional Experiments}
\label{sec:add_exp}
This section provides results from extra experiments.
\subsection{Hyperparameter Generalization}
\label{sec:app_supermask_hypers}
\begin{figure}[htb!]
	\centering
	\includegraphics[height=1.5em]{\main/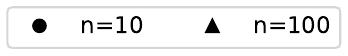}
	\par\vspace{1em}	
	\includegraphics[width=0.4\textwidth]{\main/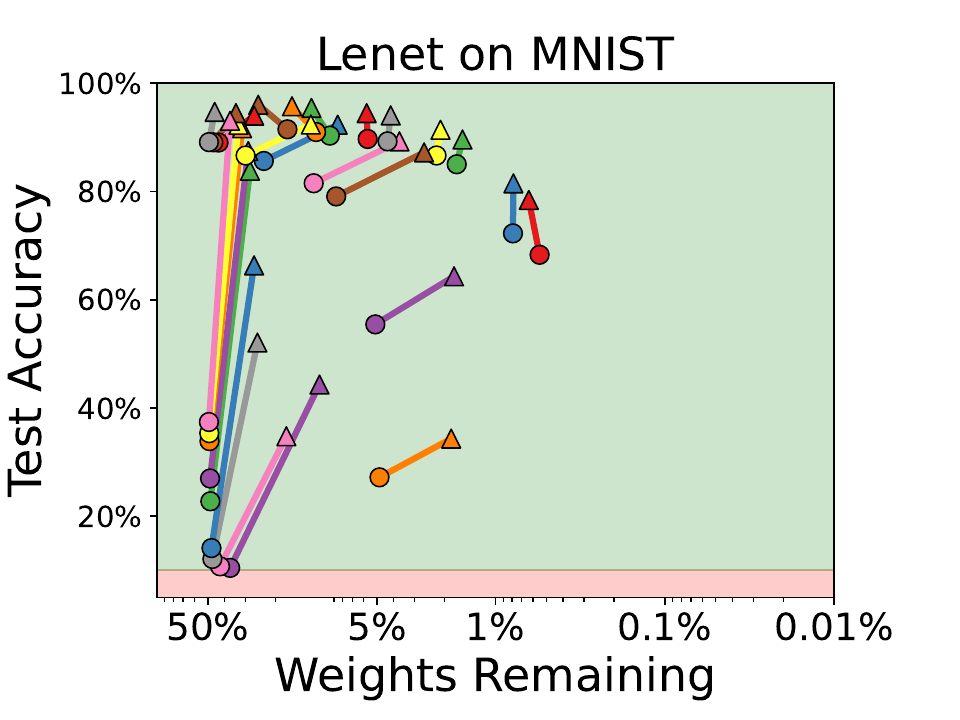}
	\includegraphics[width=0.4\textwidth]{\main/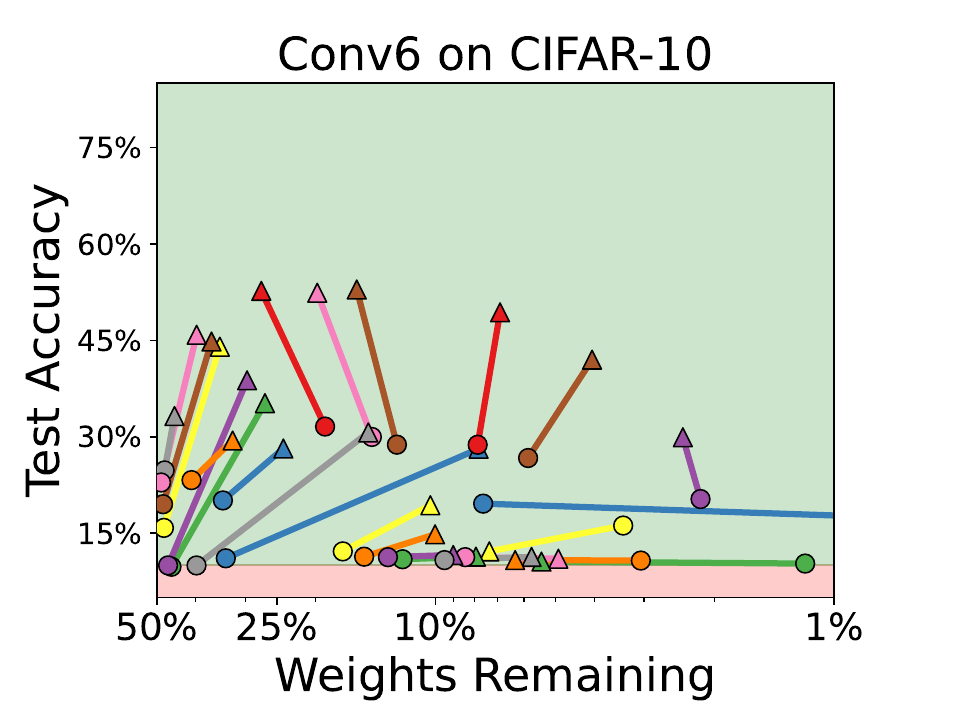}
	\caption{Scatter plot for hyperparameter generalization on supermask results.}
	\label{fig:app_hyper_gen_scatter}
\end{figure}

As mentioned before, we study the possibility of re-using the best hyperparameters for $n=10$ on $n=100$ by looking retroactively at the results from supermask experiments. We focus on escort and ARMS here, although we saw similar results with sigmoid and LOORF. \Cref{fig:app_hyper_gen_scatter} shows a scatter plot, where each point is the average result across seeds. Red regions indicate accuracy akin to chance and runs with the same hyperparameter settings are linked and colored the same. On the majority of cases, runs became sparser or retained similar sparsity, but achieved higher accuracy. The main exception were the sparser runs on the red regions, but those are clearly failed runs.

We note, however, that we tried to extend results from $n=100$ to $n=1,000$ in a similar manner, but the extrapolation did not work as well. It seems that pruning with $n=1,000$ works better with lower $\lambda$ than $n=100$.

\begin{figure}[htb!]
	\centering
	\centering
	\includegraphics[height=1.8em]{\main/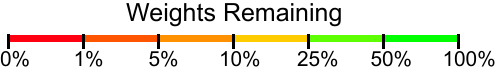}		
	\begin{subfigure}{\textwidth}
		\includegraphics[width=0.333\textwidth]{\main/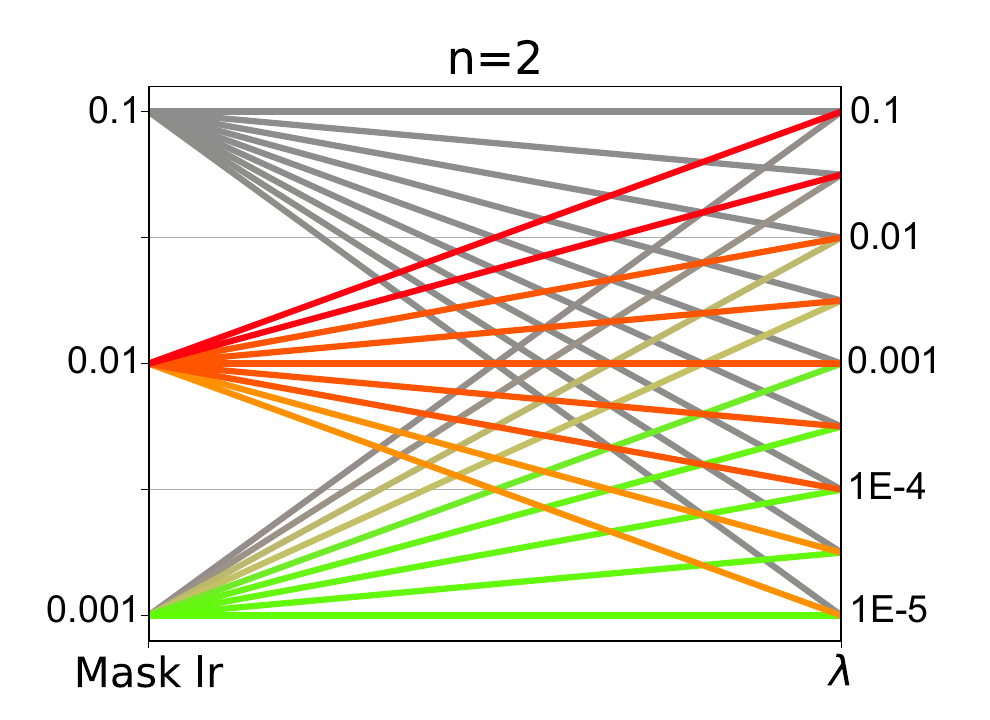}%
		\includegraphics[width=0.333\textwidth]{\main/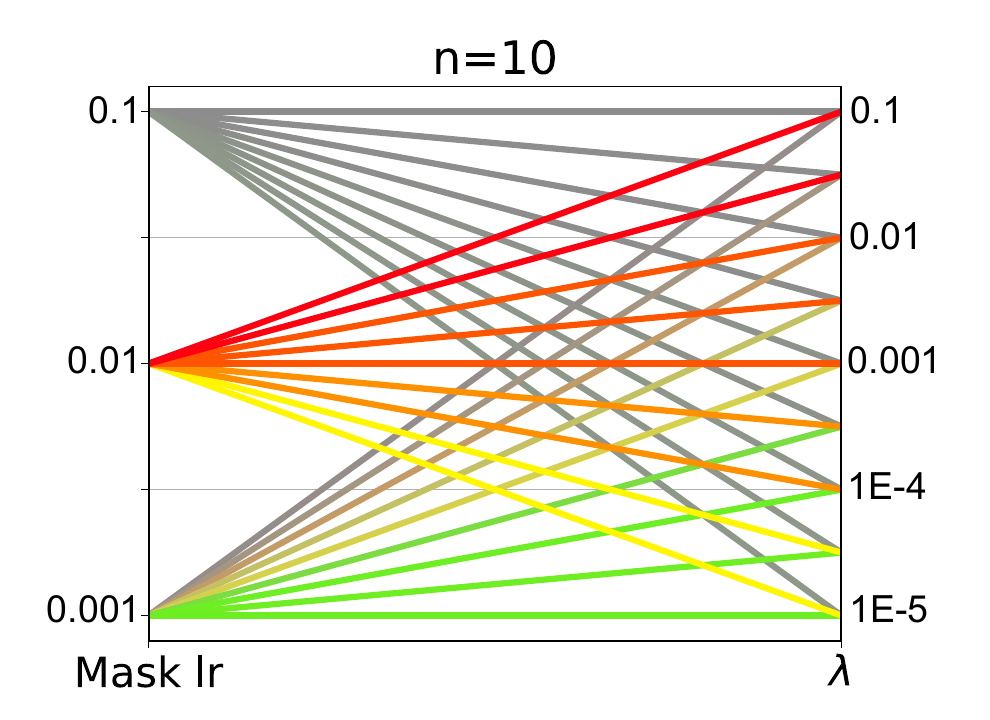}%
		\includegraphics[width=0.333\textwidth]{\main/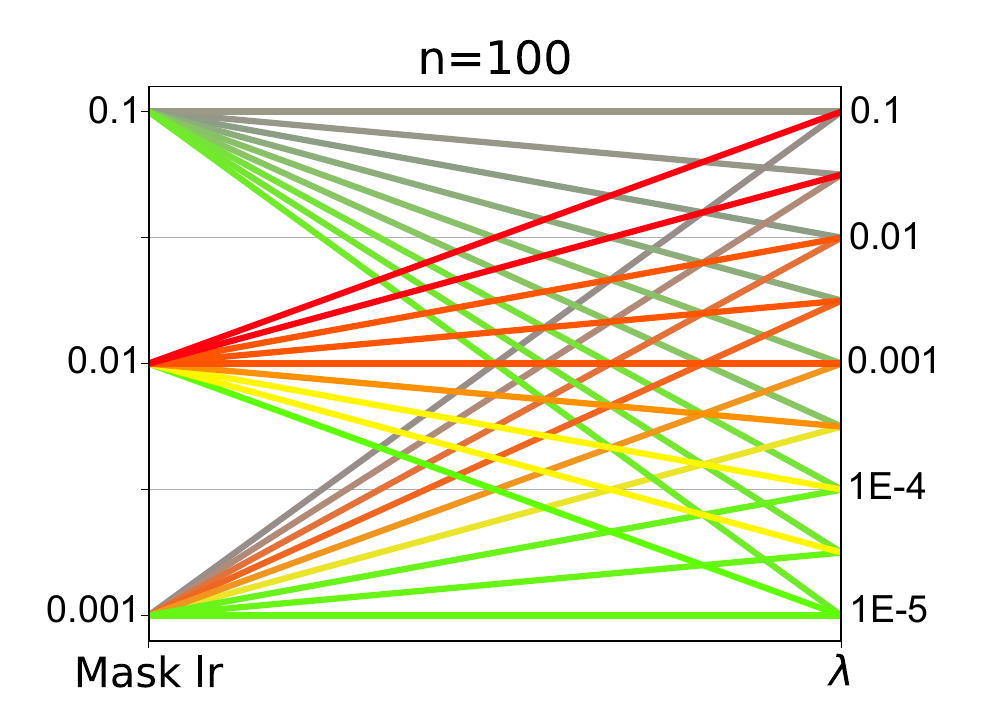}%
		\caption{Lenet on MNIST}
	\end{subfigure}
	\par
	\begin{subfigure}{\textwidth}
		\includegraphics[width=0.333\textwidth]{\main/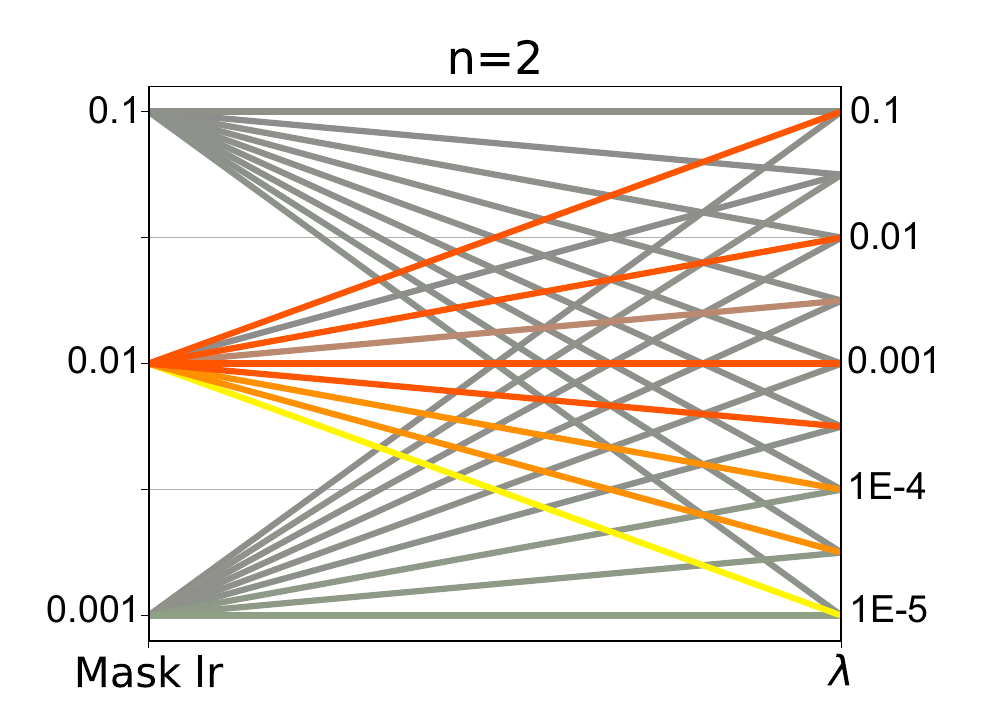}%
		\includegraphics[width=0.333\textwidth]{\main/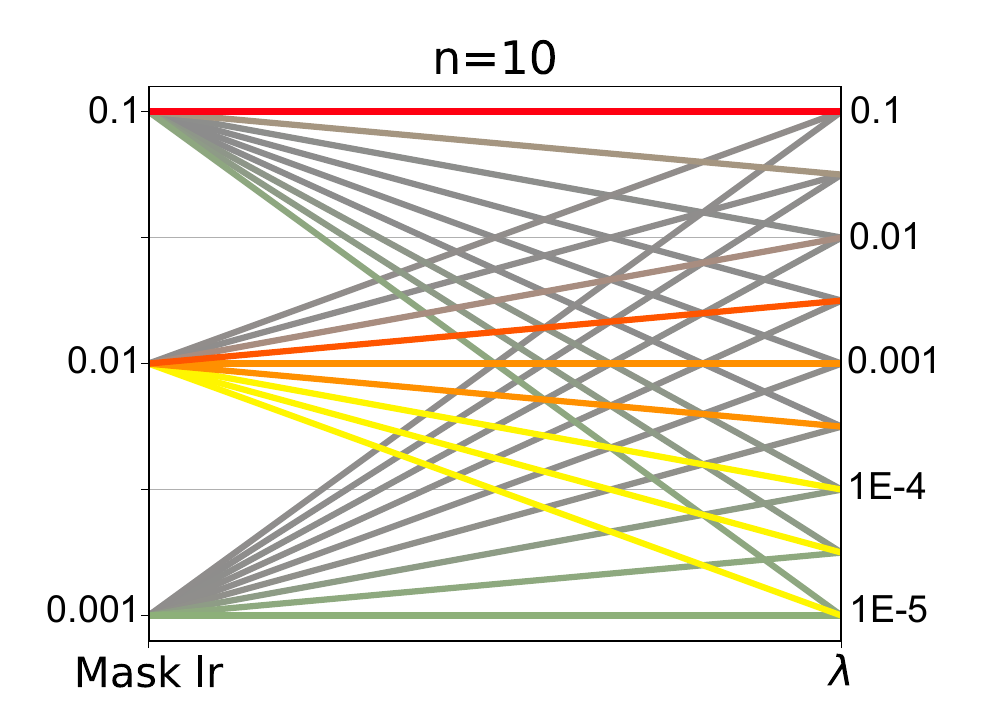}%
		\includegraphics[width=0.333\textwidth]{\main/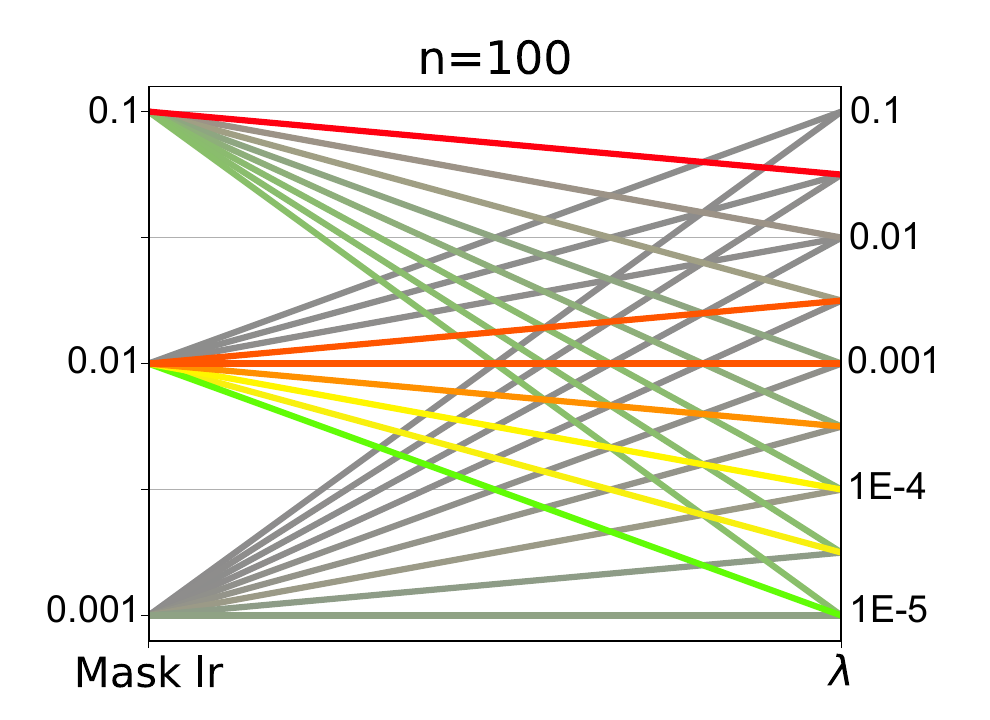}%
		\caption{Conv6 on CIFAR-10}
	\end{subfigure}
	\caption{Parallel coordinates plot for hyperparameter generalization on supermask results. We omit failed runs.}
	\label{fig:app_hyper_gen_parallel}
\end{figure}

\Cref{fig:app_hyper_gen_parallel} shows an alternative view of the same data, where each setting corresponds to a line. Lower saturation (\ie, regions closer to gray) are further from the Pareto front, whereas the more colored regions have better accuracies for their respective sparsity. Hue represents weights remaining, with red regions being sparse and green regions being dense. Overall, we note an agreement across different $n$, with similar hyperparameter combinations leading to similar sparsity and proximity to Pareto front

\subsection{Per-Layer Sparsities Throughout Training}
\label{sec:sparsities_throughout}

To better understand the solutions, we generate \Cref{fig:joint_pruning_per_layer} by selecting one run (out of 5) from each of the settings indicated by arrows in \Cref{fig:joint_pruning_results} and plotting their per-layer sparsities over time. For each method-architecture combination, left plots indicate the percentage of weights remaining relative to the dense version of the respective layer, whereas right plots show how much of the total network capacity is allocated to each layer at a particular time.

For each method-architecture group, inspection of the respective left plots reveals that all methods prioritized the protection of the first layer, pruning mostly the higher-level ones. From the right per-group plots, it is evident that, for both VGG and Resnet-20, dense (\ie, initial) architectures have more weights on high-level layers (notice the red rectangles for the first epochs of GMP and MC). It seems that GMP and CP can prune more aggressively at times, reversing this initial structure in some of the plots. MC, on the other hand, seems to change it to a lesser extent.

In addition to that, CP plots consisted mostly of fixed horizontal stripes, in contrast with the ``jagged'' stripes from MC. CP seems to change the (discrete) masks only in the first epochs, while MC takes longer to find its solutions. Even after that, MC keeps switching between masks (close in Hamming distance) until fine-tuning begins. GMP is the most gradual of the three, using the whole range of epochs by design.

\begin{figure}[htb!]
	\centering
	\includegraphics[width=0.35\textwidth]{\main/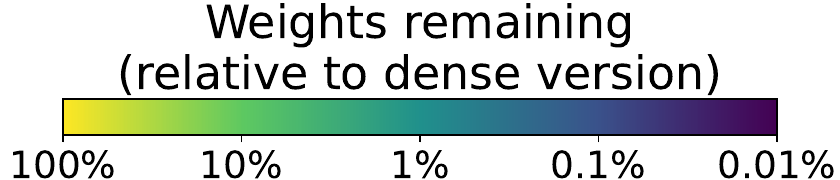}
	\hspace{2em}
	\includegraphics[width=0.35\textwidth]{\main/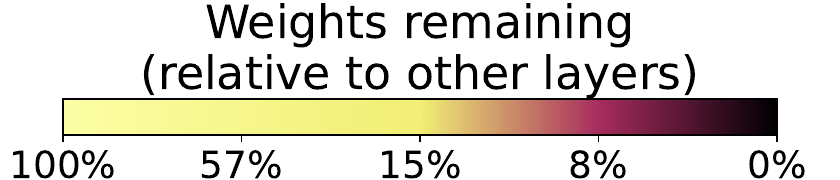}
	\par\vspace{1em}	
	\includegraphics[width=0.5\textwidth]{\main/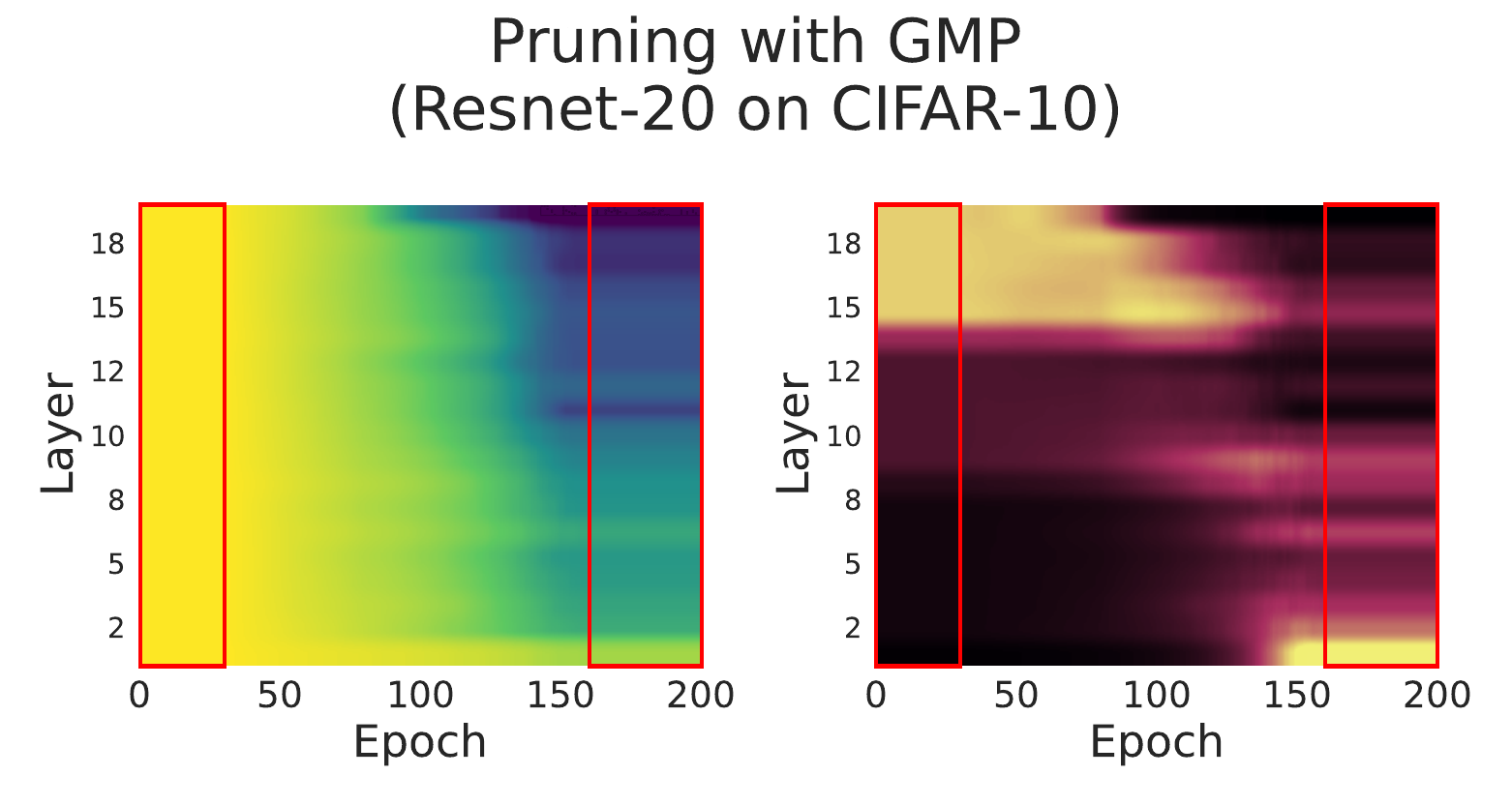}%
	\includegraphics[width=0.5\textwidth]{\main/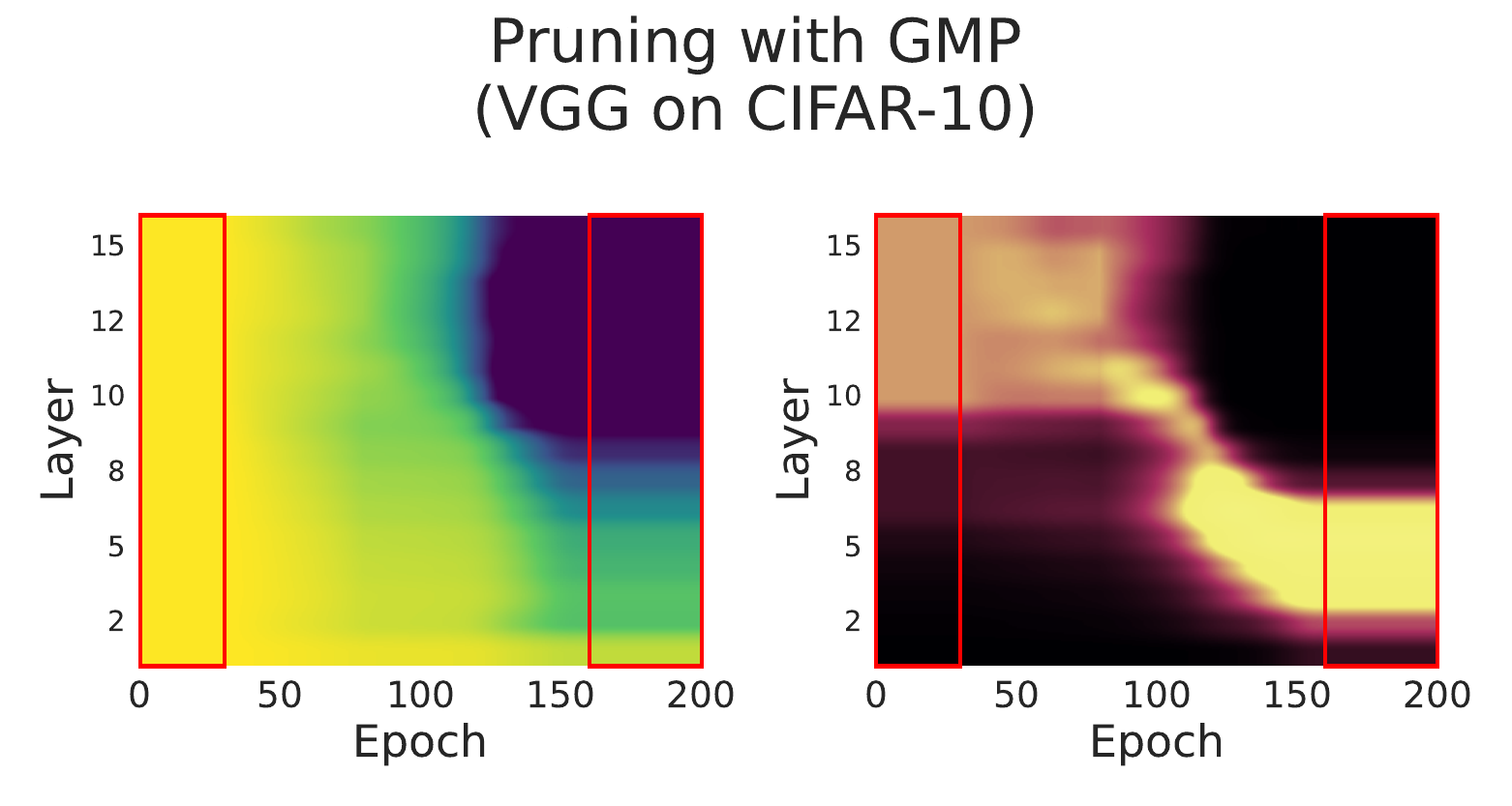}
	\includegraphics[width=0.5\textwidth]{\main/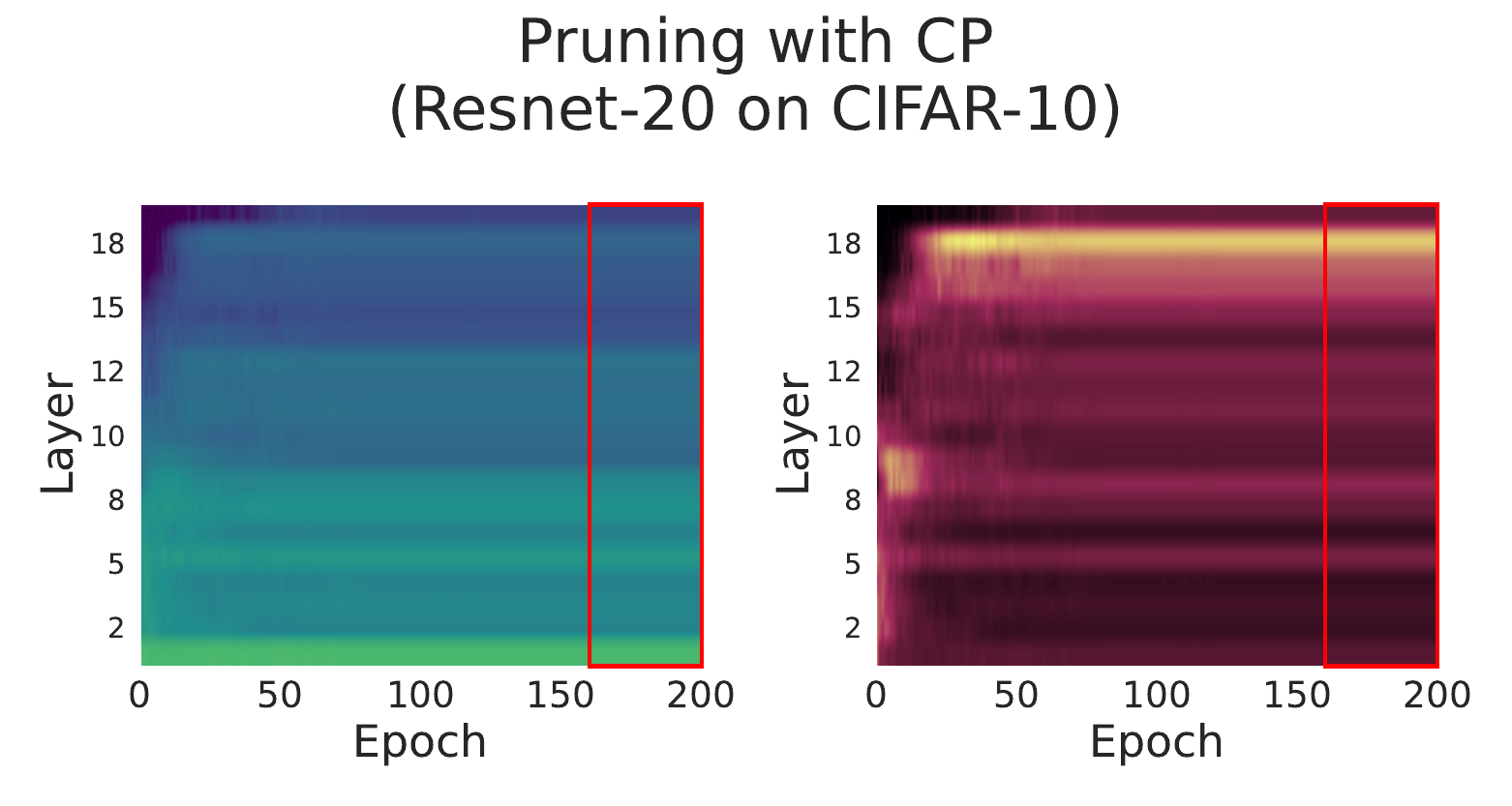}%
	\includegraphics[width=0.5\textwidth]{\main/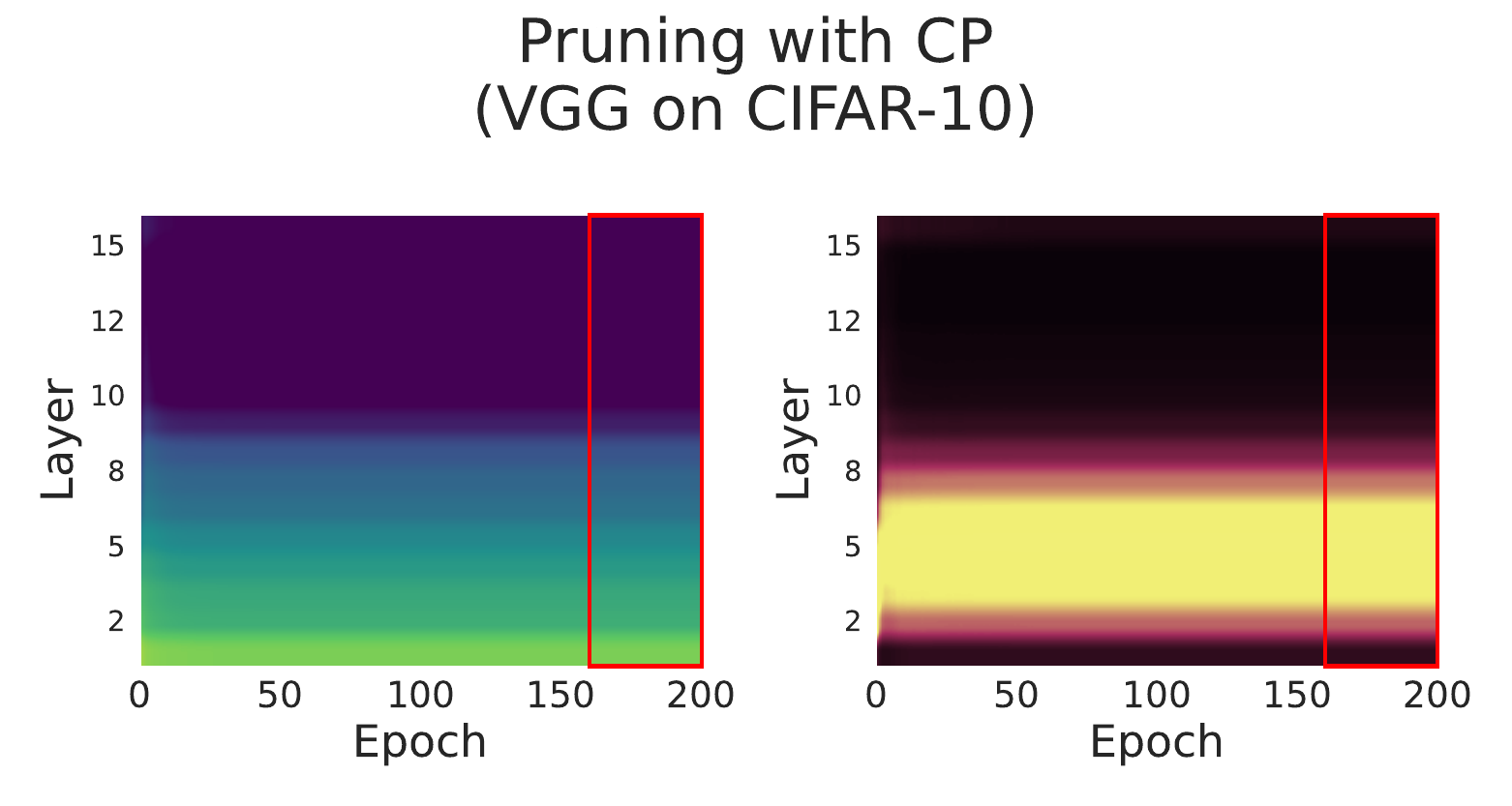}
	\includegraphics[width=0.5\textwidth]{\main/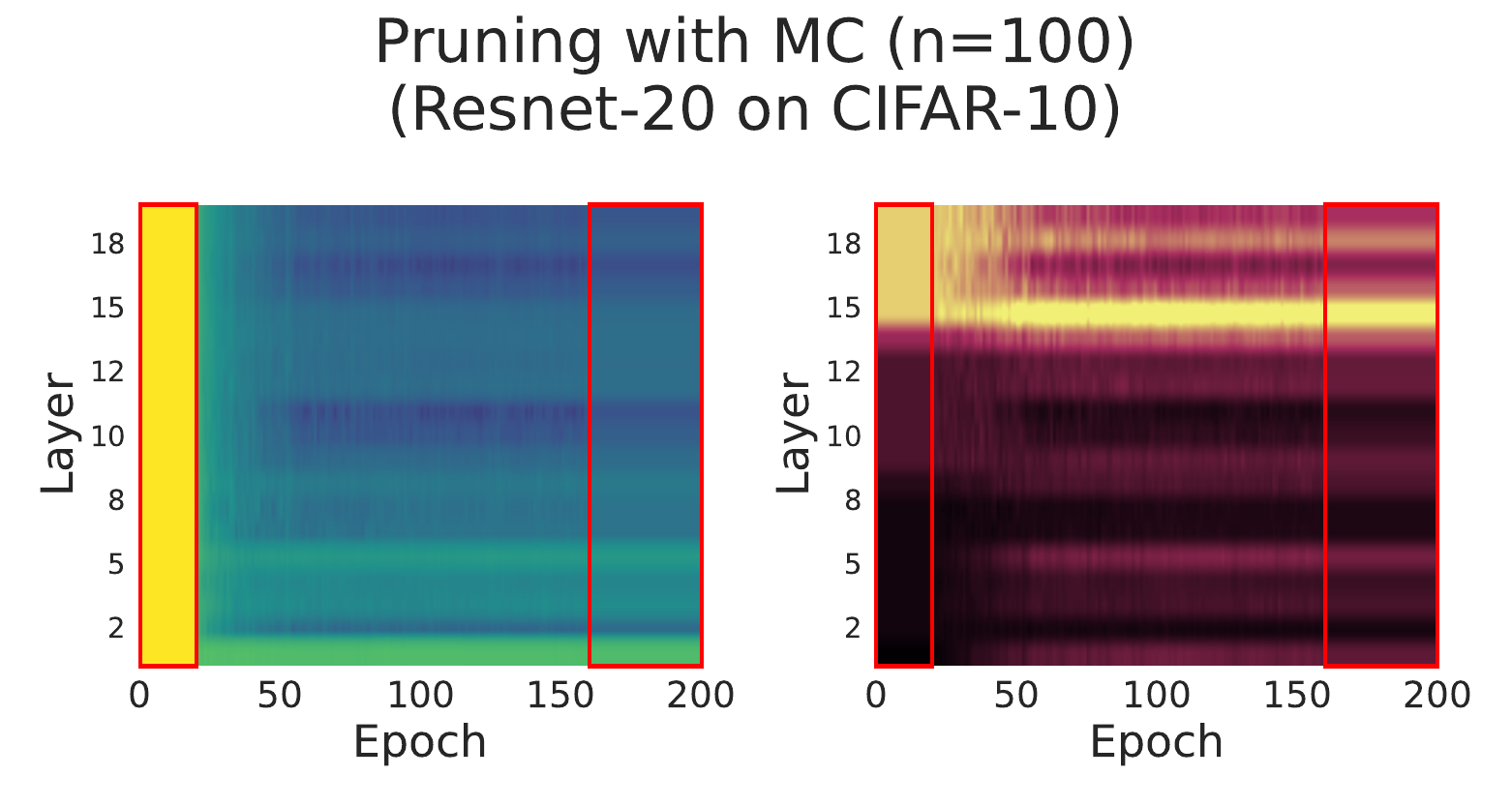}%
	\includegraphics[width=0.5\textwidth]{\main/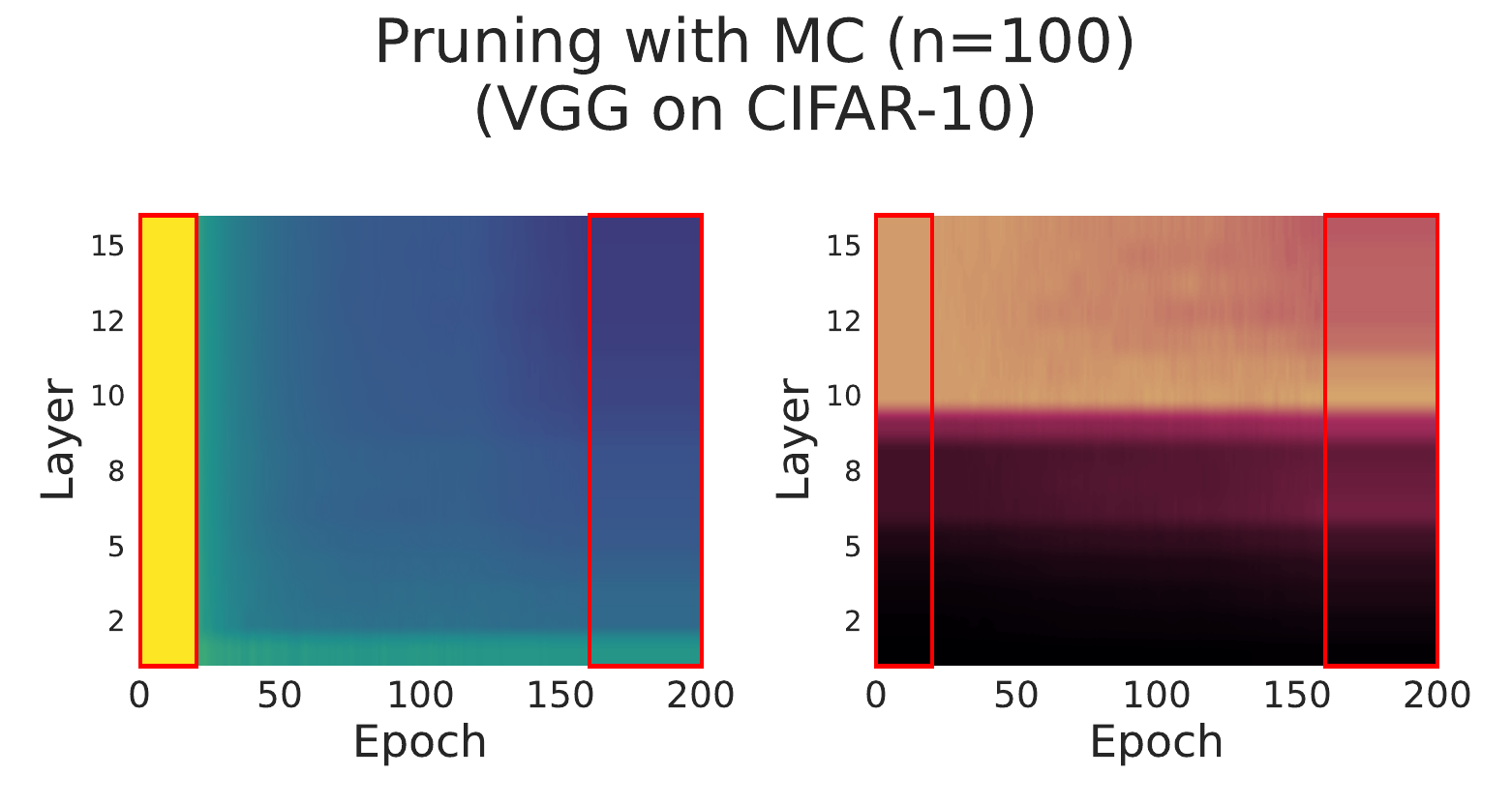}
	\caption{Per-layer sparsity for the runs indicated in \Cref{fig:joint_pruning_results}. Red rectangles correspond to epochs with manually frozen masks.}
	\label{fig:joint_pruning_per_layer}
\end{figure}

\FloatBarrier

\vskip 0.2in
\bibliography{arxiv}

\end{document}